\documentclass{article}[11pt]
\usepackage[T1]{fontenc}
\usepackage{lmodern}

\usepackage{biblatex}
\usepackage{float}
\usepackage[font={small}]{caption}
\usepackage{amsfonts}
\usepackage{amsmath}
\usepackage{biblatex}
\usepackage{amssymb}
\usepackage{algorithm}
\usepackage{bbm}
\usepackage{algpseudocode}
\usepackage{graphicx}
\usepackage{subcaption}
\graphicspath{ {/}}
\usepackage{tikz}
\usetikzlibrary{shapes, calc, shapes, arrows}

\usepackage[margin=1.25in]{geometry}

\newcommand{\blueversion}[1]{{\color{blue} #1}}
\renewcommand{\blueversion}[1]{#1}

\usepackage{setspace}

\newcommand{\showalgcaption}[1]{#1}
\renewcommand{\showalgcaption}[1]{} 

\usepackage[utf8]{inputenc} 
\usepackage[T1]{fontenc}    
\usepackage{amsfonts}       
\usepackage{nicefrac}       
\usepackage{microtype}      

\usepackage{amsthm}
\newtheorem{theorem}{Theorem}
\newtheorem{corollary}{Corollary}[theorem]
\newtheorem{lemma}[theorem]{Lemma}
\newtheorem{proposition}[theorem]{Proposition}
\theoremstyle{definition}
\newtheorem{definition}{Definition}
\newtheorem*{definition*}{Definition}

\newtheorem{assumption}{Assumption}
\newtheorem{remark}{Remark}

\DeclareMathOperator*{\argmax}{\arg\max}

\newcommand{\R}{\mathbb{R}}

\newcommand{\D}{\mathcal{D}}

\newcommand{\T}{\mathcal{T}}
\newcommand{\U}{\mathcal{U}}

\newcommand{\quadq}{quad query}
\newcommand{\quadqs}{quad queries}
\newcommand{\Quadq}{Quad query}

\newcommand{\Rvqtctc}{Real-valued query (\tctc~model)}
\newcommand{\Tqtctc}{Triplet query (\tctc~model)}
\newcommand{\Quadqtctc}{Quad query(\tctc~model)}
\newcommand{\submetric}{submetric}
\newcommand{\submetrics}{submetrics}

\newcommand{\maxmerge}{\mathsf{maxmerge}}
\newcommand{\minmerge}{\mathsf{minmerge}}
\newcommand{\threshunder}{f_r^\T}

\newcommand{\orderingtosubmetric}{\mathsf{OrderingToSubmetric}}
\newcommand{\orderingtosubmetrictctc}{\mathsf{OrderingToSubmetricTCTC}}
\newcommand{\multipleorderingtosubmetric}{\mathsf{MultipleOrderingsToSubmetric}}
\newcommand{\multipleorderingtosubmetrictctc}{\mathsf{MultipleOrderingsToSubmetricTCTC}}

\newcommand{\mergeorderings}{\mathsf{MergeOrderings}}

\newcommand{\combiner}{\mathsf{Combiner}}
\newcommand{\submetriclearner}{\mathsf{SubmetricLearner}}
\newcommand{\submetriclearnertctc}{\mathsf{SubmetricLearnerTCTC}}

\newcommand{\linearvote}{\mathsf{LinearVote}}
\newcommand{\thresholdcombiner}{\mathsf{ThresholdCombiner}}
\newcommand{\thresholdcombinertctc}{\mathsf{ThresholdCombinerTCTC}}

\newcommand{\mreal}{\mathsf{O}_{\mathsf{REAL}}}

\newcommand{\mtrip}{\mathsf{O}_{\mathsf{TRIPLET}}}

\newcommand{\mquad}{\mathsf{O}_{\mathsf{QUAD}}}
\newcommand{\mrealtctc}{\mathsf{O}_{\mathsf{REAL}}^{\mathsf{TCTC}}}

\newcommand{\mtriptctc}{\mathsf{O}_{\mathsf{TRIPLET}}^{\mathsf{TCTC}}}

\newcommand{\mquadtctc}{\mathsf{O}_{\mathsf{QUAD}}^{\mathsf{TCTC}}}
\newcommand{\tctclabeledsamples}{\mathcal{M}_t^r}

\newcommand{\midpoint}{\mathsf{mid}}
\newcommand{\midpointof}{\mathsf{MidpointOf}}
\newcommand{\size}{\mathsf{Size}}

\newcommand{\unifu}{\mathcal{U}^*}
\newcommand{\Dconst}{\mathcal{D}_{r}'}
\newcommand{\DconstR}{\mathcal{D}_{R}'}

\newcommand{\Dr}{\mathcal{D}_{r}}
\newcommand{\DR}{\mathcal{D}_{R}}
\newcommand{\maxcontraction}{c_{max}}
\newcommand{\arbiters}{arbiters}

\newcommand{\arbiter}{arbiter}
\newcommand{\Arbiter}{Arbiter}
\newcommand{\humans}{human fairness arbiters}

\newcommand{\human}{human fairness arbiter}
\newcommand{\Human}{Human fairness arbiter}
\newcommand{\diffuse}{diffuse}
\newcommand{\diffusion}{diffusion}

\newcommand{\density}{density}
\newcommand{\zsubmetric}{$0-$submetric}
\newcommand{\asubmetric}{$\alpha-$submetric}
\newcommand{\et}{\varepsilon_t}
\newcommand{\dt}{\delta_t}
\newcommand{\er}{\varepsilon_r}
\newcommand{\dr}{\delta_r}
\newcommand{\eR}{\varepsilon_R}
\newcommand{\dR}{\delta_R}
\newcommand{\eL}{\varepsilon}
\newcommand{\dL}{\delta}
\newcommand{\w}{w}
\newcommand{\ord}{\mathcal{O}}
\newcommand{\alphH}{\alpha_H}
\newcommand{\alphL}{\alpha_L}
\newcommand{\binaryinsert}{\mathsf{BinaryInsert}}
\newcommand{\binaryinsertpair}{\mathsf{BinaryInsertPair}}
\newcommand{\binaryinserttctc}{\mathsf{BinaryInsertTCTC}}
\newcommand{\binaryinsertpairtctc}{\mathsf{BinaryInsertPairTCTC}}

\newcommand{\leftovers}{\mathsf{NearCollisionList}}
\newcommand{\splitlist}{\mathsf{SplitList}}
\newcommand{\maxalphlog}{\max\{\frac{1}{\alpha},\log(N)\}}
\newcommand{\maxalphlogR}{\log(|R|N)}
\newcommand{\numrepsb}{\frac{1}{b}\ln(\frac{1}{b\delta})}
\newcommand{\maxalphlogRb}{\log(\hat{N} \numrepsb )}

\newcommand{\exactarbiter}{exact arbiter}

\newcommand{\tctc}{too close to call}
\newcommand{\exact}{exact}
\newcommand{\Utr}{\U_t^r}
\newcommand{\Utir}{\U_{t_i}^r}
\newcommand{\labelbound}{2\alphH}
\newcommand{\threshspace}{2\alphH}
\newcommand{\blueerr}{4\alpha_\T}

\newcommand{\LineComment}[1]{\Statex \textit{#1}}
\newcommand{\thresholdout}{unambiguous threshold distribution}
\newcommand{\Thresholdout}{Unambiguous threshold distribution}

\title{Metric Learning for Individual Fairness}

\author{
  Christina Ilvento \\
   \thanks{This work was supported in part by Microsoft Research and the Smith Family Fellowship. The author is grateful for the comments of Cynthia Dwork in the preparation of this manuscript.}
   John A Paulson School of Engineering and Applied Science\\
   Harvard University\\
   Cambridge, MA 02138 \\
   \texttt{cilvento@g.harvard.edu} \\
 }
\date{ }
\begin{document}
\maketitle

\begin{abstract}
    There has been much discussion concerning how ``fairness'' should be measured or enforced in classification.
    Individual Fairness [Dwork, Hardt, Pitassi, Reingold, Zemel, 2012], which requires that similar individuals be treated similarly, is a highly appealing definition as it gives strong treatment guarantees for individuals. Unfortunately, the need for a task-specific similarity metric has prevented its use in practice.
    In this work, we propose a solution to the problem of approximating a metric for Individual Fairness based on human judgments.
    Our model assumes access to a \human~who 
    is free of explicit biases and possesses sufficient domain knowledge to evaluate similarity.
    Our contributions include definitions for metric approximation relevant for Individual Fairness, constructions for approximations from a limited number of realistic queries to the \arbiter~on a  sample of individuals, and learning procedures to construct hypotheses for metric approximations which generalize to unseen samples under certain assumptions of learnability of distance threshold functions.

\end{abstract}

\clearpage
\setcounter{tocdepth}{2}
\tableofcontents
\clearpage


\section{Introduction}\label{nips:section:intro}
Determining what it means for an algorithm or classifier to be ``fair'' and how to enforce any such determination has become a subject of considerable interest as automated decision-making increasingly takes the place of direct human judgment. One attractive definition proposed is Individual Fairness \cite{Dwork-FTA}, which states that similar individuals should be treated similarly, where similarity is encoded in a task-specific \textit{metric}.

\begin{definition}[Individual Fairness  \cite{Dwork-FTA}]\label{def:individualfairness} Given a universe $U$, a metric $\D: U \times U \rightarrow [0,1]$ for a classification task with outcome set $O$, and a distance measure $d:\Delta(O) \times \Delta(O) \rightarrow [0,1]$, a randomized classifier $C: U \rightarrow \Delta(O)$ is \textit{Individually Fair}
if and only if for all $u,v \in U$, $\D(u,v) \geq d(C(u),C(v))$.
\end{definition}

Individual Fairness is appealing because each person is assured that her treatment is similar to that of any person similar to her.\footnote{By way of contrast, notions of fairness based on group level statistics can only provide individuals with the guarantee that if they are treated poorly, either someone in a different group is also treated poorly or someone in their group is treated well. Furthermore, many popular notions of statistical group fairness conflict with each other and cannot be satisfied simultaneously \cite{chouldechova2017fair,DBLP:journals/corr/KleinbergMR16}.}
However, the value of this assurance critically depends on the extent to which the
similarity metric $(\D)$
faithfully represents society's best understanding of what constitutes similarity for a given task.
Thus, the most significant barrier to implementing Individual Fairness in practice is the need to construct a similarity metric for each classification setting. 

In this work we set out a path for constructing metrics for Individual Fairness based on judgments made by a qualified, fair-minded ``\human.'' Our contributions include:
(1) a framework for useful approximations to a metric for Individual Fairness;
(2) a limited, realistic query model for determining the \arbiter's judgments of who is similar to whom;
(3) a method for constructing approximations to the true metric with limited queries to the \arbiter~by using distances from a (set of) representative individual(s);
(4) a procedure for generalizing these approximations to unseen samples based on limited learnability assumptions.
Throughout this work we make no assumption on the form of the metric or the features included in the learning procedure 
with the clearly stated exception of Assumption \ref{assumption:pacthreshold} concerning learnability of threshold functions. 
As our results are built upon a series of sequential steps including new terminology and machinery, we first present an extended introduction to highlight the key concepts, logic and results. In Sections \ref{section:preliminaries}-\ref{section:tctc} these results are discussed in greater detail and formal theorem statements and proofs are presented. Related work is discussed in Section \ref{section:related} Extended discussion of human fairness arbiters and the model is included in Section \ref{section:discussion}.

\subsection{Model}\label{nips:section:model}
In this work, we take the viewpoint that fairness is not well described by either accuracy or group statistics alone.
Instead, we view fairness as a highly contextual property one can identify but not necessarily describe.\footnote{ \cite{gillen2018online} takes a similar approach in which a judge ``knows it when she sees it,'' but is not required to articulate why a decision is unfair.} 
Our goal is to produce a metric which results in similarity judgments with which fair-minded people would agree, rather than satisfying any particular statistical properties.\footnote{We discuss different types of agreement, and the extent to which we fully achieve this goal, in Section \ref{section:arbiteragreement}.}
The core of our model is the \human, 
a fair-minded individual who is free from explicit biases or arbitrary preferences, is motivated to engage ethically and honestly in the query protocol, and has sufficient knowledge and contextual understanding of who is similar to whom for a particular task. The arbiter is not expected to provide us a description or specification of the distance metric. 

A critical part of learning metrics based on human judgments is determining the type of queries to ask in order to solicit consistent, fast responses.
To that end, we assume that we cannot ask the \arbiter~to consider more than a few individuals at a time, e.g., it is not realistic to ask the \arbiter~to find the closest pair of elements in the universe.

We ask the \arbiter~to answer two types of queries in this work: relative distance queries, (e.g., is $a$ closer to $b$ or $c$), and real-valued distance queries. 
\begin{definition}[Real-valued distance query]\label{def:realquery} $\mreal(u,v):=\D(u,v)$.
\end{definition}
\begin{definition}[Triplet query]\label{def:tripletquery}
$\mtrip(a,b,c):=\{
1 \text{ if }\D(a,b) < \D(a,c)\text{, }0 \text{ if }\D(a,c)\leq \D(a,b)\}$.
\end{definition}

Producing a consistent set of real-valued distances is not a natural judgment most people are accustomed to making, so we assume that real-valued queries are very ``expensive'' for the \arbiter~to answer. 
Furthermore, maintaining internal consistency may \textit{increase} the query cost as the number of queries increases.
Relative distance queries have been used successfully for human evaluation in image processing and computer vision, e.g. \cite{van2012stochastic,wilber2014cost}, and 
we anticipate they will be significantly easier for the \arbiter~to evaluate. 
Demonstrating how to replace difficult queries with easy queries is a significant part of our contribution.

We make several simplifying assumptions about the nature of the \human~in the main results of this work. 
    (1) There is either one \arbiter~or all \arbiters~agree on all decisions. 
    (2) The \arbiter~does not change her opinion over the query period. 
    (3) The \arbiter's responses are consistent, i.e., if  she answers that $a$ is closer to $b$ than it is to $c$, her responses to real-valued queries will also reflect this relative judgment.\footnote{Please see Section \ref{section:discussion} for additional details.}
For the majority of this work, we focus on the query model specified above, which requires the \arbiter~to answer with arbitrary precision. We also present a relaxed model which allows the \arbiter~to answer real-valued queries with bounded noise and does not require arbitrarily small distinctions in relative distances queries. The main results presented are replicated in the relaxed model. 
As the results are similar, we focus on the more simple exact model in the main presentation of our results.\footnote{Extended discussion of the exact query model and a more general definition of relative queries is included in Section \ref{section:preliminaries}. The relaxed query model is discussed in detail in Section \ref{section:tctc}.}


\subsection{Contributions}\label{nips:section:contributions}
\textbf{Approximating the metric by contracting.} Our first key observation is that Individual Fairness only requires that we do not \textit{overestimate} distances. This motivates our definition of a \textit{submetric}, which is a contraction of the original metric and can be substituted for the original metric and still maintain Individual Fairness. 

\noindent \textbf{Constructing submetrics based on distances from representative elements.} 
Taking the difference in distance to a single reference or ``representative'' point is one of the simplest ways to produce an underestimate of the distance between two elements. 
Submetrics based on distances from representative elements form the basis of all of our constructions, and although this may seem simplistic, it has a significant advantage when it comes to deciding which queries to ask the \arbiter: \textit{ordering}.
An ordering of elements by increasing distance from the representative can be constructed  with relative distance (easy) queries used as a comparator. Once this ordering is established, real-valued distances at a given granularity can be layered on top in a \textit{sublinear} number of real-valued (hard) queries.

\noindent \textbf{Choosing representatives.}
A single representative may not be sufficient to capture all relevant distance information, but 
 combining the information from multiple representative elements can produce a more complete picture of the distances between all pairs of individuals.
But which representatives should we choose to maximize distance preservation?
We discuss a general, randomized approach and show that given certain properties of the metric, i.e. how tightly packed individuals are, a random set of representatives of reasonable size will have good distance preservation properties. 

\noindent \textbf{Generalizing submetrics to unseen samples.}
Once we have established how to construct a submetric for a fixed sample of elements, our next step is to generalize to unseen samples.
Our results are based on an assumption that threshold functions, i.e. binary indicators of whether an element is closer to a representative than a given threshold, are efficiently learnable. We show how to combine threshold functions to simulate rounding distances to a representative and then exhibit appropriate parameters to construct an efficient combined learning procedure.


\noindent \textbf{Relaxing arbiter requirements.} Finally, we present a relaxation of the arbiter query model in which the arbiter (1) may respond to real-valued queries with arbitrary bounded noise and (2) is not required to make arbitrarily precise distinctions between distances and may instead declare relative comparisons to be ``too close to call.'' This model more closely matches the reality of human arbiters, and our results extend with improvements in query complexity at the cost of increased error magnitude. 

\subsection{Preliminary terminology and definitions}\label{nips:section:preliminaries}

\noindent We refer to the universe of individuals as $U$, a distribution over the universe of individuals as $\U$, and the size of the universe as $|U|=N$. We write $\unifu$ for the uniform distribution over $U$. We assume $\D: U \times U \rightarrow [0,1]$ for simplicity.
Individual Fairness does not require that distances between individuals be maintained exactly, only that they not be exceeded.
This observation motivates our definition of a \textit{submetric} which is a contraction of the true metric, i.e., it does not \textit{overestimate} any distance  beyond a small additive error term.\footnote{This relaxation is very similar to the notion of $(d,\tau)$ metric fairness of \cite{kim2018fairness} and approximate metric fairness of \cite{rothblum2018probably}.}
\begin{definition}[$\alpha-$\submetric]
Given a metric $\D$, $\D': U \times U \rightarrow [0,1]$ is an $\alpha$-\submetric~of $\D$ if for all $u,v \in U$, $\D'(u,v) \leq \D(u,v) + \alpha$.
\end{definition}
Any classifier which satisfies the distance constraints of the submetric $\D'$ will also satisfy those of $\D$, modulo small additive error.\footnote{As originally noted in \cite{Dwork-FTA}, the distance measure need not be a true metric, i.e. it does not strictly need to obey triangle inequality or distinguish unequal elements.}
Given an $\alpha$-submetric it is possible to eliminate the additive error by taking $\max\{0,\D'( x, y)-\alpha\}$. 
On the other hand, we want to avoid contracting distances to the point of triviality. 
We say that a submetric is $(\beta,c)-$nontrivial if a $\beta$ fraction of distances between pairs preserve at least a $c-$fraction of their original distance.\footnote{Nontriviality is defined over a product of identical distributions of elements in the universe. There is no general obstacle to extending our results to more complicated scenarios, but definitions of density (in Section \ref{section:choosingreps}) would need to be adjusted.} 
\begin{definition}[$(\beta,c)-$nontrivial]\label{def:nontrivial}
Given a metric $\D$,
a submetric $\D'$ of $\D$ is $(\beta,c)$-nontrivial for the distribution $\mathcal{U}$ if
$\Pr_{u,v \sim \mathcal{U} \times \mathcal{U}} \Big [\frac{\D'(u,v)}{\D(u,v)} \geq c \Big ] \geq \beta$.
\end{definition}



\begin{algorithm}[t]
  \caption{(Pseudocode)}
  \label{nips:alg:orderingtometric}
\begin{algorithmic}[1]
  \LineComment{Inputs: the representative $r$, a set of elements $U$, error parameter $\alpha$, interfaces $\mtrip$ and $\mreal$.}
\LineComment{Output: an $\alpha-$submetric $\Dr'$.}
  \State Initialize the submetric $\D_r'(x,y)\leftarrow 0$ for all $x,y \in U\times U$.
  \State Order the elements of $U$ by distance from $r$ using $\mtrip$ as a comparator.
  \State Designate the entire ordered list as the first continuous range.
  \State \textbf{while} there are still ranges left to be labeled \textbf{do}
  \State \quad Select a range left to be labeled.
  \State \quad Query $\mreal(r,$ first$)$ and $\mreal(r,$ last$)$ for the first and last elements in the range. 
  \State \quad \quad \textbf{if} the difference in distances is $>\alpha$ \textbf{then}
  \State \quad \quad \quad Split into two continuous ranges, each with half of the elements in the current range.
  \State \quad \quad \textbf{else} set $\D'_r(r,x)$ to $\mreal(r,$ first$)$ for each element $x$ in the range.
  \State Set $\D_r'(x,y)=|\D_r'(r,x) - \D_r'(r,y)|$ for all $x,y$ in the ordering.
  \State \textbf{return} $\D_r'$.

  \end{algorithmic}
\end{algorithm}
\subsection{Constructing submetrics from \arbiter~judgments}\label{nips:section:submetrics}
A core component of this work is constructing submetrics based on distance information (either exact or underestimated) from a single representative element. We define the \textit{representative submetric} $\Dr$ in the following Lemma. (The proof of follows from triangle inequality.)
\begin{lemma}\label{nips:lemma:repsubmetric}
Given a representative $r$, $\Dr(x,y):=|\D(r,x) - \D(r,y)|$ is a 0-submetric of $\D$.
\end{lemma}
Given a sample of $N$ individuals, $\Dr$ can be constructed from $O(N)$ queries to $\mreal$.
Although $O(N)$ may seem good compared with the $O(N^2)$ queries required to reconstruct the whole metric, it can be improved to $O(\log(N))$ by supplementing with relative distance queries.
Our general strategy will be to show that (1) an \textit{ordering} of elements by distance from a representative can be constructed using $\mtrip$ as a comparator, and (2) given this ordering, the real-valued distances between each element and the representative can be closely approximated by labeling the ordering with distances at granularity $\alpha$, which requires a sublinear number of real-valued queries.
Algorithm \ref{nips:alg:orderingtometric} outlines this process.\footnote{See Section \ref{section:humansubmetrics} Algorithms \ref{alg:tripletordering} and \ref{alg:orderingtometric} for the detailed specifications and analysis.}

Theorem \ref{nips:theorem:orderingtometric} states that Algorithm \ref{nips:alg:orderingtometric} produces an $\alpha-$submetric, which follows from observing that rounding $\D(r,x)$ and $\D(r,y)$ down by at most $\alpha$ results in an increase (or decrease) of at most $\alpha$ in $|\D(r,x) - \D(r,y)|$. 
The bound of $O(N\log(N))$ relative distance queries follows from a straightforward analysis of sorting.
The bound of $O(\maxalphlog)$ real-valued queries is included in Section \ref{section:humansubmetrics}. Briefly, the analysis considers the maximum number of continuous ranges that, when split, result in one range with difference greater than $\alpha$ and one with less. In the worst case, this results in logarithmic dependency on $N$ or $\frac{1}{\alpha}$.


\begin{theorem}\label{nips:theorem:orderingtometric}
Algorithm \ref{nips:alg:orderingtometric} 
produces an \asubmetric~of $\D$ which preserves $\D(r,u)$ for each $u \in U$ (with additive error $\leq\alpha$) from $O(\maxalphlog)$ queries to $\mreal$ and $O(N\log(N))$ queries to $\mtrip$.
\end{theorem}

The submetric produced by Algorithm \ref{nips:alg:orderingtometric} preserves distances between $r$ and other elements well, as $\D'_r(r,x)$ is rounded down by at most $\alpha$, but we cannot make guarantees on distance preservation between arbitrary pairs  without further information.
For example, with only the information that $u$ and $v$ are equally distant from $r$, it is impossible to distinguish whether the distance between $u$ and $v$ is zero or equal to twice their distance from $r$. (See Figure \ref{fig:repcomparison}). 
Submetrics constructed based on different representatives preserve different information about the underlying metric,  so we can construct more expressive submetrics by aggregating information from multiple representatives. Taking $\maxmerge(\{\D_i\},x,y):= \max_i\D_i(x,y)$, it's straightforward to show that if all $\D_i$ are submetrics of $\D$, then the $\maxmerge$ of the set is also a submetric of $\D$, and that the merge preserves the ``best'' distance known for each pair.\footnote{Formal analysis of $\maxmerge$ and the proof of Lemma \ref{nips:lemma:repsubmetric} appear in Section \ref{section:preliminaries}. The proof of Theorem \ref{nips:theorem:orderingtometric} as well as a precise description of Algorithm \ref{nips:alg:orderingtometric} appear in Section \ref{section:humansubmetrics}.}





\begin{figure}
    \centering

    \begin{subfigure}[b]{0.4\textwidth}
     \includegraphics[width=\textwidth]{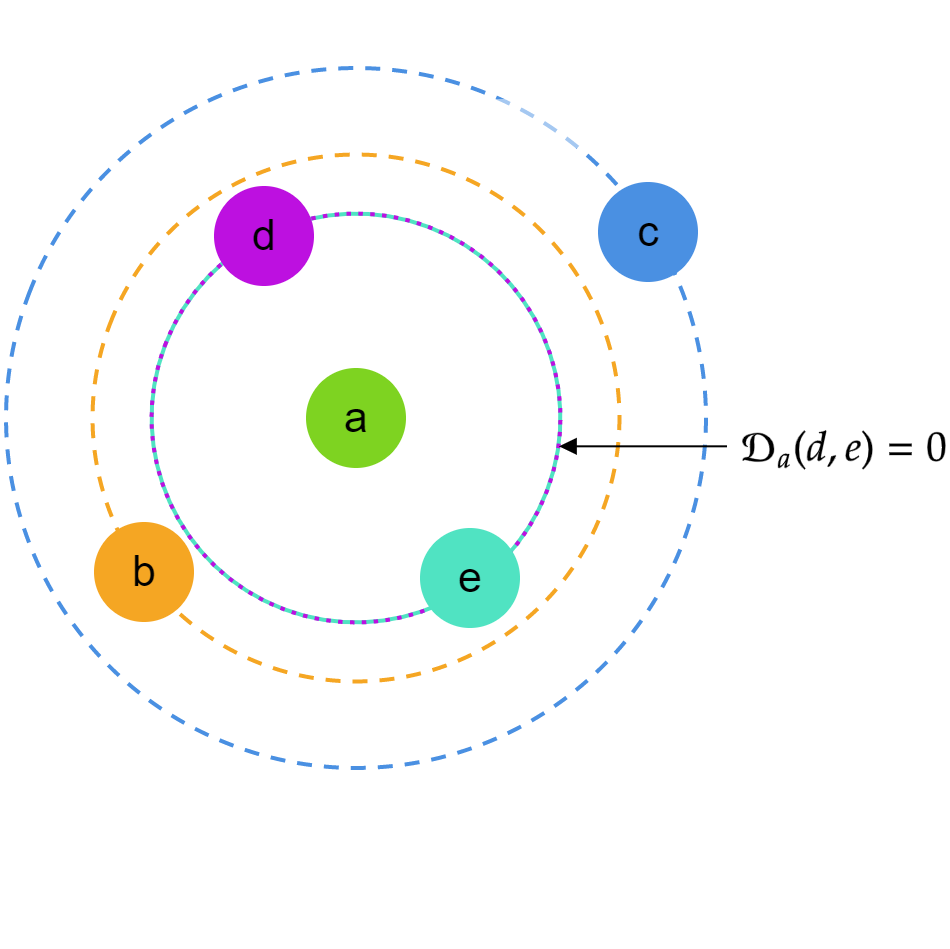}
        \caption{$a$ chosen as representative.} \label{fig:repa}
    \end{subfigure}
    \quad \qquad
 \begin{subfigure}[b]{0.4\textwidth}
    \includegraphics[width=\textwidth]{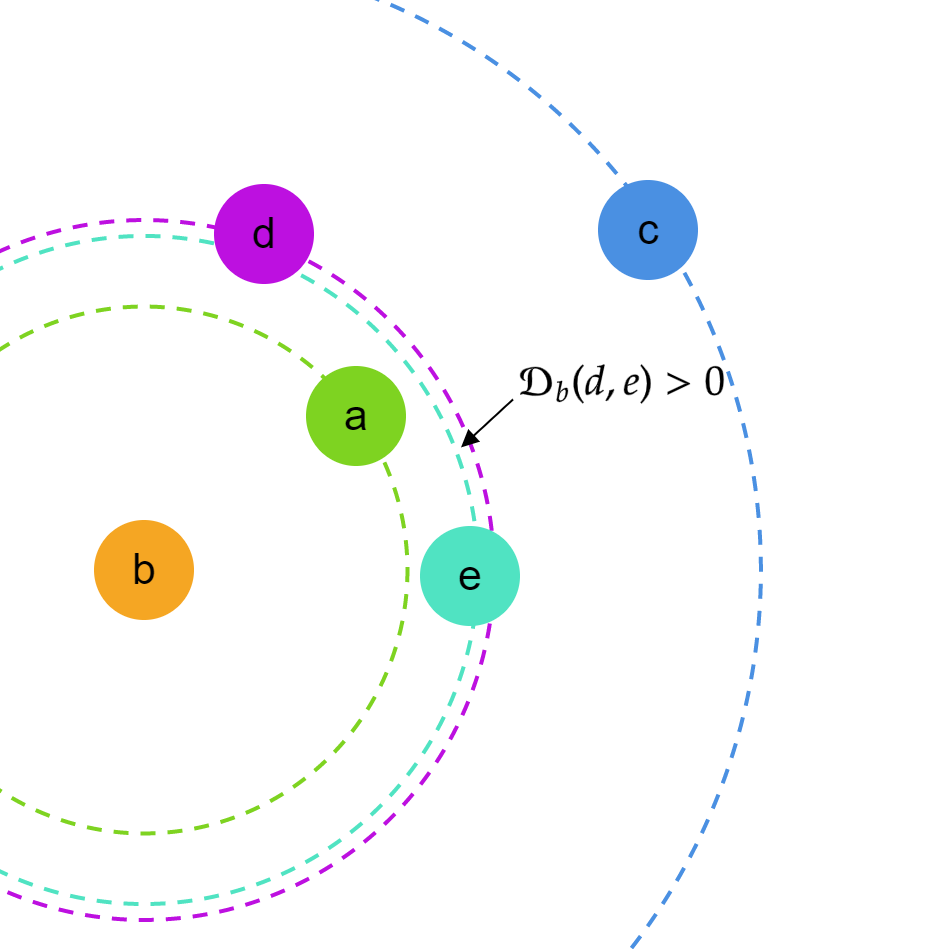}
        \caption{$b$ chosen as representative.} \label{fig:repe}
    \end{subfigure}
  
      \setlength{\belowcaptionskip}{-15pt}
    \caption{The impact of representative choice on distance preservation. The
 distance between each element and the chosen representative is the radius of the shell containing the element. The difference in radii of each pair of shells indicates the distance between the pair of elements under $\D_a$ or $\D_b$. If $a$ is chosen as a representative, notice that $d$ and $e$ are indistinguishable using distance from $a$ alone.
    Choosing representative $b$ preserve distances better than $a$, but still does not distinguish $d$ and $e$ very well. 
    } \label{fig:repcomparison}
\end{figure}
\subsection{Choosing good representative elements}\label{nips:section:choosingreps}

Although the $\maxmerge$ of submetrics based on multiple representatives is an improvement over a single representative, we still cannot make any guarantees about distances between pairs which do not include a representative.
There are two approaches one might take to give non-triviality guarantees for arbitrary pairs: (1) develop specialized strategies for combining representative submetrics which depend on the structure of the metric, e.g., Euclidean distance, or (2) characterize generic randomized representative selection strategies. In this extended introduction, we focus on the randomized strategies for full generality.


\textbf{Distance preservation via $\gamma-$nets.}
The crux of the argument for nontriviality with random representatives is (1) a random set of representatives is likely to be ``close to'' a significant portion of the distribution $\U$, and (2) we can bound the magnitude of underestimates based on the distance from a representative.
 Below, we formally define a $\gamma-$net to capture the notion of being ``close to'' or ``covering'' a set of elements.
 \begin{definition}\label{def:gammanet}
 A set $R \subseteq U$ is said to form a \textit{$\gamma-$net} for a subset $V \subseteq U$ under $\D$ if for all balls of radius $\gamma$ (determined by $\D$) containing at least one element $v \in V$, the ball also contains $r \in R$.
 \end{definition}

Intuitively,
the distance between $r$ and $x$ will be nearly identical to the distance between a close neighbor of $r$ and $x$, 
so we can conclude that if a set of representatives forms a $\gamma-$net for a subset of $U$, then pairs with at least one element in the net will have their original distance preserved up to a $2\gamma$ contraction.
(Proofs of Lemmas \ref{nips:lemma:generalbound} and \ref{nips:lemma:generalnet} follow from triangle inequality.)
\begin{lemma}\label{nips:lemma:generalbound}
For all $u,v \in U \backslash \{r\}$, $\D(u,v) - \Dr(u,v) \leq \min\{2\D(r,u),2\D(r,v)\}$, where $\Dr(u,v):=|\D(r,u)-\D(r,v)|$.
\end{lemma}
\begin{lemma}\label{nips:lemma:generalnet}
If a set of representatives $R \subseteq U$ forms a $\gamma-$net for $V \subseteq U$, then for every pair $x,y \in V \times U$ there exists $r \in R$ such that $\D(x,y) - \Dr(x,y) \leq 2 \gamma$, where $\Dr(x,y):=|\D(r,x) - \D(r,y)|$.
\end{lemma}
Of course, forming a $\gamma-$net for an \textit{arbitrary} $\gamma$ isn't enough to give a good nontriviality guarantee. 
To understand how representatives which form a $\gamma -$net will preserve distances, we define \textit{\density}~and \textit{\diffusion}~below to characterize the relevant properties of the metric and distribution.
The notion of $(\gamma,a,b)-$dense is intended to capture the weight ($a$) of elements that have a significant weight ($b$) on their close neighbors (distance $\gamma$) under $\U$ as a way to characterize how likely it is that a randomly chosen representative will be $\gamma$-close to a significant fraction of elements.
\begin{definition}[$(\gamma,a,b)-$dense]\label{def:density}
Given a distribution $\mathcal{U}$ over $U$, a metric $\D$ is $(\gamma,a,b)-$dense for $\mathcal{U}$ if there exists a subset $A\subseteq U$ with weight $a$ under $\mathcal{U}$ such that for all $u \in A$
$\Pr_{v \sim \mathcal{U}}[\D(u,v) \leq \gamma] \geq b$.
\end{definition}




$(p, \zeta)-$\diffuse, defined below, captures what fraction of distances can tolerate a contraction proportional to $\zeta$ without becoming trivial.
\begin{definition}[$(p,\zeta)-$\diffuse]\label{def:diffuse}
Given a distribution $\U$, a metric $\D$ is $(p,\zeta)-$\diffuse~if the fraction of distances between pairs of elements in $\U \times \U$ greater than $\zeta$ is $p$, i.e.
$\Pr_{u,v \sim \U \times \U}[\D(u,v)\geq \zeta]\geq p$.
\end{definition}


\begin{figure}
\begin{centering}
    \begin{subfigure}[b]{0.4\textwidth}
        \includegraphics[width=\textwidth]{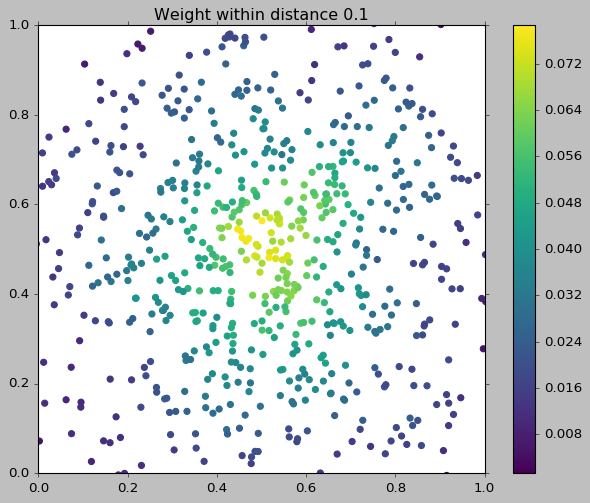}
        \label{nips:fig:avsbgen}
    \end{subfigure}
    ~ \quad 
    \begin{subfigure}[b]{0.35\textwidth}
        \includegraphics[width=\textwidth]{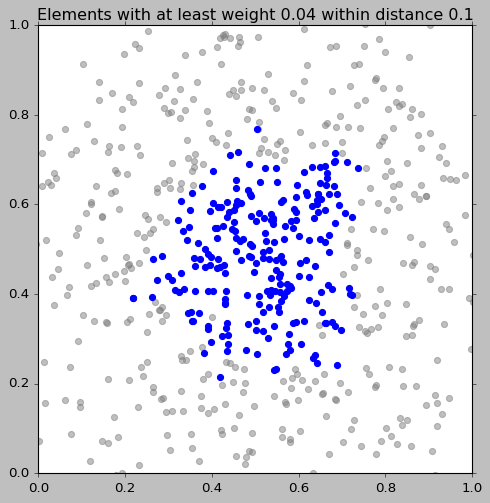}
        \label{nips:fig:avsb4}
    \end{subfigure}
    \setlength{\belowcaptionskip}{-12pt}
    \setlength{\abovecaptionskip}{-5pt}
    \caption{A visualization of the weight of elements, $b$, within distance $\gamma=.1$ of each element under $\unifu$ for an example universe of points in $[0,1]^2$ where $\D$ is taken as Euclidean distance. 
    The color assigned to each point on the left indicates the weight of elements in the universe which are within distance $\gamma=0.1$ from the element under the uniform distribution $\unifu.$
    On the right, the points with at least weight $b=0.04$ of $\unifu$ within distance $\gamma=0.1$ are highlighted in blue.
    This example is $(\gamma=0.1, a=.31,b=0.04)-$dense for $\unifu$.
    That is, 31\% of elements in the universe are within distance 0.1 of 4\% of the rest of the universe.
    } \label{nips:fig:avsb}
 \end{centering}   
\end{figure}
A metric can be described by many combinations of density and diffusion parameters, as illustrated in Figure \ref{nips:fig:avsb}.
These parameters are highly related, and we generally consider the combination of $(\gamma,a,b)-$dense and $(p,\frac{2\gamma}{1-c})-$\diffuse.
Although $\frac{2\gamma}{1-c}$ initially seems an arbitrary quantity, it indicates that a $p-$fraction of pairs will have distances preserved by a factor of $c$ if the maximum contraction for those pairs is no more than $2\gamma$. 
Thus the values of $\gamma$ and $c$, which in turn dictate $p$, $a$, and $b$, (assuming $\zeta=\frac{2\gamma}{1-c}$) can loosely be seen as a tradeoff between how many pairs will have distance preservation guarantees and how significant the guarantees will be. 






\textbf{Nontriviality properties of $\gamma -$nets.}
Next, we relate the magnitude of $\gamma $ to the non-triviality properties of the  $\maxmerge$ of a set of representative submetrics.
Lemma \ref{nips:lemma:fullnetc} states that a submetric based on a set of representatives which form a $\gamma-$net for a subset of $U$ will have nontriviality properties related to the \diffusion~properties of $\D$ and the weight of the subset in $\U$.
\begin{lemma}\label{nips:lemma:fullnetc}
If a set of representatives $R \subseteq U$ form a $\gamma-$net for weight $\w$ of $\U$
and $\D$ is $(p,\frac{2\gamma}{1-c})-$\diffuse~on $\U$,
then the submetric $\DR(x,y):=\maxmerge(\{\Dr | r \in R\},x,y)$ is $(p', c)-$nontrivial for $\U$, where $p' \geq p - (1-\w)^2$.
\end{lemma}

The proof follows from a worst-case analysis of the fraction of pairs with at least one element in the net with distance large enough that a $2\gamma$ contraction leaves at least a $c$-fraction of the original distance. The nontriviality guarantees of Lemma 
\ref{nips:lemma:fullnetc} are conservative, and 
we stress that our goal is to show the possibility of positive results, rather than achieving optimal performance or guarantees.



\textbf{Representative set size.}
We now consider how likely it is that a set of random representatives drawn from $\U$ will form a $\gamma-$net for a significant fraction of $\U$. 
Lemma \ref{nips:lemma:randomrepgeneralization} characterizes the necessary representative set size based on the density and diffusion properties of the metric.
The proof follows from characterizing the probability of ``hitting'' a sufficient weight of the distribution with a sample of a given size, and arguing that no element in our subset of interest can be more than $3\gamma$ far from any of the ``hitting'' elements. 

\begin{lemma}\label{nips:lemma:randomrepgeneralization}
Given access to unlimited queries to the \arbiter, if a metric $\D$ is $(\gamma, a, b)-$dense and
$(p,\frac{6\gamma}{1-c})-$\diffuse~on $\U$,
then a random set of representatives $R$ of size at least $\frac{1}{b}\ln(\frac{1}{b\delta})$ will produce a $(p - (1-a)^2, c)$-nontrivial submetric 
for $\U$ with probability at least $1-\delta$. 
\end{lemma}

Random sampling is not the only method to construct a $\gamma-$net, and our strategy is motivated by simplicity as much as generality. In practice it may be more efficient to use the distance information from previously selected representatives to inform the selection of the next representative. For example, omitting or down-weighting any candidates that are already very close to existing representatives, or using a greedy strategy.\footnote{Section \ref{section:choosingreps} contains proofs for Lemmas \ref{nips:lemma:generalbound}-\ref{nips:lemma:randomrepgeneralization} and extended discussion of specialized strategies for representative selection, in particular strategies taking advantage of known metric structure.}


\subsection{Generalizing \arbiter~judgments}\label{nips:section:generalization}
Now that we have shown how to construct a nontrivial submetric with ongoing access to the \arbiter, we consider the problem of generalizing the \arbiter's responses to unseen samples.
Our goal is to construct efficient learners for submetrics as in Valiant's Probably Approximately Correct (PAC) model of learning \cite{valiant1984theory}. However, we do not want to be too prescriptive about the submetric concept class, particularly about the representation of elements. Instead, we will make an assumption about the learnability of \textit{threshold functions} and construct learning procedures for submetrics using threshold functions as building blocks without any additional direct access to labeled or unlabeled samples.
More formally, our goal is to produce an efficient submetric learner, defined below. 
\begin{definition}[Efficient submetric learner]\label{def:efficientlearner}
A learning procedure is an \textit{efficient $\alpha-$submetric learner} if for all $\varepsilon,\delta \in (0,1]$, given access to labeled examples, with probability at least $1-\delta$ over the randomness of the sampling and the learning procedure produces a hypothesis $h_r$ such that
$\Pr_{x,y \sim \U \times \U}[h_r(x,y) - \D(x,y) \geq \alpha] \leq \varepsilon$ 
in time  $O(poly(\frac{1}{\varepsilon},\frac{1}{\delta}))$.
\end{definition}

To show how to construct an efficient submetric learner, we first formalize our assumption of learnability of threshold functions. Next, we show how to combine the threshold function hypotheses for each representative to simulate rounding the distance between the representative and each element down to the nearest threshold. Finally, we specify the appropriate parameters for each component to achieve the desired bounds.

\textbf{Learnability of threshold functions.}
Assumption \ref{assumption:pacthreshold} (below) states that for every representative, there exists a set of thresholds
and a learner for each threshold which,
with high probability,
produces an accurate hypothesis for the threshold function  which generalizes to unseen samples.\footnote{The formal statement of Assumption \ref{assumption:pacthreshold} is included in Section \ref{section:learnability}.} 
(``With high probability'' always refers to the probability over the randomness of sampling and the learner.)
We first formally define a threshold function, which is a binary indicator of whether a particular element $u \in U$ is within distance 
$t \in [0,1]$ of a representative $r$ as $   T_t^r(u) := \{
      1 $ if $ \D(r,u) \leq t, $ $
      0 $ otherwise$\}$.

\begin{assumption}\label{assumption:pacthreshold}\emph{(Informal)}
Given a metric $\D$ and a representative $r$, there exists a set of thresholds $\T$ such that
     $t \in [0,1]$ for all $t \in \T$,
 $0 \in \T$,
and $|\T|= O(1)$,
and for every $t \in \T$ there exists an efficient learner $L_t^r(\et,\dt)$ which for all $\et,\dt \in (0,1]$, with probability at least $1-\dt$, 
produces a hypothesis $h_t^r$ such that
$\Pr_{x \sim \U}[h_t^r(x) \neq T_t^r(x)] \leq \et$ in time $O(poly(\frac{1}{\et},\frac{1}{\dt}))$ with access to labeled samples of $T_t^r(u \sim \U)$ for any distribution $\U$.
\end{assumption}
\textbf{Constructing submetric learners from threshold learners.}
Given Assumption \ref{assumption:pacthreshold}, our next step is to determine how to combine the threshold learners into a learner for the representative submetric. 
(Notice that training data for the threshold function learners can be produced by post-processing the outputs of Algorithm \ref{nips:alg:orderingtometric}.) 
Our strategy is similar to the rounding strategy used in Algorithm \ref{nips:alg:orderingtometric}, using the threshold functions to identify the largest threshold which underestimates the distance between the representative and the element under consideration.
The $\linearvote$ mechanism takes in a set of hypotheses for the thresholds and outputs the threshold with which the most hypotheses agree.
When all hypotheses output the correct value of their corresponding threshold function, $\linearvote$ is equivalent to rounding $\D(r,x)$ down to the nearest threshold.
\begin{definition}[$\linearvote$]\label{def:linearvote}
Given an ordered set of thresholds, $\T = \{t_1, t_2, \ldots, t_{|T|}\}$, and a set of hypotheses $H_\T^r = \{h_{t_1}^r, h_{t_2}^r, \ldots, h_{t_{|T|}}^r\}$, one corresponding to each threshold function, 
$\linearvote(\T,H_\T^{r},x) := \argmax_{t_i}\sum_{t_j < t_i} (1-h_{t_j}^r(x)) + \sum_{t_j \geq t_i} h_{t_j}^r(x)$.
\end{definition}

\begin{algorithm}[]
  \caption{Pseudocode}
  \label{nips:alg:learner}
\begin{algorithmic}[1]
\LineComment{Inputs: error and failure probability parameters $\varepsilon,\delta$, density parameter $b$, a set of threshold function learners, the threshold set $\T$, and interfaces to the \arbiter.}

\State Sample a set of representatives $R \sim \U$ of size $\frac{1}{b}\ln(\frac{2}{b\dL})$ to produce $\gamma-$net with Pr $\geq 1 - \frac{\dL}{2}$.

\State Generate labeled training data for the threshold learners via Algorithm \ref{nips:alg:orderingtometric}.

\State Run each threshold learner $L_{t_i}^r$ with error parameters $\et \leftarrow \frac{\eL}{2|R||\T|}$ and $\dt \leftarrow \frac{\dL}{2|R||\T|}$ to produce threshold function hypotheses $h_{t_i}^r$ for $r \in R$ and $t_i \in \T$.
\State For each representative, produce a hypothesis for distance from the representative by taking $h_r(x,y) := |\linearvote(\T,\{h_{t_i}^r | t_i \in \T \},x) - \linearvote(\T,\{h_{t_i}^r | t_i \in \T \},y)|$.

\State Combine the hypotheses for each representative into $h_R(x,y):= \maxmerge(\{h_r | r \in R\}, x,y)$.
\State \textbf{return} $h_R$.
\end{algorithmic}
\end{algorithm}

Algorithm \ref{nips:alg:learner} combines all of our constructions thus far to create an efficient submetric learner: it chooses a set of representatives, learns threshold functions for each threshold for each representative, and combines the resulting hypotheses using $\linearvote$ and $\maxmerge$ to produce a single submetric hypothesis.\footnote{Algorithm \ref{nips:alg:learner} summarizes Algorithms \ref{alg:thresholdcombiner}-\ref{alg:submetriclearner}, see Sections \ref{section:generalizinghumans} and \ref{section:choosingreps}.} Theorem \ref{nips:theorem:maingeneralization} builds on the result of Lemma \ref{nips:lemma:randomrepgeneralization} and concludes that the parametrization of Algorithm \ref{nips:alg:learner} results in an efficient submetric learner. 
\begin{theorem}\label{nips:theorem:maingeneralization}[Informal]
Given a distance metric $\D$, and a distribution $\U$ over the universe,
if
there exist a set of thresholds $\T$ with maximum gap $\alpha_\T$ and efficient learners $\{L_{t_i \in \T}^r\}$ as in Assumption \ref{assumption:pacthreshold}, and
$\D$ is $(\gamma, a, b)-$dense and
$(p,\frac{6\gamma + \alpha_\T}{1-c})-$\diffuse~on $\U$, 
then there exists an efficient $\alpha_\T$-submetric learner which produces a hypothesis $h_R$ 
such that
 $h_R$ is $(p - (1-a)^2 - \eL, c)-$nontrivial for $\U$. 
\end{theorem}
The proof of Theorem \ref{nips:theorem:maingeneralization} 
follows from an analysis of the error parameter propagation.\footnote{See the proofs of Theorems \ref{theorem:maingeneralization} and \ref{thm:querycomplexity}  for detailed analysis.}  We briefly give some intuition for the analysis and implications of the theorem. 
First, the magnitude $\alpha_\T$ error follows from the same single direction rounding argument as for Algorithm \ref{nips:alg:orderingtometric}. 
The error probability follows from noticing that at least one threshold function must be in error for one of the elements to result in an error in $\linearvote$.
The failure probability ``budget'' is split evenly between failure to choose a good set of representatives (Line 1) as specified in Lemma \ref{nips:lemma:randomrepgeneralization}, and failure of the underlying learning procedures (Line 3) derived by union bound. 
Compared with Lemma \ref{nips:lemma:randomrepgeneralization}, the diffusion and nontriviality parameters are adjusted to take into account the additional rounding error magnitude of $\alpha_\T$ introduced by $\linearvote$  
and 
the combined hypothesis error probability $\varepsilon$. 
In practice, we expect that the set of thresholds which are learnable are unlikely to occur at regular intervals. Post-processing is a valuable tool to reduce the magnitude of $\alpha_\T$ (by re-mapping the threshold values in step 4 to reduce the maximum gap), but comes at the cost of reduced nontriviality guarantees.

The desired query complexity to the \arbiter~follows from basic analysis of the parameters. However, the query complexity bound can be improved significantly by observing that no independence of errors between threshold functions is assumed, allowing a single call to Algorithm \ref{nips:alg:orderingtometric} for each representative (rather than $|\T|$ calls). The dependence on $|R|$ can also be improved to logarithmic by sorting a single merged list of (representative, element) pairs, but we defer detailed discussion to Sections  \ref{section:generalizinghumans} and \ref{section:choosingreps}.


\subsection{Relaxing the query model}\label{nips:section:relaxed}
Our results extend to a relaxed model in which \arbiters~are not expected to make arbitrarily small distinctions between distances or individuals and may answer real-valued queries with bounded noise.
The relaxed model assumes that there are two fixed constants, $\alphL$, the minimum precision with which the \arbiter~can distinguish elements or distances, and $\alphH$, a bound on the magnitude of the (potentially biased) noise in the \arbiter's real-valued responses. 
For any comparisons with difference smaller than $\alphL$, the \arbiter~declares the elements indistinguishable or the difference ``too close to call.'' The model allows for a ``gray area'' between $\alphL$ and $\alphH$ in which the \arbiter~may either respond with the true answer or ``too close to call.'' For any differences larger than $\alphH$, the arbiter responds with the true answer. 


For the most part, our results translate to the relaxed model with minimal modification to the logic of the proofs to handle two-sided error in real-valued queries. Interestingly, the real-value query complexity improves to constant, as the worst-case behavior in Algorithm \ref{nips:alg:orderingtometric} is avoided as the \arbiter~``knows'' not to worry about inconsequentially small distances. However, this does result in additional error magnitude, so the improved query complexity does not come for free. 
Furthermore, 
unlike the \exact~model we won't necessarily be able to label a sample with perfect accuracy for every threshold function learner due to the bi-directional error.
To handle this labeling problem, we modify the distribution of samples presented to each learner, eliminating samples whose labels are ambiguous, again resulting in increased error. 
Formal results in the relaxed model are discussed in Section \ref{section:tctc}.

\section{Related Work}\label{section:related}
Metric learning is a richly studied area. Two surveys \cite{DBLP:journals/corr/BelletHS13,kulis2013metric} provide an overview of the literature unrelated to Individual Fairness.
There is a significant body of literature concerned with learning distance metrics from human feedback in practice with heuristic optimization for applications like image similarity, feature identification and other applications including \cite{frome2007learning}, \cite{adaptively-learning-crowd-kernel}, \cite{zou2015crowdsourcing}, \cite{jamieson2011low}, \cite{wilber2014cost}, \cite{van2012stochastic}.

With respect to constructing metrics for Individual Fairness or generalizing individually fair classifiers to unseen samples, we highlight four recent works.
\cite{gillen2018online} considers an online linear contextual bandits setting, and imposes a fairness constraint that similar contexts should be treated similarly, where similarity is assumed to be a Mahalanobis distance. 
\cite{gillen2018online} takes a similar view of human feedback for learning fairness to ours, but their online setting, fairness model, metric assumptions and goals are different. The work most similar to ours of Jung et al. (\cite{jung2019eliciting}) has very similar motivation, but their model is restricted to equivalence (or near equivalence) queries. They consider the problem of arbiter consistency and explicitly consider the multiple arbiter model in their empirical work. The equivalence model considered in \cite{jung2019eliciting} can be expressed in the relaxed arbiter model of this work, (i.e., allowing a large too-close-too-call region and allowing for an appropriately sized noise parameter to place no requirement on arbiters reporting values other than ``not equal''). Applying the results of this work to the multi-arbiter empirical model proposed by Jung et al., either by attempting to elicit more nuanced judgments beyond equivalence or to better understand the properties of the equivalence-only model versus the more general relaxed model, is an exciting direction for future work.
\cite{rothblum2018probably} and \cite{kim2018fairness} consider the problem of generalizing Individual Fairness with differing levels of oracle access to the metric, and  one could view our results as providing a path for efficiently generating metric samples for these settings. Our notion of a submetric is similar to $(d,\tau)$ metric fairness of \cite{kim2018fairness}, and our definition of efficient submetric learner is very close to the definition of ``approximately metric-fair'' of \cite{rothblum2018probably}.
We view the present work as a complement to these directions.

With respect to query types and human fairness judges, as previously noted Gillen et al, \cite{gillen2018online}, consider a similar model in which a human judge `knows unfairness when she sees it.' Dasgupta and Luby, \cite{dasgupta2017learning} also consider the benefits of ``partial feedback'' from a human expert in clustering applications with very similar motivation to our query type choices.

The problems of ranking and scoring are closely related to the problem of combining arbiter judgments to construct orderings based on relative queries. In this work, we did not address how to handle differences in orderings between arbiters. However, there is a significant body of work concerned with aggregating or combining orderings or rankings from multiple sources. For example, Dwork, Kumar, Naor and Sivakumar consider the problem of combining rankings from multiple sources in \cite{dwork2001rank}. Volkovs, Larochelle and Zemel consider rank aggregation as a supervised learning problem, and consider questions of crowd-sourcing in \cite{volkovs2012learning}. Dwork, Kim, Reingold, Rothblum and Yona consider the fairness and accuracy properties of rankings in \cite{DKRRY2019}.



\section{Additional definitions and terminology}\label{section:preliminaries}
In addition to the preliminary terminology and definitions introduced in Section \ref{nips:section:intro}, there are several additional definitions and dilemmas which will prove useful in the more technical discussion of later sections. In particular, we introduce one additional relative query type, the concept of a consistent underestimator () and explicit characterization of representative submetrics and representative set submetrics.

\paragraph{Expanded Query Model}

As mentioned in Section \ref{nips:section:intro}, we restrict ourselves to queries involving a limited number of elements. In addition to the triplet query, we consider a second type of relative query, the quad query, which asks the arbiter to compare distances between two distinct pairs. 


\newtheorem*{def:realquery}{Definition \ref{def:realquery}}
\begin{def:realquery}[Real query]
$\mreal(u,v):=\D(u,v)$
\end{def:realquery}

\newtheorem*{def:tripletquery}{Definition \ref{def:tripletquery}}
\begin{def:tripletquery}[Triplet query]
$\mtrip(a,b,c):=\{
1 \text{ if }\D(a,b) < \D(a,c)\text{, }0 \text{ if }\D(a,c)\leq \D(a,b)\}$.
\end{def:tripletquery}

\begin{definition}[\Quadq]
$\mquad(a,b,x,y):= \{1 $ if $\D(a,b)>\D(x,y)$,  $0 $ otherwise$\}$.
\end{definition}

Relative distance queries have been used successfully for human evaluation in image processing and computer vision, e.g. \cite{van2012stochastic,wilber2014cost}. 
A \quadq~does require the \human~to consider an additional element compared with a triplet query. This may result in additional overhead for the \human, particularly in cases where examining each individual requires significant time. As such, we consider \quadqs~slightly more costly than triplet queries, but still significantly less costly than real-valued distance queries.
Furthermore, the binary nature of the response makes parallelizing relative distance queries between several \humans~(who are in agreement) straightforward.

For brevity in algorithms and theorem statements we will refer to $\mreal$, $\mtrip$ and $\mquad$ as the interfaces to the \arbiter.

\paragraph{Additional definitions and lemmas.}
We now introduce several additional definitions and lemmas which simplify discussion in later technical sections.



A core component of this work is that submetrics can be constructed based on distance information from a single representative element. We refer to the submetric constructed from differences in distance to a particular representative $r$ as $\Dr$.
\begin{definition}[Representative Submetric]\label{def:repsubmetric}
Given a representative $r \in U$, we define the submetric $\Dr(u,v) := |\D(r,u) - \D(r,v)|$ for all $u,v \in U$.
\end{definition}
The following straightforward lemma and proof explicitly, restated from the introduction, show that $\D_r$ as defined is a \zsubmetric~of $\D$.

\newtheorem*{lemma:repsubmetric}{Lemma \ref{nips:lemma:repsubmetric}}
\begin{lemma:repsubmetric}[Restatement]
$\D_r(u,v):=|\D(r,u)-\D(r,v)|$ for $r,u,v \in U$ is a \zsubmetric~of $\D$.
\end{lemma:repsubmetric}
\begin{proof}
The proof follows from triangle inequality. $\D(r,u) \leq \D(r,v)+\D(u,v)$, thus $\D(r,u) - \D(r,v) \leq \D(u,v)$. Likewise, $\D(r,v) - \D(r,u) \leq \D(u,v)$. Thus $|\D(r,u)-\D(r,v)| \leq \D(u,v)$.
\end{proof}

Lemma \ref{nips:lemma:repsubmetric} shows that $\Dr(x,y)$ constructed from \textit{exact} evaluations of $|\D(r,x)-\D(r,y)|$ is a submetric, but in practice we will want to construct submetrics from approximate evaluations of $\D(r,x)$ and $\D(r,y)$.
Just as a submetric is a contraction of the true metric, a \textit{consistent underestimator} is a contraction of the distance between a representative and other elements of the universe. The key property of a consistent underestimator is that distances are contracted \textit{consistently}, i.e. a weak ordering of distances from $r$ is preserved.
\begin{definition}[Consistent Underestimator]\label{def:consistentunderestimator}
Given a universe $U$ and a metric $\D: U \times U \rightarrow [0,1]$,
a function $f_r: U \rightarrow [0,1]$ is said to be an \textit{$\alpha-$consistent underestimator} for $r$ with respect to $\D$ if for all $u,v \in U$
\begin{align*}
\text{1. }f_r(u) &\leq \D(r,u)  &  \text{2. } f_r(u) &\leq f(v) \text{ iff } \D(r,u) \leq \D(r,v)    & \text{3. }
|f_r(u) - f_r(v)| &\leq |\D(r,u) - \D(r,v)| + \alpha\end{align*}
We define the maximum contraction of a consistent underestimator $\maxcontraction := \max_{u,v \in U \times U}\D(u,v) - |f_r(u) - f_r(v)|$.
\end{definition}

Analogous to the construction of $\D_r$, an $\alpha-$consistent underestimator for a representative $r$ can also be used to construct an $\alpha-$submetric, denoted $\Dconst$. Figure \ref{fig:subcomparison} illustrates the difference in the exact evaluation of $\D(r,x)$ versus the consistent underestimator.
\begin{definition}[Representative Consistent Underestimator Submetric]\label{def:repunderestimator}
Given a representative $r$ and $f_r$, an $\alpha-$consistent underestimator for $\Dr$, we define the $\alpha-$submetric $\Dconst(u,v) := |f_r(u) - f_r(v)|$ for all $u,v \in U$.
\end{definition}
In the following lemma and proof, we explicitly state and show that the construction of $\Dconst$, as specified in Definition \ref{def:repunderestimator}, results in an $\alpha-$submetric.
\begin{lemma}\label{lemma:consistenterror}
Given an $\alpha-$consistent underestimator $f_r$ for $r$ with respect to $\D$,  $\Dconst(u,v) := |f_r(u) - f_r(v)|$ is an $\alpha-$submetric of $\D$.
\end{lemma}
\begin{proof}
Notice that $|f_r(u) - f_r(v)| \leq |\D(r,u) - \D(r,v)| + \alpha$ by Definition \ref{def:consistentunderestimator} (3), and that $|\D(r,u) - \D(r,v)| \leq \D(u,v)$ by triangle inequality.
\end{proof}

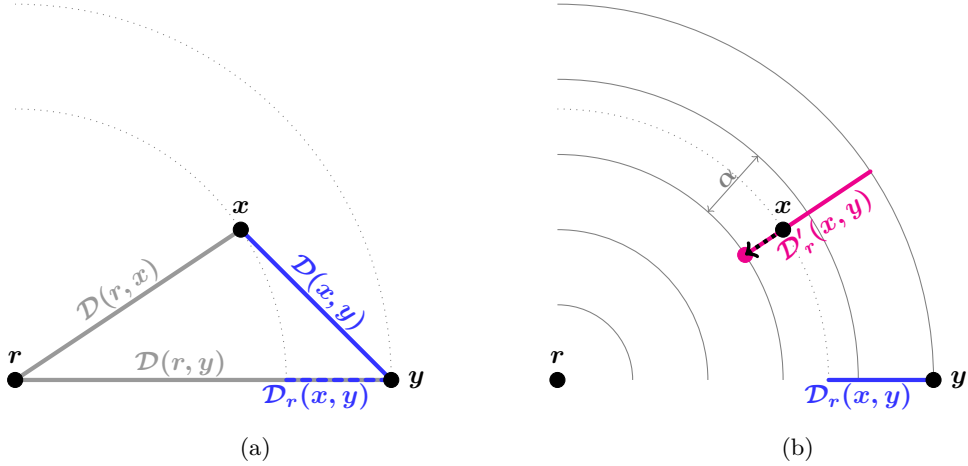
\begin{figure}
    \centering
    \begin{subfigure}[b]{0.45\textwidth}

\begin{tikzpicture}[font=\boldmath]

    \coordinate (A) at (0,0) {};
    \coordinate (B) at (5,0) {};
    \coordinate (C) at (3,2) {};
    \coordinate (D) at (3.605,0) {};

    \node at (A) [above = 1mm] {$r$};
    \node at (B) [right = 1mm ] {$y$};
    \node at (C) [above = 1mm ] {$x$};

    \draw[ultra thick, gray!80, arrows={}, line cap=round]  (A) -- (C) node[sloped,midway,above=-0.1cm] {$\D(r,x)$};
    \draw[ultra thick, blue!80,      arrows={}, line cap=round]  (C) -- (B) node[sloped,midway,above=-0.1cm] {$\D(x,y)$};
    \draw[ultra thick, gray!80, arrows={},line cap=round]  (A) -- (B) node[sloped,midway,right=-0.3cm,above=-0.1cm] {$\D(r,y)$};
     \draw [gray,dotted] (5,0) arc [radius=5, start angle=0, end angle= 90];
     \draw [gray,dotted] (3.605,0) arc [radius=3.605, start angle=0, end angle= 90];
    \draw[dashed, ultra thick, blue!80, line cap=round] (D) -- (B) node[sloped,midway,right=-0.3cm,below=-0.1cm] {$\D_r(x,y)$};

    \draw [fill] (A) circle [radius=.1];
    \draw [fill] (B) circle [radius=.1];
    \draw [fill] (C) circle [radius=.1];
\end{tikzpicture}
\caption{}
        \label{fig:exact}
    \end{subfigure}
    ~ 
    \begin{subfigure}[b]{0.45\textwidth}

\begin{tikzpicture}[font=\boldmath]

    \coordinate (A) at (0,0) {};
    \coordinate (B) at (5,0) {};
    \coordinate (C) at (3,2) {};
    \coordinate (D) at (2.4961,1.6641) {};
    \coordinate (E) at (3,0) {};
    \coordinate (F) at (3.605,0) {};
    \coordinate (G) at (4.1603,2.7735) {};
    \coordinate (a1) at (2,2.236) {};
    \coordinate (a2) at (2.666,2.981) {};

    \node at (A) [above = 1mm ] {$r$};
    \node at (B) [right = 1mm ] {$y$};
    \node at (C) [above = 1mm ] {$x$};
    \node at (D) [left = 1mm ] {};

     \draw [gray] (5,0) arc [radius=5, start angle=0, end angle= 90];
     \draw [gray] (4,0) arc [radius=4, start angle=0, end angle= 90];
     \draw [gray] (3,0) arc [radius=3, start angle=0, end angle= 90];
     \draw [gray] (2,0) arc [radius=2, start angle=0, end angle= 90];
     \draw [gray] (1,0) arc [radius=1, start angle=0, end angle= 90];
     \draw [gray,dotted] (3.605,0) arc [radius=3.605, start angle=0, end angle= 90];
    \draw [fill, magenta] (D) circle [radius=.1];

    \draw[ultra thick, blue!80, line cap=round] (F) -- (B) node[sloped,midway,right=-0.3cm,below=-0.1cm] {$\D_r(x,y)$};

    \draw[ultra thick, magenta, line cap=round] (G) -- (D) node[sloped,midway,right=0.1cm,below=-0.05cm] {$\D_r'(x,y)$};
    \draw[ultra thick, black, dotted, arrows={->}] (C) -- (D);
    \draw[gray, arrows={<->}] (a1) -- (a2) node[sloped, midway, above=-0.1cm] {$\alpha$};
    \draw [fill] (A) circle [radius=.1];
    \draw [fill] (B) circle [radius=.1];
    \draw [fill] (C) circle [radius=.1];

\end{tikzpicture}
\caption{}
    \label{fig:approx}
    \end{subfigure}

    ~ 

    \caption{(a) illustrates that the $0-$submetric $\Dr(x,y):=|\D(r,x) - \D(r,y)|$ never overestimates $\D(x,y)$ by triangle inequality.
    (b) illustrates the $\alpha-$submetric $\Dconst$, which is based on rounding the distances between each element and $r$ down to fixed thresholds at $\alpha$ granularity. Notice that $x$ is rounded down from its original distance to the nearest threshold, whereas $y$ is already exactly on a threshold and is not changed by the rounding procedure. In this case, although $\Dconst(x,y) >\Dr(x,y)$, $\Dconst(x,y)$ is still less than $\D(r,y)$.
    } \label{fig:subcomparison}
\end{figure}

We  use ``submetric'' to refer to submetrics in general with unspecified $\alpha$, and \zsubmetric~to explicitly reference submetrics with no additive error.
Of course, given an $\alpha$-submetric it is possible to produce a \zsubmetric~via postprocessing.
\begin{proposition}\label{prop:post}
Given an $\alpha$-submetric $\D'$ of $\D$, the submetric $\D_{1}'(x, y):= \max\{0,\D'( x, y)-\alpha\}$ is a \zsubmetric~of $\D$.
\end{proposition}

Now that we have specified submetrics based on exact distances from a representative and a consistent underestimator for a representative, we consider the nontriviality properties of these submetrics.
Notice that although the overestimate magnitude $\alpha$ is independent of $r$, the distances \textit{preserved} are highly dependent on the choice of $r$. (See Figure \ref{fig:repcomparison}.)
$\D_r$ exactly preserves the distance between $r$ and every $u \in U$, so we can conclude that $\D_r$ is $(\frac{1}{N},1)-$nontrivial for $\unifu$. (Notice that $r$ has a $\frac{1}{N}$ probability of selection in $\unifu$).
Likewise, $\D_r'$ with maximum contraction $\maxcontraction$ will preserve $\D(r,u) - \maxcontraction$ for all $u \in U$. Thus we can relate nontriviality for a distribution $\U$ to $\Pr_{u \sim \U}[\D(r,u)>\maxcontraction]$.
However we cannot make guarantees on distance preservation for distances between arbitrary pairs in $\U \times \U$ under $\Dr$ or $\Dconst$ without further information.
For example, $u$ and $v$ may be equally distant from $r$, so $\D_r(u,v)=0$, but may also be equally distant from each other.
Up until Section \ref{section:choosingreps}, we will conservatively consider nontriviality only as a function of distances preserved between pairs in $U \times \{r\}$ to focus our attention on learning approximations of $\D_r$ and $\D_r'$ which generalize to unseen samples. In Section \ref{section:choosingreps}, we return to this question and formulate the necessary properties of $\U$ and $\D$ to reason more generally about distance preservation and nontriviality.

As previously illustrated in Figure \ref{fig:repcomparison}, submetrics constructed based on different representatives will preserve different information about the true underlying metric. 
We therefore consider constructing submetrics by aggregating information from multiple representatives.
The following lemma and corollaries state that arbitrary mixing of submetrics, for example taking the maximum distance between a pair of elements in a set of submetrics, will result in a valid submetric and furthermore, the resulting submetric will have at most the additive error of the input submetrics.
\begin{lemma}\label{lemma:submerge}
Given a set of submetrics $\{\D_i | i \in [k]\}$ for $\D$, for any arbitrary mapping $F: U \times U \rightarrow [k]$,
$\D_{merge}(u,v):=\D_{F(u,v)}(u,v)$
is a submetric of $\D$.
\end{lemma}
\begin{proof}
Notice that for all $i \in [k]$, $\D_{i}(u,v) \leq \D(u,v)$, and thus $\D_{merge}$ is also a submetric.
\end{proof}
\begin{corollary}
\label{cor:maxmerge} Given a set of submetrics $\{\D_i | i \in [k]\}$ for $\D$, define $\maxmerge(\{\D_i | i \in [k]\},u,v) := \max_{i \in [k]}\D_i(u,v)$. The $\maxmerge$ of a set of submetrics of $\D$ is a submetric of $\D$.
\end{corollary}
\begin{corollary}
\label{cor:maxmergeconst} Given a set of $\alpha-$submetrics $\{\D_i' | i \in [k]\}$ for $\D$, The $\maxmerge$ of $\{\D_i'\}$ is an $\alpha-$submetric of $\D$.
\end{corollary}

Throughout this work, we will use the $\maxmerge$ of a set of representative submetrics as a core of our constructions. Below, we define $\DR$ and $\DconstR$ based on a set of representatives $R$.
\begin{definition}[Representative Set Submetric]
Given a set of representatives $R \subseteq U$, we define the representative set submetric $\D_R(u,v):=\maxmerge(\{\Dr | r \in R\},u,v)$ and representative set consistent underestimator submetric $\DconstR(u,v):\maxmerge(\{\Dconst | r \in R\},u,v)$ for all $u,v \in U$.
\end{definition}

\section{From human judgments to submetrics}\label{section:humansubmetrics}
In this section we consider the problem of determining which and how many queries to ask of our \human~in order to construct a submetric for a pre-specified universe $U$.
This setting can be viewed either as the problem of learning a metric over a fixed universe (e.g., determining a metric over the entire set of college applicants in a particular year) or as a process for generating training data to learn a submetric which generalizes to unseen samples as in Section \ref{section:generalizinghumans}, or as input to any other method for learning fairly with access to a sample of distance, e.g. \cite{kim2018fairness,rothblum2018probably}.

Naively, we could construct $\Dr$ with $O(N)$ queries to $\mreal$ by simply querying for every distance from the representative $r$. Furthermore, $\DR$ can be constructed in the same way by issuing $O(|R|N)$ queries to $\mreal$ to constuct each representative submetric and then merging.
Although the linear dependence on $N$ may seem good compared with $O(N^2)$, we anticipate that the cost of real-valued queries is high and increases with the number of queries. Although the number of queries is linear in $N$, the \textit{cost} in terms of human effort may not be.

We now work towards constructing a submetric from a sublinear
number of real-valued queries by supplementing with $O(N \log(N))$ triplet queries, at the cost of introducing bounded additive error.
Our general strategy will be to show that given an \textit{ordering} consistent with the metric, we can learn a \submetric~from a constant or sublinear number of queries to $\mreal$ by rounding distances from the each representative down to fixed thresholds.
More concretely,
a representative consistent ordering for $r$ is an ordered list of elements from smallest to largest distance from $r$.
\begin{definition}[Representative-consistent ordering]\label{defn:consistentordering}
An ordering $\ord=\{r,x_1,x_2, \ldots \}$  of 
elements of $U$ is consistent with the representative element $r$ with respect to the metric $\D$ if for all  $i<j$, $\D(r,x_i) \leq \D(r,x_j)$.

\end{definition}

Given the notion of a representative consistent ordering, we now show that rounding down
to ``threshold'' distances at granularity $\alpha$ is sufficient to produce an $\alpha-$submetric. Threshold rounding is also very useful for \textit{preserving} distances and can be helpful for generalization, as we will see in Sections \ref{section:generalizinghumans} and \ref{section:choosingreps}. We now formally define a Threshold Consistent Underestimator and prove a bound on the maximum contraction and additive error.

A threshold consistent underestimator is the function which rounds down the distance between and element $x \in U$ and a fixed representative $r$ to the nearest threshold in a prespecified set.
\begin{definition}[$\alpha-$consistent threshold underestimator]\label{def:tcunderestimator}
Given a universe $U$, a metric $\D: U \times U \rightarrow [0,1]$, a representative $r \in U$, and an ordered set of distinct thresholds $\T = \{0,t_1,\ldots,t_k\}$ (for constant $k$) where $t_i \in [0,1]$,
\[\threshunder(x) := \argmax_{t_i \in \T}\{t_i \leq \D(r,x) \}\]
is the threshold consistent underestimator wrt $\D,r,$ and $\T$. We refer to the maximum distance between any adjacent thresholds in $\T$ as $\alpha_\T := \max_{i \in [|\T|-1]}t_{i+1} - t_{i}$.
\end{definition}

It is simplest to consider $\T$ to consist of a set of evenly spaced thresholds at granularity $\alpha_\T$, although the analysis does not depend on this and certainly allows varied threshold spacing.
Lemma \ref{lemma:tunderestimator} formally states that an $\alpha_\T$-consistent threshold underestimator $\threshunder$ is an $\alpha_\T$-consistent underestimator that has contraction of distances between an element and the representative $r$ of at most $\maxcontraction=\alpha_\T$. 
\begin{lemma}\label{lemma:tunderestimator}
Given 
$\threshunder(u)$ for $u \in U$, an $\alpha_\T-$consistent underestimator, where $\alpha_\T = \max_{i \in [|\T|-1]}t_{i+1} - t_{i}$. The submetric $\Dconst:= |\threshunder(u) - \threshunder(v)|$, is an $\alpha_\T-$submetric with maximum contraction with respect to $\Dr$ bounded by $\alpha_\T$, ie
$\Dr - \Dconst \leq \alpha_\T$.
\end{lemma}
\begin{proof}
By definition, $\threshunder(u)$ satisfies conditions 1 and 2 of a consistent underestimator.
Notice that $u$ and $v$ have distances from $r$ reduced from $\D(r,u)$ and $\D(r,v)$ by rounding down by at most $\alpha_\T$, and thus $|\threshunder(u) - \threshunder(v)| \in [\Dr(u,v) - \alpha_\T,\Dr(u,v) + \alpha_\T]$, satisfying the third condition of a consistent underestimator and the bound on the maximum contraction of $\Dconst$.

\end{proof}

This property of threshold consistent under estimators implies that if we can construct an ordering of the elements with respect to their distance from the representative and then label the elements at regular intervals, then we can produce a consistent underestimator.


\subsection{Constructing metric consistent orderings}\label{section:orderings}
We can construct a metric consistent ordering by using $\mtrip$ as a comparator, as $\mtrip(r,x,y)$ indicates which of $x$ or $y$ has greater distance from $r$. Using such a comparator, we can build an ordered list via binary search.\footnote{In practice, there may be many other simple to evaluate query types which can also be used to produce an ordering. We focus on Triplet Queries as they have some existing usage in the literature, but these results can generalize to any query type which can be used to generate an ordering or as a comparator.} This procedure is detailed in Algorithm \ref{alg:tripletordering}.




\begin{algorithm}[]
  \caption{\showalgcaption{$\mathsf{TripletOrdering}(\mtrip,U,r)$}}
  \label{alg:tripletordering}
\begin{algorithmic}[1]
  \State {\bfseries Input:} the universe $U$, a representative $r \in U$.
  \Procedure{$\mathsf{TripletOrdering}$}{$\mtrip,U,r$}
  \State Initialize an empty list $\ord$
  \State Append $\{r\}$ to $\ord$
  \While{$U$ is not empty}
    \State $\mathsf{BinaryInsert}(r,\ord,0,\size(\ord),pop(U),\mtrip)$
  \EndWhile
  \State \textbf{return} $\ord$
  \EndProcedure
  \State
  \LineComment{Inputs: $L$ an ordered list of element sets, the list range delimiters $b$ and $e$, the element to insert $x$, an interface to the \human, $\mtrip$.}
  \Function{$\mathsf{BinaryInsert}$}{$r,L,b,e, x,\mtrip$}
   \If{$b=e$}
     \State check relative position compared with $L[b]$.
     \If{$\mtrip(r,x,L[b])$=0}
         \State Insert $x$ at position $b+1$ in $L$
     \Else
         \State Insert $x$ at position $b-1$ in $L$
     \EndIf
     \State insert at relative position
     \State \textbf{return}
   \EndIf
  \State $\midpoint \leftarrow \midpointof(L)$
  \State $c\leftarrow \mtrip(r,x,L[\midpoint]))$
  \If{$c =0$}
    \State $\mathsf{BinaryInsert}(r,L,\midpoint,e,(r,x),\mtrip)$
  \Else
    \State $\mathsf{BinaryInsert}(r,L,b,\midpoint,(r,x),\mtrip)$
  \EndIf
  \EndFunction

\end{algorithmic}
\end{algorithm}

\begin{lemma}\label{lemma:tripletordering}
Given a universe $U$ and a representative $r \in U$, Algorithm \ref{alg:tripletordering} produces a representative consistent ordering $L$ for $r$ from $O(N \log(N))$ queries to $\mtrip$.
\end{lemma}
The proof follows from a straightforward analysis of the binary search procedure with $\mtrip$ used for comparisons.
\begin{proof}
We consider correctness and query complexity separately.

\paragraph{Query complexity.} Notice that $\binaryinsert$ is called $N$ times, once for each element. Each recursive call to $\binaryinsert$ eliminates at least half of the sets in $L$ under consideration, and so $\binaryinsert$ has recursion depth of $O(\log(N))$. Each recursive call to $\binaryinsert$ makes a single call to $\mtrip$.
Thus the total number of queries to $\mtrip$ is $O(N\log(N))$.

\paragraph{Correctness.} Each element is inserted into an ordered list via binary search, and as such every element earlier in the list is at least as close to $r$ as any element later in the list.
\end{proof}

\subsection{Constructing $\alpha-$submetrics from orderings}

Algorithm \ref{alg:orderingtometric}, below, outlines the process of labeling and ordering by distance from the representative at a particular granularity, $\alpha$. Algorithm \ref{alg:orderingtometric} repeatedly splits the input ordering into contiguous ranges of elements until the difference in distances between the first and last elements in the range to the representative are at most $\alpha$. Once each  range has reached the appropriate size, the distance between each element in the range and the representative is then set to the minimum distance in its range, which maintains a weak ordering of distances from the representative and corresponds to rounding $\D(r,x)$ down by no more than $\alpha$.\footnote{All arrays are indexed from $0$ in all algorithms.}

\begin{algorithm}[]
  \caption{\showalgcaption{$\orderingtosubmetric(\ord, r, \alpha, \mreal)$}}
  \label{alg:orderingtometric}
\begin{algorithmic}[1]
  \LineComment{Inputs: the representative $r$, a representative consistent ordering $\ord$ consistent with $r$, an error parameter $\alpha$, an interface to the \human, $\mreal$. Returns a submetric $D_r': U\times U \rightarrow [0,1]$.} 
  \Procedure{$\orderingtosubmetric$}{$\ord,r,\alpha,\mreal$}
  \State Initialize $f_r(x_i) = 0$ for all $x_i \in \ord$
  \State $\mathsf{SplitList}(\ord,r,\alpha,\mreal,f_r)$
  \State \textbf{return} $\D_r'(x,y):= |f_r(x) - f_r(y)|$
  \EndProcedure
  \State
  \Function{$\mathsf{SplitList}$}{$\ord,r,\alpha, \mreal,f_r$}
  \State $d_{bottom} = \mreal(\ord[0],r)$
  \State $d_{top} = \mreal(\ord[\size(\ord)-1],r)$
  \If{$d_{top} - d_{bottom} \leq \alpha$}
  \ForAll{$x_i \in \ord$}
  \State Set $f_r(x_i)=d_{bottom}$
  \EndFor
  \Else
  \State $\midpoint \leftarrow \midpointof(\ord)$
  \State $\mathsf{SplitList}(\ord[0,\midpoint], r, \alpha, \mreal,f_r)$
  \State $\mathsf{SplitList}(\ord[\midpoint,\size(\ord)-1], r, \alpha, \mreal,f_r)$
  \EndIf
  \EndFunction

  \end{algorithmic}
\end{algorithm}

Lemma \ref{lemma:orderingtometric} states that given a representative consistent ordering,
an $\alpha-$submetric can be constructed via Algorithm \ref{alg:orderingtometric} with $O(\maxalphlog)$ queries to $\mreal$.
Algorithm \ref{alg:orderingtometric} utilizes the representative consistent ordering to make fewer queries to $\mreal$ by labeling elements in the ordering with distances at granularity $\alpha$ from $r$ and rounding intermediate elements to produce an $\alpha-$consistent threshold underestimator, which is then used to construct an $\alpha-$submetric.\footnote{In Algorithm \ref{alg:orderingtometric} we use ``Set'' and ``Initialize'' to mean setting or initializing a global copy of $f_r$ to avoid tedious bookkeeping. We also use $\midpointof$ to specify the midpoint function, which chooses the midpoint for odd length lists and rounds down for even length lists.}

\begin{lemma}\label{lemma:orderingtometric}
Given a universe $U$, and an ordering $\ord$ consistent with a representative $r \in U$ for a metric $\D$, Algorithm \ref{alg:orderingtometric} produces an \asubmetric~of $\D$ which preserves the distance between each element $u \in U$ and $r$ (with additive error $\alpha$) from $O(\maxalphlog)$ queries to $\mreal$.
\end{lemma}

\begin{proof}
We address query complexity and error magnitude separately for clarity.

\textbf{Query complexity.}
At most 2 queries to $\mreal$ are made per call of $\mathsf{SplitList}$.
Thus to analyze query complexity, it is sufficient to analyze the number of calls to $\splitlist$.
There are three conditions in which $\splitlist$ makes additional recursive calls:
\begin{enumerate}
    \item $\splitlist$ makes two calls which immediately terminate, i.e. both sides of the split represented ranges with $d_{top}-d_{bottom} \leq \alpha$.
    \item $\splitlist$ makes two calls which do not immediately terminate,  i.e. both sides of the split represented ranges with $d_{top} - d_{bottom}> \alpha$.
    \item $\splitlist$ makes one immediately terminating call and one not immediately terminating call, i.e. one side of the split represented a range of $> \alpha$ and the other $\leq\alpha$
\end{enumerate}

Notice that at any point we have identified some number of ranges of size at least $\alpha$, call this number $k$, and some number of elements left to be labeled, call this number $m$. There are at most $\frac{1}{\alpha}+1$ disjoint continuous ranges of at least size $\alpha$ in $[0,1]$, so $k \leq \frac{1}{\alpha} + 1$.
Likewise, for additional calls to be made $m$ must be greater than $0$.

Consider how each call type changes $k$ and $m$:
\begin{enumerate}
    \item Type 1 calls decrease $m$ by $\frac{N}{2^{i}}$ where $i$ is the recursion depth of the call, as every element in the current range will be labeled in the next step, and may increase $k$ by at most 1.
    \item Every type 2 call increases $k$ by 1, as an existing range of size $>\alpha$ is split into two disjoint ranges of size $>\alpha$.
    \item Type 3 calls decrease $m$ by $\frac{N}{2^{i+1}}$, as $\frac{1}{2}$ of the current range will be labeled in the next step.
\end{enumerate}

If all of the calls to $\splitlist$ are type 1 or 2, then at most $O(\frac{1}{\alpha})$ calls are made as there are at most $O(\frac{1}{\alpha})$ disjoint continuous ranges of length $\alpha$ in $[0,1]$. 
If all calls to $\splitlist$ are type 1 or type 3, then at most $\log(N)$ calls are made as at least half of the elements in the range are labeled by the subsequent terminating call(s).

If there are a mix of all three types, notice that there can still be at most $O(\frac{1}{\alpha})$ calls in the entire recursive tree of type 2. Thus it remains to consider how mixing type 3 calls with type 2 calls impacts the total number of calls.

As a warm-up, suppose a type 3 call is issued at depth $i$. If all of its children are type 1 or 3 calls, then there can be at most $O(\log(\frac{N}{2^{i+1}}))$ children, as the parent call and each child call labels at least $\frac{1}{2}$ of its range.
Now suppose a type 2 call is issued at depth $i$. If its children are all type 1 or 3 calls, then there can be at most $O(2\log(\frac{N}{2^{i+1}}))$ children, as the parent call spawns \textit{two} sub-trees with initial size $\frac{N}{2^{i+1}}$ as opposed to just one. 
Therefore we know the worst case sequence of calls includes both type 2 calls and type 3 calls.

We now show that the worst case recursion tree has no type 2 calls as children of type 3 calls.
Suppose for the sake of contradiction that a valid recursion tree $T$ has a node $A$ at depth $i$ of type 3 with a child $B$ of type 2.
Recall that type 3 nodes have only one child which does not immediately terminate, but type 3 nodes have two.
Call $B$'s children $B_1$ and $B_2$. At depth $i$, $m$ increases by $\frac{N}{2^{i+1}}$, as half of its elements are set to be labeled by the type 3 node A. At depth $i+1$, no additional elements are set to be labeled, but $k$ increases to $k+1$.

Now consider an alternative tree, $T'$ which is identical to $T$ in every way except: (1) Node $A$ is changed to type 2, (2) Node $A$ has two new children $A'_1$ and $A'_2$ of type 3, (3) $A'_1$'s non-terminating child is $B_1$ and $A'_2$'s non-terminating child is $B_2$. At depth $i$, $k$ increases to $k+1$. At depth $i+1$, $A'_1$ and $A'_2$ each set $\frac{N}{2^{i+2}}$ elements to be labeled, so $m$ increases by $\frac{N}{2^{i+1}}$. Thus, $T'$ is a valid recursion tree, but it exceeds the number of calls in $T$ by one.

Thus the worst case recursion tree will have some constant number of type 2 nodes in the highest levels which transition to type 3 and 1 nodes in the deeper levels. Suppose the type 2 nodes reach depth $\rho$, where $2^\rho$ is bounded by $O(\frac{1}{\alpha})$ as the number of type 2 nodes is a constant bounded by $O(\frac{1}{\alpha})$. Then there will be $2^\rho$ nodes at depth $\rho$ with $\frac{N}{2^\rho}$ elements in the range of each node. Each node can have at most $O(2\log(\frac{N}{2^{\rho+1}}))$ type 1 or 3 descendants, so the total number of nodes in the recursion tree is $O(2^{\rho+1}(\log(N) - \log(2^{\rho + 1})) = O(\log(N))$.


However, the worst case analysis above must consider the most pathological cases. Notice that for every type 3 query made, there must have been half of the elements in the range in a clump of distances from $r$ with less than $\alpha$ difference and the other half with distance greater than $\alpha$. If distances from each representative are distributed more smoothly, then this is unlikely to happen too many times.

\paragraph{Overestimate Error}
To reason about the error, notice that $f_r(x_i) \in [\D(r,x_i)-\alpha, \D(r,x_i)]$, as each element's distance from $r$ is rounded down by at most $\alpha$. Thus $f_r$ is an $\alpha-$consistent underestimator (and also a threshold consistent underestimator) and the final construction of $\D_r'$ is an $\alpha-$submetric by Lemma \ref{lemma:consistenterror}.
\end{proof}

The primary benefit of a sublinear number of queries to $\mreal$ is that the \human~needs to maintain consistency with a smaller set of previous outputs. Furthermore, \humans~may only be able to answer real-valued queries to within some minimum granularity, and stating the granularity up front may help them avoid wasting time verifying the consistency of ultimately inconsequentially small distance adjustments.\footnote{See Section \ref{section:tctc} For more complete treatment of a model in which the \arbiter~has limited distinguishing power.}
The following theorem, which states that a representative consistent underestimator submetric $\Dconst$ can be constructed in a sublinear number of queries to $\mreal$ and $O(N\log(N))$ queries to $\mtrip$, is an immediate consequence of Lemmas \ref{lemma:tripletordering} and \ref{lemma:orderingtometric}.

\begin{theorem}\label{theorem:labelqueries}
Given access to $\mreal$ and $\mtrip$, an $\alpha-$\submetric~can be constructed from $O(\maxalphlog)$ queries to $\mreal$ and $O(N\log(N))$ queries to $\mtrip$ which preserves distances (up to the additive error) from a representative $r$.
\end{theorem}

As before, we can also expand the expressiveness of the submetric by using $\maxmerge$, while still maintaining the same small additive error bound. Naively, this could be accomplished in $O(|R|\maxalphlog)$ queries to $\mreal$ given orderings for a set of representatives $R$ by applying Algorithm \ref{alg:orderingtometric} independently on each representative's ordering. However, the linear dependence on $|R|$ can be improved by using our third query type, \quadqs.

To see this, notice that the orderings, $\{\ord_r | r \in R\}$ can be merged into a \textit{single} ordering by distance from representative using \quadqs. To compare two elements from different lists, $\mquad((r_i,x),(r_j,y))$ will suffice to determine which is closer to its respective representative. Thus, we can use any standard sorted list merging approach to combine the sorted lists with respect to each specific representative $\{\ord_r| r \in R\}$ into a single sorted list $\ord_R$ of (element, representative) pairs sorted by distance of the element from its corresponding representative with $O(|R|N\log(|R|))$ queries to $\mquad$. The logic of Algorithm \ref{alg:orderingtometric} operating on this list of pairs goes through unchanged except for the representative used in the query to $\mreal$, and some bookkeeping to separate the labeled and rounded list of pairs back into individual representative orderings. Algorithm \ref{alg:listmerge} outlines this process.

The following theorem summarizes this combined result and states that given a set of representatives $R$, $\DconstR$ can be constructed with $O(|R|N\log(N))$ queries to $\mtrip$, $O(|R|N\log(|R|))$ queries to $\mquad$ and $O(\maxalphlogR)$ queries to $\mreal$.
\begin{theorem}\label{theorem:multiplerepquad}
Given a set of representatives $R$ and access to $\mreal,\mtrip,$ and $\mquad$, an $\alpha-$submetric can be constructed from $O(\maxalphlogR)$ queries to $\mreal$, $|R|N\log(N)$ queries to $\mtrip$ and $|R|N\log(|R|)$ queries to $\mquad$ which preserves distances (up to the additive error) from the set of representatives $R$.
\end{theorem}
The proof of Theorem \ref{theorem:multiplerepquad} follows from a straightforward analysis of list merging, detailed in Algorithm \ref{alg:listmerge}.

\begin{proof}
A representative consistent ordering for each $r \in R$ can be constructed via Algorithm \ref{alg:tripletordering} in $O(N\log(N))$ queries to $\mtrip$ for each representative, $O(|R|N\log(N))$ queries to $\mtrip$ total.

Algorithm \ref{alg:listmerge} given such a set of orderings constructs $\DconstR$, an $\alpha-$submetric which preserves distances up to the additive error from each representative $r \in R$. The modified version of $\orderingtosubmetric$ still requires $O(\maxalphlogR)$ queries to $\mreal$ as the list length has increased to $|R|N$.
Thus all that remains to prove the theorem is to reason about the number of queries to $\mquad$. Algorithm \ref{alg:listmerge} makes $N|R|$ calls to $\mathsf{BinaryInsert}$ as each element in each ordering is inserted at most once into a list of length at most $|R|$. Each $\mathsf{BinaryInsert}$ into a list of length $k$ requires $O(\log(k))$ queries to $\mquad$ as a comparator. Thus the total number of queries to $\mquad$ is $O(|R|N\log(|R|))$.
\end{proof}

\begin{algorithm}[]
  \caption{\showalgcaption{$\mergeorderings(\{\ord_r|r\in R\}, \alpha, \mquad)$}}
  \label{alg:listmerge}
\begin{algorithmic}[1]
\LineComment{Inputs: a set of orderings $\{\ord_r | r \in R\}$, an error parameter $\alpha$, a interfaces to the \human, $\mquad$ and $\mreal$.}
  \Procedure{$\multipleorderingtosubmetric$}{$\{\ord_r|r\in R\}, \alpha, \mquad, \mreal$}
  \State $\ord_{merge} \leftarrow \mergeorderings(\{\ord_r|r\in R\}, \alpha, \mquad, \mreal)$
  \State $\{\Dconst | r \in R\} \leftarrow \mathsf{CreateSubmetrics}(\ord_{merge}, R, \alpha, \mreal)$ 
  \State \textbf{return } $\DconstR(x,y):= \maxmerge(\{\Dconst | r \in R\}, x,y)$
  \EndProcedure

\State
  \LineComment{Inputs: a set of orderings $\{\ord_r | r \in R\}$, an interface to the \human, $\mquad$.}
  \Function{$\mergeorderings$}{$\{\ord_r|r\in R\}, \alpha, \mquad, \mreal$}
  \State Initialize an empty list $F$ for the final output
  \State Initialize an empty list $L$
  \For{$r \in R$}
    \State $\mathsf{BinaryInsert}(L, 0, \size(L),(\ord_r[0],r),\mquad)$
    \State Remove $\ord_r[0]$
  \EndFor
  \While{all $\ord_r$ are not empty}
  \State Remove the first element from $L$ and append it to $F$
  \If{$\ord_r$ is not empty}
    \State $\mathsf{BinaryInsert}(L, 0, \size(L),(\ord_r[0],r),\mquad)$
    \State Remove $\ord_r[0]$
  \EndIf
  \EndWhile
  \State \textbf{return} $F$
  \EndFunction

  \State
  \LineComment{Inputs: $L$ an ordered list of element representative pairs, the list range delimiters $b$ and $e$, the element to insert $(x,r)$, an interface to the \human, $\mquad$.}
  \Function{$\mathsf{BinaryInsert}$}{$L,b,e, (x,r),\mquad$}
  \If{$b=e$}
    \State Insert $(x,r)$ into $L$ at position $b$
    \State \textbf{return}
  \EndIf
  \State $\midpoint \leftarrow \midpointof(L)$
  \If{$\mquad((x,r),L[\midpoint])$}
    \State $\mathsf{BinaryInsert}(L,\midpoint,e,(x,r),\mquad)$
  \Else
    \State $\mathsf{BinaryInsert}(L,b,\midpoint,(x,r),\mquad)$
  \EndIf
  \EndFunction
   \algstore{myalg}
 \end{algorithmic}
 \end{algorithm}

 \begin{algorithm}[]
 \begin{algorithmic}
 \algrestore{myalg}
  \LineComment{Inputs: $\ord_{merge}$ an ordered list of element representative pairs, the set of representatives $R$, an error parameter $\alpha$, and an interface to the \human, $\mreal$.}
  \Function{$\mathsf{CreateSubmetrics}$}{$\ord_{merge}, R, \alpha, \mreal$}
  \For{$r \in R$}
   \State Initialize $f_r(u)$ to all zeros
   \EndFor
  \State $\mathsf{SplitPairList}(\ord_{merge},r,\alpha,\mreal,f_r)$
  \State \textbf{return} $\{\D_r'(x,y):= |f_r(x) - f_r(y)| $ $| r \in R\}$
  \EndFunction
  \State
  \Function{$\mathsf{SplitPairList}$}{$\ord,R,\alpha, \mreal,\{f_r\}$}
  \State $(x_b, r_b) = \ord[0]$
  \State $(x_t, r_t) = \ord[\size(\ord)-1]$
  \State $d_{bottom} = \mreal(x_b,r_b)$
  \State $d_{top} = \mreal(x_t,r_t)$
  \If{$d_{top} - d_{bottom} \leq \alpha$}
  \ForAll{$(x,r) \in \ord$}
  \State Set $f_r(x)=d_{bottom}$
  \EndFor
  \Else
  \State $\midpoint \leftarrow \midpointof(\ord)$
  \State $\mathsf{SplitPairList}(\ord[0:\midpoint], R, \alpha, \mreal,\{f_r\})$
  \State $\mathsf{SplitPairList}(\ord[\midpoint:\size(\ord)-1], R, \alpha, \mreal,\{f_r\})$
  \EndIf
  \EndFunction
  \end{algorithmic}
\end{algorithm}

\subsection*{Summary}
In this section, we have shown how to use $O(\log(N))$ real-valued queries and $O(N\log(N))$ triplet queries in order to construct nontrivial representative submetric for a fixed universe of $N$ individuals.
When learning multiple representative submetrics, we have also shown how to improve the naive linear dependency on the number of representatives to logarithmic by supplementing with a $O(|R|N\log(|R|)$ \quadqs~and $O(|R|N\log(N))$ triplet queries.

In the next section (\ref{section:generalizinghumans}), we will show how to construct generalizable representative submetrics, i.e., how to predict what \humans~``would have said'' on unseen examples.
In the following section (\ref{section:choosingreps}), we tackle how to choose a small set of representatives to improve nontriviality guarantees.

\section{Generalization}\label{section:generalizinghumans}
In this section, we consider the problem of learning how to predict the \human's judgments on unseen samples from $\U$. (We consider how to pick the set of representatives in Section \ref{section:choosingreps}.) In particular, we will consider the problem of generalizing a representative submetric to fresh samples from $\U$.
Our goal is to construct efficient learners for submetrics as in Valiant's Probably Approximately Correct (PAC) model of learning \cite{valiant1984theory}. However, we do not want to be too prescriptive about the submetric concept class, particularly about the representation of elements in the universe. Instead, we will make an assumption about the learnability of \textit{threshold functions} (Definition \ref{def:thresholdfunction}) and construct learning procedures for submetrics using threshold functions as building blocks without any additional direct access to labeled or unlabeled samples from $\U$.

We restate the formal definition of an efficient submetric learner below.\footnote{This goal of learning a hypothesis that with high probability, does not exceed distances on most pairs in $\U \times \U$ is almost identical to Rothblum and Yona's notion of ``approximately metric-fair'' \cite{rothblum2018probably}.}

\newtheorem*{def:efficientlearner}{Definition \ref{def:efficientlearner}}
\begin{def:efficientlearner}[Efficient Submetric Learner - Restatement]
We say that a learning procedure is an \textit{efficient $\alpha-$submetric learner} if for any error and failure probability parameters $\varepsilon,\delta \in (0,1]$, given access to labeled examples of $\D(r,x \sim \U)$, with probability at least $1-\delta$ over the randomness of the sampling and the learning procedure produces a hypothesis $h_r: U\times U \rightarrow [0,1]$ such that
\[\Pr_{x,y \sim \U \times \U}[h_r(x,y) - \D(x,y) \geq \alpha] \leq \varepsilon\]
in time  $O(poly(\frac{1}{\varepsilon},\frac{1}{\delta}))$.
\end{def:efficientlearner}

In our formal definition, we are again purposefully vague about the type of the labeled examples, and all of our subsequent constructions will use labeled examples for the threshold functions and set $\alpha $ corresponding to the maximum difference between adjacent thresholds.
Whenever we use a set of ordered thresholds, $\T$, we will write $\alpha_{\T} =\max_{t_{i}\in \T}\{t_{i}-t_{i-1}\}$ to denote the maximum difference between adjacent thresholds.

In the remainder of this section, we formalize the relatively weak assumption that there exist a set of efficient learners for a set of binary threshold functions (Definition \ref{def:thresholdfunction}). Second,
we show the construction of an efficient learner for a submetric with additive error dependent on the set of thresholds based on voting by hypotheses produced by each threshold function learner.
Finally, we show how to combine a set of learners for submetrics as a first step to improving nontriviality as a warm-up for Section \ref{section:choosingreps}.

\subsection{Learnability of threshold functions}\label{section:learnability}
Assumption \ref{assumption:pacthreshold} (restated below for clarity) states that for every representative, there exists a set of thresholds
and a learner for each threshold in the set which,
with high probability\footnote{In the remainder of this work, when we state that a learner produces a hypothesis with high probability, we will always take the probability over the randomness of the sampling and the learning procedure.},
produces an accurate hypothesis for the threshold function for each threshold in the set which generalizes to unseen samples.
We first formally define a threshold function, which is a binary indicator of whether a particular element $u \in U$ is within distance $t$ of $r$ for a threshold $t \in [0,1]$ and a representative $r$, and then restate the learnability assumption.
\begin{definition}[threshold function]\label{def:thresholdfunction}
A threshold function $T_t^r(u): U \rightarrow \{0,1\}$ is defined
\[
    T_t^r(u) := \begin{cases}
      1 & \D(r,u) \leq t \\
      0 & \text{otherwise}
   \end{cases}
\]
with respect to a representative $r$ and metric $\D$.
\end{definition}

\newtheorem*{assumption:pacthreshold}{Assumption \ref{assumption:pacthreshold}}
\begin{assumption:pacthreshold}[Restatement]
Given a metric $\D$ and a representative $r$, there exists a set of thresholds $\T$ such that
\begin{enumerate}
    \item $t \in [0,1]$ for all $t \in \T$,
    \item $0 \in \T$,
    \item $\alpha_\T = \max_{i \in [|\T|-1]}t_{i+1} - t_{i}$,
    \item $|\T|= O(1)$,
\end{enumerate}
and for every $t \in \T$ there exists an efficient learner $L_t^r(\et,\dt)$ which for all $\et,\dt \in (0,1]$, with probability at least $1-\dt$ over the randomness of the sample and the learning procedure produces a hypothesis $h_t^r$ such that
\[\Pr_{x \sim \U}[h_t^r(x) \neq T_t^r(x)] \leq \et\] in time $O(poly(\frac{1}{\et},\frac{1}{\dt}))$ with access to labeled samples of $T_t^r(u \sim \U)$ for any distribution $\U$ over the universe. 
That is, the concept class $T_t^r$ is efficiently learnable for all $t \in \T$.
\end{assumption:pacthreshold}

As noted before, we are intentionally vague about the representation of $U$ because we tuck any issues of representation away into the assumption of learnability of threshold functions. All of our subsequent constructions will only interact with samples from $\U$ through the learners for the threshold functions, and as such, the representation can be completely abstracted away.
In Assumption \ref{assumption:pacthreshold}, the choice of $r$ is also not explicitly specified. In this work, we will take Assumption \ref{assumption:pacthreshold} to apply to every $r \in U$.

\subsection{Constructing submetric learners from threshold learners}
Given Assumption \ref{assumption:pacthreshold}, our next step is to determine how to combine the threshold learners into a learner for the threshold consistent underestimator for $r$ with respect to $\T$ which can be post-processed into an $\alpha_\T$ submetric.


We first show how to combine a set of hypotheses for threshold functions into a hypothesis for a threshold consistent underestimator. The $\linearvote$ mechanism, redefined below, takes in a set of hypotheses for the thresholds and outputs the threshold that the most hypotheses agree with.

\newtheorem*{def:linearvote}{Definition \ref{def:linearvote}}
\begin{def:linearvote}[$\linearvote$ - Restatement]
Given an ordered set of thresholds, $\T = \{t_1, t_2, \ldots, t_{|T|}\}$, and a set of hypotheses $H_\T^r = \{h_{t_1}^r, h_{t_2}^r, \ldots, h_{t_{|T|}}^r\}$, one corresponding to each threshold function, $\linearvote$ outputs the threshold that the most hypotheses agree with.

\[\linearvote(\T,H_\T^{r},x) := \argmax_{t_i}\sum_{t_j < t_i} (1-h_{t_j}^r(x)) + \sum_{t_j \geq t_i} h_{t_j}^r(x)\]
$\linearvote$ is equivalent to $\threshunder$ (Definition \ref{def:tcunderestimator}) when all of the $h_{t_i}^r$ output the correct value.

\end{def:linearvote}




Algorithm \ref{alg:thresholdcombiner} takes as input a set of thresholds and learners for those thresholds and (1) calls these learners with appropriately scaled parameters (2) and combines the resulting hypotheses via $\linearvote$ to produce a hypothesis $h_r$ for the $\alpha_\T$-submetric $\Dconst(x,y):=|\threshunder(x) - \threshunder(y)|$.\footnote{Algorithms \ref{alg:thresholdcombiner}, \ref{alg:combiner}, and   \ref{alg:submetriclearner} are all invoked with error and failure parameters. To keep the parameter names clear, we refer to $ \et$ and $\dt $ for the threshold function learners; $\er$ and $\dr$ for the single representative submetric  learner Algorithm \ref{alg:thresholdcombiner}; $\eR$ and $\dR$ for the combined representative set submetric learner Algorithm \ref{alg:combiner}; and $\eL$ and $\dL$ for the complete learning procedure Algorithm \ref{alg:submetriclearner}. Each time a learning procedure is invoked, we specify the relevant parameters using these variables.} In Algorithms \ref{alg:thresholdcombiner}, \ref{alg:combiner}, \ref{alg:submetriclearner}, we implicitly assume access to labeled samples of $T_t^r(u \sim \U).$ Sample complexity is explicitly analyzed in Theorem \ref{thm:querycomplexity}.
\begin{algorithm}[]
  \caption{$\thresholdcombiner(\T , L=\{L_{t_i \in \T}^r\}, \er, \dr)$}
  \label{alg:thresholdcombiner}
\begin{algorithmic}[1]
    \LineComment{Inputs: $\T = \{t_i\}$, the set of thresholds. $L=\{L_{t_i}^r\}$ the set of learners for each threshold in $\T$ for a particular representative, $r$. Error and failure probability parameters $\er, \dr$.}
    \Procedure{$\thresholdcombiner$}{$\T,L=\{L_{t_i}^r\}, \er, \dr$}
    \State Initialize an empty list of hypotheses $H_\T^r$.
    \For{$L_{t_i}^r \in L$}
      \State $h_{t_i}^r \leftarrow L_{t_i}^r(\frac{\er}{2|\T|},\frac{\dr}{|\T|})$
      \State Add $h_{t_i}^r$ to $H_\T^r$.
  \EndFor
  \State \textbf{return} $h_r(x,y) := |\linearvote(\T, H_\T^r, x) - \linearvote(\T,H_\T^r,y)|$
 \EndProcedure
\end{algorithmic}
\end{algorithm}

Theorem \ref{theorem:singlerepgeneral} states that  given a set of learners as specified in  Assumption \ref{assumption:pacthreshold}, Algorithm \ref{alg:thresholdcombiner} will produce a hypothesis for the $\alpha_\T-$submetric $\D_{r}'(x,y):= |\threshunder(x) - \threshunder(y)|$ with probability at least $1-\dr$ with error at most $\er$.

\begin{theorem}\label{theorem:singlerepgeneral}
Under Assumption \ref{assumption:pacthreshold}, there exists an efficient $\alpha_\T-$submetric learner. That is,
given a representative $r$, a distance metric $\D$, a distribution $\U$ over the universe, and a set of a constant number of thresholds $\T$,
if there exists a set of efficient learners $L=\{L_{t_i \in \T}^r\}$ as specified in Assumption \ref{assumption:pacthreshold},
then there exists an efficient learner which produces a hypothesis $h_r: U \rightarrow [0,1]$ such that $\Pr_{x,y \sim \U \times \U}[|h_r(x,y) - \Dconst(x,y)|\geq \alpha_\T] \leq \er$ with probability at least $1-\dr$ for all $\er$, $\dr \in (0,1]$ in time $O(poly(\frac{1}{\er}, \frac{1}{\dr}, |\T|))$, where $\Dconst(x,y):=|\threshunder(x)-\threshunder(y)|$ for all $x,y \in U \times U$.
\end{theorem}
\begin{proof}
Consider the construction of $h_r(x)$ as specified in Algorithm \ref{alg:thresholdcombiner}.
The failure probability of Algorithm \ref{alg:thresholdcombiner} is $\dr$ by union bound, as the procedure only fails if at least one of the learners in $L$ failed to produce an $\frac{\er}{2|\T|}-$good hypothesis for $T_{t_i}^r$.
As each learner in $L_{t_i}^r$ runs in time $O(poly(\frac{1}{\et}), O(\frac{1}{\dt}))$ by Assumption \ref{assumption:pacthreshold},
running all $|\T|$ learners takes time $O( poly(\frac{1}{\er},\frac{1}{\dr},|\T|))$, as Algorithm \ref{alg:thresholdcombiner} invokes each $L_{t_i}^r$ with $\et,\dt$ scaled by a factor of $\frac{1}{|\T|}$.
Recall to satisfy the definition of an efficient learner (Definition \ref{def:efficientlearner}), that Algorithm \ref{alg:thresholdcombiner} must run in time $O(\frac{1}{\er},\frac{1}{\dr})$. Given that $|\T|$ is constant, this requirement is satisfied.
With respect to accuracy,
notice that $h_r(x,y)$ only outputs a value more than $\alpha_\T$ away from $\Dconst(x,y)$ if at least one of $h_{t_i}^r(x)$ or $h_{t_i}^r(y)$ is in error.
Assuming all of the $L_{t_i}^r$ output good hypotheses, the probability that at least one of $h_{t_i}^r(x)$ or $h_{t_i}^r(y)$ is in error is at most $2\sum_{t_i \in \T} \frac{\er}{2|\T|} = \er$ by union bound. Thus, Algorithm \ref{alg:thresholdcombiner} satisfies the conditions of the theorem.
\end{proof}

Two key properties of the proof, which will be important in our consideration of query complexity to generate the labeled samples (Theorem \ref{thm:querycomplexity}), are (1) each of the threshold function learners learns on the same distribution $\U$, and (2) no independence of errors between the threshold function learners is assumed.

As in the previous section, combining information from multiple representatives can improve nontriviality guarantees.
Algorithm \ref{alg:combiner} takes as input a set of learners for representative submetrics for a set of representatives $R \subseteq U$ (for example, learners based on Algorithm \ref{alg:thresholdcombiner}) and produces a hypothesis $h_R$ based on the $\maxmerge$ of the hypotheses produced by the input learners.

\begin{algorithm}[]
  \caption{$\combiner(L=\{L_r\}, \eR, \dR)$}
  \label{alg:combiner}
\begin{algorithmic}[1]
  \LineComment{Inputs: a set of learners $\{L_r\}$ for each representative $r \in R \subseteq U$, and error and failure probability parameters $\eR,\dR$.}
  \Procedure{$\combiner$}{$L=\{L_r\}, \eR, \dR$}
  \State Initialize an empty list $H_R$.
  \For{$L_r \in L$}
      \State $h_r \leftarrow L_i(\frac{\eR}{|R|},\frac{\dR}{|R|})$
      \State Add $h_r$ to $H_R$
  \EndFor
  \State \textbf{return} $h_R(u,v) := \maxmerge(H_R, u, v)$
  \EndProcedure
\end{algorithmic}
\end{algorithm}

Theorem \ref{theorem:simplegeneralization} states that given a set of learners for threshold functions for a set of representatives (Assumption \ref{assumption:pacthreshold}), Algorithm \ref{alg:combiner} produces a hypothesis $h_R$ with probability at least $1-\dR$ with error at most $\eR$ which approximates $\DconstR(x,y):= \maxmerge(\{\Dconst|r \in R\}, x,y)$, where the $\Dconst$ are based on threshold consistent underestimators.
In contrast to the statement of Theorem \ref{theorem:singlerepgeneral}, which does not explicitly address nontriviality, Theorem \ref{theorem:simplegeneralization} introduces a nontriviality guarantee which
relies on the fraction of distances that exceed the contraction of the consistent underestimators.
This additional requirement stems from the fact that consistent underestimators with contraction in the distances between the representative and other elements in $U$ will not entirely preserve the original distance.\footnote{Note that the choice of $2\alpha_\T$ in order to preserve $\frac{1}{2}$ of the original distance is somewhat arbitrary. In Section \ref{section:choosingreps} we give a parametrizable guarantee.} Roughly speaking, for the nontriviality properties to hold, we need at least a $p$-fraction of distances in the distribution to be large enough that an $\alpha_{\T}$ contraction of the original distance is insignificant.
As we have not yet specified how representatives are chosen or how those choices preserve distances, we assume that all pairs with sufficiently large distances include at least one representative. We explicitly note the dependence on $|R|$ in the theorem statement as a placeholder until the required size for $|R|$ is established (Lemma \ref{nips:lemma:randomrepgeneralization}).

\begin{theorem}\label{theorem:simplegeneralization}
Given a distance metric $\D$, and a distribution $\U$ over the universe, 
if there exist a set of thresholds $\T$ and efficient learners $L=\{L_{t_i \in \T}^r\}$ as in Assumption \ref{assumption:pacthreshold}, and weight $p$ of pairs of elements in $\U\times\U$ include at least one representative $r \in R$ and have distance greater than $2\alpha_\T$,
then there exists an efficient learner which produces a hypothesis $h_R$ with probability greater than $1- \dR$ such that
\begin{enumerate}
    \item $\Pr_{x,y \sim \U \times \U}[h_R(x,y) > \D(x,y) + \alpha] \leq \eR$
    \item $h_R$ is $(p- \eR, \frac{1}{2})-$nontrivial for $\U$.
\end{enumerate}
    The learner runs in time
$O(poly(|\T|,|R|,\frac{1}{\eR},\frac{1}{\dR}))$ for all $\eR$, $\dR \in (0,1]$.

That is, under Assumption \ref{assumption:pacthreshold}, if weight $p$ of pairs in $\U\times\U$ which include at least one representative in $R$ have distance greater than $2\alpha_\T$, then there exists an efficient $(p - \eR, \frac{1}{2})-$nontrivial $\alpha_\T-$submetric learner.
\end{theorem}

\begin{proof}
Consider Algorithm \ref{alg:combiner} parametrized with $L=\{L_r\}$ constructed via Algorithm \ref{alg:thresholdcombiner} operating on $L=\{L_{t_i \in \T}^r\}$.

\textbf{Running time.}
Algorithm \ref{alg:combiner} makes $|R|$ calls to Algorithm \ref{alg:thresholdcombiner}.
Algorithm \ref{alg:thresholdcombiner} runs in time $O(poly(|\T|,\frac{1}{\er},\frac{1}{\dr}))$, where $\er=\frac{\eR}{|R|}$ and $\dr=\frac{\dR}{|R|}$ are the error and failure probability parameters with which Algorithm \ref{alg:combiner} invokes Algorithm \ref{alg:thresholdcombiner}.
Thus Algorithm \ref{alg:combiner} runs in time $O(poly(|R|,|\T|,\frac{1}{\eR},\frac{1}{\dR}))$.

\textbf{Failure probability.}
We say that Algorithm \ref{alg:combiner} has ``failed'' if at least one of $L_r$ fails to produce an $\frac{\eR}{|R|}-$good hypothesis $h_R$. The failure probability of Algorithm \ref{alg:thresholdcombiner} is $\leq \sum_{r \in R}\frac{\dr}{|R|}=\dR$ by union bound.

\textbf{Overestimate error probability.}
Suppose that all of the learners in $L$ produce a good candidate $h_r$ with error probability $\frac{\eR}{|R|}$ or less.
Now, consider the probability that the result of $\maxmerge(H_R,u,v)$ is an over-estimate by more than $\alpha_\T$. This can only happen if at least one of the $h_r$ is in error by more than $\alpha_\T$. Thus by union bound, the probability of over-estimate is at most $\eR.$

\textbf{Nontriviality.} Each of the $h_r$ has additive and subtractive error at most $\alpha_\T$, so for any $r \in R$ and $u \in U$ such that $\D(r,u) \geq 2\alpha_\T$, at least half of the original distance will be preserved. Thus, making the worst case assumption\footnote{We could omit the error probability $\eR$ in the statement of nontriviality and leave implicit that the nontriviality guarantees ``stack'' with the hypothesis error probability. However, this is somewhat misleading as we cannot assume that the errors of the hypothesis are randomly distributed.} that all $\eR$ weight of errors result in distance underestimates on the relevant pairs, the metric learned is $(p - \eR,\frac{1}{2})-$nontrivial for $\U$.
\end{proof}

Notice that in the analysis of the error and failure probability for Algorithm \ref{alg:combiner}, there is no particular requirement that the learners used to produce $h_r$ for each representative be based on thresholds. The only requirement is that the learners produce $h_r$ such that $\Pr_{x,y \sim \U}[|h_r(x,y) - \D(x,y)|\geq \alpha_\T]\leq \er$ with probability at least $1- \dr$. Thus in settings with alternative mechanisms to produce such $h_r$, they can be substituted without compromising the result. We state the following corollary to formalize this intuition.

\begin{corollary}\label{cor:simplegeneralization}
Given a distance metric $\D$, and a distribution $\U$ over the universe, %
if there exist a set of efficient learners $L=\{L_{r\in R}\}$ such that, given access to labeled samples, $L_r$ produces a hypothesis $h_r$ such that $\Pr_{x,y \sim \U \times \U}[|h_r(x,y) - \D(x,y)| \geq \alpha]\leq \er$ with probability at least $1-\dr$
in time $O(poly(\frac{1}{\er},\frac{1}{\dr}))$
 and weight $p$ of pairs of elements in $\U\times\U$ include at least one representative $r \in R$ and have distance greater than $2\alpha_\T$,
then there exists an efficient learner which produces a hypothesis $h_R$ with probability greater than $1- \dR$ such that
\begin{enumerate}
    \item $\Pr_{x,y \sim \U \times \U}[h_R(x,y) > \D(x,y) + \alpha] \leq \eR$
    \item $h_R$ is $(p- \eR, \frac{1}{2})-$nontrivial for $\U$.
\end{enumerate}
The learner runs in time $O(poly(|R|,\frac{1}{\eR},\frac{1}{\dR}))$ for all $\eR$, $\dR \in (0,1]$.
\end{corollary}

As discussed in Section \ref{section:preliminaries} (Proposition \ref{prop:post}), A submetric can be postprocessed to reduce the additive error. Corollary \ref{cor:basicpost} below reflects the result of postprocessing, in particular the impact on the distance distribution requirements.
\begin{corollary}\label{cor:basicpost}
Given a distance metric $\D$, and a distribution $\U$ over the universe, 
if there exist a set of efficient learners $L=\{L_{r\in R}\}$ such that, given access to labeled samples, $L_r$ produces a hypothesis $h_r$ such that $\Pr_{x,y \sim \U \times \U}[|h_r(x,y) - \D(x,y)| \geq \alpha]\leq \er$ with probability at least $1-\dr$
in time $O(poly(\frac{1}{\er},\frac{1}{\dr}))$
 and weight $p$ of pairs of elements in $\U\times\U$ include at least one representative $r \in R$ and have distance greater than \boldmath $2\alpha_\T + \alpha$\unboldmath,
then there exists an efficient learner which produces a hypothesis $h_R$ with probability greater than $1- \dR$ such that
\begin{enumerate}
    \item \boldmath$\Pr_{x,y \sim \U \times \U}[h_R(x,y) > \D(x,y)] \leq \eR$\unboldmath
    \item $h_R$ is $(p- \eR, \frac{1}{2})-$nontrivial for $\U$.
\end{enumerate}
The learner runs in time $O(poly(|R|,\frac{1}{\eR},\frac{1}{\dR}))$ for all $\eR$, $\dR \in (0,1]$.
\end{corollary}

Theorem \ref{theorem:simplegeneralization} is the first step to learning submetrics which generalize to unseen samples, but the limited nontriviality guarantee is potentially problematic.
The next section considers how the choice of representatives and the properties of the metric on the distribution $\U$ impact nontriviality.

\section{Choosing Representatives}\label{section:choosingreps}
There are two approaches one might take to improve the nontriviality guarantee of Theorem \ref{theorem:simplegeneralization}: (1) develop specialized strategies for combining representative submetrics which depend on the structure of the metric, or (2) characterize generic randomized strategies.
We briefly consider the first approach below, and then devote the remainder of the section to the second approach.

\subsection{Metric structure dependent strategies.}
First, one could propose a representative selection mechanism tailored to a particular problem setting.
This is a very reasonable strategy if some structure of the metric is known which can be exploited to better combine the representative submetrics, or there are specific distance preservation properties other than nontriviality which are deemed desirable.

For example, suppose that we had some understanding that the underlying metric we wish to learn is Euclidean distance in two dimensions. Even without knowing the features relevant to each dimension, we can propose a generic ``representative GPS'' submetric combination procedure. 
We could choose $3$ representatives (with some additional conditions to ensure they form a basis) and use Algorithm \ref{alg:thresholdcombiner} to learn a representative submetric $\D_r$ for each representative with reasonably small contraction which generalizes to unseen samples. These distances can be used to build up a $2-$dimensional embedding of the representative points and any new points observed.
Notice that each new point can have at most one valid position in the embedding depending on its distance from the $3$ representatives.\footnote{This assumes that all distances are exact, there is some slack when the distance from each representative is an underestimate. We omit a full treatment of this problem in this work, both in terms of number of dimensions and approximation of representative distance, as it is not inherently important to understanding the motivation for setting specific strategies. Briefly, when distances from each representative are not exact it is possible that the region of possible locations is not contiguous for a new point. In terms of computing distances between two points of uncertain location, this can be ``fixed'' from an overestimate perspective by taking the minimum of the distances between all possible locations, but at the cost of weaker nontriviality guarantees.}
%
Thus for any pair $u,v \in U \times U$ we can compute their distance based on their relative positions from the set of representatives with error probability proportional to the error of our hypotheses for $\{\Dr\}$. Essentially, with a strong assumption on the form of the metric, we may be able to propose a representative submetric combination strategy which gives very good nontriviality guarantees.

\subsection{Random representatives}
When little or no information is known about the structure of the metric, or the known structure dos not lend itself to a simple representative selection strategies, choosing a set of representatives at random is a reasonable alternative strategy.
When a set of representatives is chosen at random, a key component of the argument for how well the set will preserve distances is how distances between pairs are distributed in $\U\times \U$. For instance, if most of the weight in $\U\times \U$ is concentrated on pairs which are maximally distant, it may be more difficult to generate a set of good representatives compared with an alternative distribution over $U$ which results in a broader range of distances.
A set of randomly chosen representatives will have certain nontriviality properties which depend on the more generic ``density'' properties of the metric and distribution $\U$, which we define below. In contrast to a setting-specific strategy, we don't make any assumptions about how submetrics based on different representatives can be combined other than the universally applicable merges specified in Lemma \ref{lemma:submerge}.

We devote the remainder of this section to understanding the generalization properties of a random set of representatives.
First, we formalize the definition of a $\gamma -$net to capture the notion of a set of representatives ``covering'' a fraction of the distribution (subset of the universe) and prove several useful lemmas relating the size of $\gamma$ to the nontriviality properties of the submetric.
Next, we formally define the \density~and \diffusion~parameters for a metric and distribution over the universe, and show how the nontriviality properties of $\gamma-$nets relate to these parameters.
Roughly speaking, \density~describes how closely packed elements are and characterizes how easy it is to construct a $\gamma$-net, whereas \diffusion~describes how many distances are large enough to tolerate a contraction.
Intuitively, more closely packed points (high density) will make it easier to find a representative closer to those points, but the tradeoff is additional small absolute distances  between points (lower diffusion), which will be more impacted by the underestimate error of the net.\footnote{For example, a universe with all points except one clustered together with distances less than $\alpha$ will be easy to cover with representatives at distance at most $\alpha$ from each element, but any contraction of size approximately $ \alpha$ will destroy any distinguishing power between the clustered points.}
Finally, we characterize the number of randomly sampled representatives needed to form a $\gamma -$net, given the \density~and \diffusion~characteristics of the metric and distribution, and use this to prove our main generalization result.

 \subsection{Distance preservation via $\gamma-$nets}
 The crux of the argument for nontriviality with random representatives is (1) a random sample of representatives is likely to be ``close to'' a significant portion of $\U$, and (2) we can bound the magnitude of underestimates based on the distance from a representative for arbitrary metrics.
 Recall the definition of a $\gamma-$net, which captures the notion of being ``close to'' or ``covering'' a set of elements.
 \newtheorem*{def:gammanet}{Definition \ref{def:gammanet}}
\begin{def:gammanet}[Restatement]
 A set $R \subseteq U$ is said to form a \textit{$\gamma-$net} for a subset $V \subseteq U$ under $\D$ if for all balls of radius $\gamma$ (determined by $\D$) containing at least one element $v \in V$, the ball also contains $r \in R$.
\end{def:gammanet}
 

 To reason about nontriviality of a set of representatives which form a $\gamma-$net, we derive a bound on the contraction of distances between pairs based on their distances to a representative.
 Intuitively,
 the distance between a representative and another element in the universe will be nearly identical to the distance between a close neighbor of the representative and that element.
 Lemma \ref{nips:lemma:generalbound} (restated below) states that, given a representative $r$, $\Dr$ underestimates $\D(u,v)$ by at most $\min\{2\D(r,u),2\D(r,v)\}$.
\newtheorem*{nips:lemma:generalbound}{Lemma \ref{nips:lemma:generalbound}}
\begin{nips:lemma:generalbound}[Restatement]
For all $u,v \in U \backslash \{r\}$, $\D(u,v) - \Dr(u,v) \leq \min\{2\D(r,u),2\D(r,v)\}$, where $\Dr$ is the representative submetric for $r \in U$.
\end{nips:lemma:generalbound}
 \begin{proof}
 By construction, $\Dr(u,v) = |\D(r,v)-\D(r,u)|.$ Without loss of generality, assume $\D(r,u) \leq \D(r,v).$ By triangle inequality, $\D(r,v) \geq \D(u,v) - \D(r,u)$, so
 $\D(u,v) - \Dr(u,v) \leq 2\D(r,u)$.
 \end{proof}
 \begin{corollary}
 For all $u,v \in U \backslash \{r\}$, $\D(u,v) - \Dconst(u,v) \leq \min\{2\D(r,u),2\D(r,v)\} + \alpha$, where $\Dconst$ is the consistent underestimator representative submetric for $r \in U$ with maximum contraction $\alpha$.
 \end{corollary}

 Lemma \ref{nips:lemma:generalbound} is very useful for understanding the distance contractions for sets of representatives which form $\gamma-$nets for $U$, as every pair is close to at least one representative.
 Of course, forming a $\gamma-$net for an arbitrary $\gamma$ isn't enough on its own to give a good nontriviality guarantee.\footnote{ For example, if all of the elements in $U$ are contained in two well separated balls of radius $\gamma$, a $\gamma-$net will preserve distances between pairs with one element in each ball well, but distances between pairs within the \textit{same} ball may not be. This issue is a significant motivation for defining nontriviality as a relative distance preservation guarantee, rather than an absolute maximum contraction. Notice that the absolute contraction in this case is potentially very small, only $2\gamma$, but the relative contraction may be significantly higher for pairs contained in the same ball. Later applications seeking to use the submetric as constraints on a classifier will not be able to make nuanced decisions between elements in the same ball, which may be problematic for some settings.}

\subsection{Density and \diffusion}
To understand how representatives which form a $\gamma -$net will preserve distances, we recall the definitions of  \textit{\density}~and \textit{\diffusion}~below to characterize the relevant properties of the metric and distribution.
The notion of $(\gamma,a,b)-$dense is intended to capture the weight ($a$) of elements that have a significant weight ($b$) on their close (distance $\gamma$) neighbors under $\U$ as a way to characterize how likely it is that a randomly chosen representative will be $\gamma$-close to a significant fraction of elements.

\newtheorem*{def:density}{Definition \ref{def:density}}
\begin{def:density}[$(\gamma,a,b)-$dense - Restatement]
Given a distribution $\mathcal{U}$ over the universe $U$, a metric $\D: U \times U \rightarrow [0,1]$ is said to be $(\gamma,a,b)-$dense for $\mathcal{U}$ if there exists a subset $A\subseteq U$ with weight $a$ under $\mathcal{U}$ such that for all $u \in A$
\[\Pr_{v \sim \mathcal{U}}[\D(u,v) \leq \gamma] \geq b\]
\end{def:density}


Figure \ref{nips:fig:avsb} illustrates the tradeoff between $a$ and $b$ for a particular choice of $\gamma$ for $\unifu$ on an example universe in $\R^2$.



In addition to density, we will also frequently consider the fraction of distances larger than a given constant. This allows us to reason about how much the contraction in the submetric will affect the distances preserved, as in the statement of Theorem \ref{theorem:simplegeneralization}. This notion is formalized as \diffusion. 

\newtheorem*{def:diffuse}{Definition \ref{def:diffuse}}
\begin{def:diffuse}[$(p,\zeta)-$\diffuse - Restatement]
Given a distribution $\U$, a metric $\D$ is $(p,\zeta)-$\diffuse~if the fraction of distances between pairs of elements in $\U \times \U$ greater than $\zeta$ is $p$, ie
\[\Pr_{u,v \sim \U \times \U}[\D(u,v)\geq \zeta]\geq p\]
\end{def:diffuse}


Definition \ref{def:diffuse} is highly reminiscent of nontriviality (Definition \ref{def:nontrivial}) and we formally relate diffusion to nontriviality in Lemma \ref{lemma:fullnet}. 
Notice that, although there are five parameters describing a metric and distribution across the two definitions, these parameters are highly related.
We will generally consider distributions which are $(\gamma,a,b)-$dense and $(p,\frac{2\gamma}{1-c})-$\diffuse.
Although $\frac{2\gamma}{1-c}$ initially seems an arbitrary quantity, it indicates that a $p-$fraction of pairs will have distances preserved by a factor of $c$ if the maximum contraction for those pairs is no more than $2\gamma$. Thus the values of $\gamma$ and $c$, which in turn dictate $p$, $a$, and $b$, (assuming $\zeta=\frac{2\gamma}{1-c}$) can loosely be seen as a tradeoff between how many pairs will have distance preservation guarantees and how large the guarantees will be. In the case of the example in Figure \ref{nips:fig:avsb}, we could describe the uniform distribution as $(.88, .4)-$\diffuse, or $(.88, \frac{2\gamma}{1-c})-$\diffuse, where $c=\frac{1}{2}$ and $\gamma=0.1$.
That is, with contraction $2\gamma$ at least 88\% of the pairs in $\unifu\times\unifu$ will have at least half of their distances preserved.

\subsubsection{Nontriviality properties of $\gamma -$nets}
Given the formalization of \diffusion, we can now relate the magnitude of $\gamma $ to the nontriviality properties of the merged representative set submetric.
Lemma \ref{lemma:fullnet} states that a set of representatives which form a $\gamma-$net for $U$ will have nontriviality properties related to the \diffusion~properties of $\D$.
\begin{lemma}\label{lemma:fullnet}
If a set of representatives $R \subseteq U$ form a $\gamma-$net for a universe $U$
and $\D$ is $(p,\frac{2\gamma}{1-c})-$\diffuse~on $\U$,
then $\DR$ is $(p,c)-$nontrivial on $\mathcal{U}$.
\end{lemma}
\begin{proof}
Recall from the proof of Lemma \ref{nips:lemma:generalbound} that the distance between a pair $\Dr(u,v)$ has contraction at most $\min\{2\D(r,v),2\D(r,u)\}$.
Thus, the distance between any pair of elements is contracted by at most $2\gamma$. A $p$ fraction of distances between pairs are greater than $\frac{2\gamma}{1-c}$, so an absolute contraction of $2\gamma$ for these elements yields a ratio of at least $\frac{2\gamma/(1-c) - 2\gamma}{\frac{2\gamma}{1-c}} = 1 - \frac{2\gamma}{\frac{2\gamma}{1-c}} = c$ and thus a $2\gamma$ absolute contraction is at most a $c$ relative contraction for this set of elements. So we conclude that the max-merge of $\Dr$  for $r \in R$ is $(p, c)-$nontrivial for $\U$.
\end{proof}

Corollary \ref{cor:fullnetunderestimate} states that in the case of consistent underestimators with $\maxcontraction=\alpha'$ that accounting for the potential underestimate error in the \diffusion~parameter is sufficient to yield the same nontriviality guarantees as in Lemma \ref{lemma:fullnet}.
Corollary \ref{cor:fullnetunderestimate} follows from observing the maximum possible contraction due to the underestimation from the $\gamma-$net placement and the underestimation of the consistent underestimators.
\begin{corollary}\label{cor:fullnetunderestimate}
If a set of representatives $R \subseteq U$ form a $\gamma-$net for a universe $U$
and $\D$ is $(p,\frac{2\gamma + \alpha'}{1-c})-$\diffuse~on $\U$
then $\Dconst$, produced from $\alpha-$consistent underestimators with maximum contraction $\alpha'$, is $(p,c)-$nontrivial for $\mathcal{U}$.
\end{corollary}


Returning to the example universe from Figure \ref{nips:fig:avsb}, Lemma \ref{lemma:fullnet} implies that if we selected a set of representatives $R$ which formed a $0.1-$net for the whole universe, then $\DR$ produced from exact evaluations of $\D(r,u)$ for all $u \in U$ and $r \in R$ would be $(0.88, \frac{1}{2})-$nontrivial for $\unifu$. That is, $\DR$ would preserve half of the original distance for almost $90\%$ of pairs in $\unifu\times\unifu$.

We now recall and present the proof for Lemma \ref{nips:lemma:fullnetc}, the weighted subset analog of Lemma \ref{lemma:fullnet}, which states that if a set of representatives form a $\gamma-$net for a subset of $U$, then the nontriviality properties depend on the weight of that subset in $\U$.

\newtheorem*{nips:lemma:fullnetc}{Lemma \ref{nips:lemma:fullnetc}}
\begin{nips:lemma:fullnetc}[Restatement]
If a set of representatives $R \subseteq U$ form a $\gamma-$net for weight $\w$ of $\U$
and $\D$ is $(p,\frac{2\gamma}{1-c})-$\diffuse~on $\U$,
then the submetric $\DR$ is $(p', c)-$nontrivial for $\U$, where $p' \geq p - (1-\w)^2$.
\end{nips:lemma:fullnetc}

\begin{proof}
Consider the pairs in $\U \times \U$ which have distance at least $\frac{2\gamma}{1-c}$. The total weight of such pairs in $\U\times \U$ is $p$. Pairs with neither element in the net can have weight at most $(1-\w)^2$. 
Assuming the worst case scenario that all $(1-\w)^2$ weight of pairs with neither element in the net are also pairs with distance at least $\frac{2\gamma}{1-c}$, 
at least a $p' \geq p - (1-\w)^2$ weight in $\U \times \U$ have at least one element in the net and a distance of at least $\frac{2\gamma}{1-c}.$

By the same logic as in the proof of Lemma \ref{lemma:fullnet}, pairs with distance at least $\frac{2\gamma}{1-c}$ have relative contraction at most $c$ if at least one member is in the $\gamma-$net.
Thus the $\maxmerge$ of the submetrics from representatives in $R$ is $(p',c)-$nontrivial for $p' \geq p - (1-\w)^2$.
\end{proof}

Corollary \ref{nips:lemma:fullnetc}.1 restates Lemma \ref{nips:lemma:fullnetc} in terms of consistent underestimators, accounting for the maximum contraction in the \diffusion~parameters.

\newtheorem*{cor:fullnetcunderestmate}{Corollary \ref{nips:lemma:fullnetc}.1}
\begin{cor:fullnetcunderestmate}
If a set of representatives $R \subseteq U$ form a $\gamma-$net for weight $\w$ of $\U$
and $\D$ is $(p,\frac{2\gamma+ \alpha'}{1-c})-$\diffuse~on $\U$,
then the $\alpha-$submetric $\Dconst$, formed from $\alpha-$consistent underestimators with maximum contraction $\alpha'$, is $(p', c)-$nontrivial for $\U$, where $p' \geq p - (1-\w)^2$.
\end{cor:fullnetcunderestmate}


The nontriviality guarantees of Lemmas \ref{lemma:fullnet} and \ref{nips:lemma:fullnetc} are conservative. They incorporate a worst-case assumption on the distribution of large distances in Lemma \ref{nips:lemma:fullnetc}, and entirely ignore the exact distance preservation from the representatives in both Lemmas. Again, we stress that our goal in this section is to show the possibility of positive results, and we do not attempt to achieve optimal performance or guarantees.

Corollary \ref{nips:lemma:fullnetc}.2 restates the Lemma directly in terms of the probability that at least one element in the pair sampled is covered by the $\gamma-$net and the distance is greater than $\frac{2\gamma + \alpha'}{1-c}$ in order to get a tighter characterization of nontriviality.

\newtheorem*{cor:fullnetcunderestmateknownp}{Corollary \ref{nips:lemma:fullnetc}.2}
\begin{cor:fullnetcunderestmateknownp}
If a set of representatives $R \subseteq U$ form a $\gamma-$net for a subset $V \subseteq U$,
and $\Pr_{u,v \sim \U \times \U}[(u \in V \lor v \in V) \wedge (\D(u,v) > \frac{2\gamma + \alpha'}{1-c})]\geq p$,
then the $\alpha-$submetric $\Dconst$, formed from $\alpha-$consistent underestimators with maximum contraction $\alpha'$, is $(p, c)-$nontrivial for $\U$.
\end{cor:fullnetcunderestmateknownp}


Given a set of representatives, it is possible to empirically measure $p$ on a sample to improve the bounds given by Lemma \ref{nips:lemma:fullnetc} or Corollary \ref{nips:lemma:fullnetc}.2. For maximum generality, we will rely only on the density and \diffusion~properties of the metric and distribution, but we include Corollary \ref{nips:lemma:fullnetc}.2 as a reminder that the bounds given are by no means tight.

\subsubsection{Representative set size }\label{subsection:netgeneralization}
We now consider how likely it is that a set of random representatives drawn from $\U$ will form a $\gamma-$net for $\U$ given the density properties of $\D$ on $\U$. Lemma \ref{nips:lemma:randomrepgeneralization} (restated below) states that a set of random representatives $R$ of size $O(\frac{1}{b}\ln(\frac{1}{b\delta}))$ will be sufficient to guarantee with high probability that the submetric $\D_R$ constructed from exact evaluations of $\Dr$ via queries to the \human~on new samples from $\U\times \U$ will have nontriviality properties related to the density and \diffusion~of $\D$ for $\U$.

\newtheorem*{nips:lemma:randomrepgeneralization}{Lemma \ref{nips:lemma:randomrepgeneralization}}
\begin{nips:lemma:randomrepgeneralization}[Restatement]
If a metric $\D$ is $(\gamma, a, b)-$dense and
$(p,\frac{6\gamma}{1-c})-$\diffuse~on $\U$,
then a random set of representatives $R$ of size at least $\frac{1}{b}\ln(\frac{1}{b\delta})$ will produce a $(p - (1-a)^2, c)$-nontrivial submetric $\DR$
for $\U$ with probability at least $1-\delta$, where
$\DR$ is constructed from exact evaluations of $\Dr$ via queries to the \human.
\end{nips:lemma:randomrepgeneralization}

\begin{proof}
Notice that if a set of representatives $R \subseteq U$ forms a $3\gamma-$net for an $a$ fraction of $\U$, then by Lemma \ref{nips:lemma:fullnetc} the submetric $\DR$ will be $(p',c)-$nontrivial for $p' \geq p - (1-a)^2$.

Suppose that a metric is $(\gamma,a,b)$ dense. Denote the weight $a$ subset of $U$ (with associated weight $b$ $\gamma-$close subsets) as $A$. Suppose that a random sample $R \sim \U$ of size $m$ does \textit{not} form an $\gamma-$net for $A$.  Then it must be the case that there is at least weight $b$ of $\U$ not included in $R$. That is, the associated weight $b$ subset of at least one element in $A$ is not ``hit'' by any representative. Thus, it is sufficient to bound the probability that weight $b$ of $\U$ corresponding to an element in $A$ is not hit by a sample of size $m$ to determine if our sample forms an $\gamma-$net for $A$, satisfying the conditions of the lemma.

As a warm-up, suppose that all of the weight $b$ subsets corresponding to elements in $A$ are disjoint.
The probability that all $m$ samples do not fall into a particular weight $b$ subset of $\U$ is $(1-b)^m$. 
Notice that if all elements in $A$ have disjoint associated weight $b$ subsets, then the probability that all $m$ samples do not fall into at least one of the disjoint weight $b$ subsets of $\U$ is at most  $\frac{1}{b}(1-b)^m$. (Notice, that there are at most $\frac{1}{b}$ disjoint weight $b$ subsets in total weight $1$.) Rearranging and substituting $1-b \leq e^{-b}$, any $m$ which satisfies:
\[\frac{1}{b}(1-b)^m \leq \delta\]
\[e^{-mb}\leq b\delta\]
\[m \geq \frac{1}{b}\ln(\frac{1}{b\delta})\]
will fail to hit any subset of weight at least $b$ with probability at most $\delta$. Thus, if the associated weight $b$ subsets for $A$ are disjoint, a set of representatives of size $\frac{1}{b}\ln(\frac{1}{b\delta})$ is sufficient to produce a $\gamma-$net for $A$ with probability at least $1-\delta$.

Now, consider the (more likely) case that the weight $b$ subsets for elements in $A$ are not disjoint.
We will show that there is a set of disjoint weight $b$ subsets, $B_{remain}$, such that if every disjoint subset in $B_{remain}$ is ``hit'' by at least one element in $R$, then every element in $A$ is at most distance $3\gamma$ from a representative, i.e. $R$ forms a $3\gamma-$net for $A$.

Consider the entire set of weight $b$ subsets associated with elements in $A$. Now, suppose that we removed the minimum number of subsets such that the remaining weight $b$ subsets were all disjoint. Call the minimal set of removed subsets $B_{remove}$, and the set of remaining disjoint weight $b$ subsets $B_{remain}$.
Consider removing each subset in $B_{remove}$ one at a time. The last subset removed must have overlap with at least one subset in $B_{remain}$, or there would be a smaller minimum set we could have removed which does not contain the last subset.
Notice that we may remove the subsets in $B_{remove}$ in any order, and yet this observation still holds for the final subset removed. Thus, each subset in $B_{remove}$ must have overlap with a subset in $B_{remain}$, so the furthest any element in a subset in $B_{remove}$ could be from a representative that ``hits'' a set in $B_{remain}$ is $4\gamma$.
However, an element in $A$ in associated with a weight $b$ subset in $B_{remove}$ can only be distance $3\gamma$ from the hitting representative, as it is at most distance $\gamma$ from at least one of the element(s) overlapping with $B_{remain}$, which are in turn at most distance $2\gamma$ from the hitting representative.

As in the disjoint case above, the size of $B_{remain}$ is bounded by $1/b$, and the same logic applies, but forming a $3\gamma-$net.
Thus for a set of randomly sampled representatives of size $m\geq \frac{1}{b}\ln(\frac{1}{b\delta})$, the probability of the representatives chosen not forming a $3\gamma-$net for weight $a$ of $\U$ is at most $1-\delta$.
\end{proof}

Corollary \ref{nips:lemma:randomrepgeneralization}.1 is the consistent underestimator analog of Lemma \ref{nips:lemma:randomrepgeneralization}.

\newtheorem*{cor:randomrepgeneralization}{Corollary \ref{nips:lemma:randomrepgeneralization}.1}
\begin{cor:randomrepgeneralization}
If a metric $\D$ is $(\gamma, a, b)-$dense and
$(p,\frac{6\gamma + \alpha'}{1-c})-$\diffuse~on $\U$,
then a random set of representatives $R$ of size at least $\frac{1}{b}\ln(\frac{1}{b\delta})$ will produce a $(p - (1-a)^2, c)$-nontrivial $\alpha-$submetric $\DconstR$, constructed from exact evaluations of $\alpha-$consistent underestimators, for each representative with maximum contraction $\alpha'$ for $\U$ with probability at least $1-\delta$.
\end{cor:randomrepgeneralization}


Our strategy of using random representatives is motivated by a desire for as much generality as possible with respect to the form of the metric. However, random sampling is not the only method to construct a $\gamma-$net.

\begin{remark}\label{remark:greedygamma}
Choosing a set at random to form a $\gamma-$net ignores the information provided by each of the representatives.  
A $\gamma-$net for a fixed sample, or some weight of a fixed sample, can be constructed via a greedy algorithm rather than random sampling. 
The key obstacle to analyzing the effectiveness of a greedy procedure is that the choice of the next representative, based on the weight of elements it may add to the net, can be based only on the existing incomplete distance information. In some cases, this incomplete information may lead to very sub-optimal choices.
However, there may be procedures which take advantage of \quadqs~and rough ordering information to reduce the number of mistakes made, at the cost of additional queries to the \human.
For example, \quadqs~can be used to check a small sample of the elements a candidate representative $r_c$ is expected to add against a known distance pair of approximately distance $\gamma$ in order to better estimate the expected contribution to the net.
We anticipate that such strategies may be useful in practice, even if a rigorous theoretical analysis for arbitrary metrics is pessimistic. It may also be useful to characterize the set of metrics which have bounded error in this incomplete information scenario, and we pose this as an open question for future work.
\end{remark}

\subsection{Generalization with random representative sets}
Thus far we have shown that a random set of representatives can have good properties for new samples drawn from the distribution, assuming we construct the submetric from \textit{exact} evaluations of $\Dr$ or $\Dconst$, i.e. with \textit{unlimited} access to the \human.
We now combine the results of Theorem \ref{theorem:simplegeneralization} and Lemma \ref{nips:lemma:randomrepgeneralization} to show how to construct an efficient submetric learner which produces submetrics with good nontriviality properties, given \textit{limited} query access to the \arbiter~for training data generation.

Algorithm \ref{alg:submetriclearner} picks a set of representatives which will form a $\gamma-$net for weight of a $(\gamma,a,b)-$dense metric with probability at least $1-\dL/2$. (Recall from Lemma \ref{nips:lemma:randomrepgeneralization} that the number of representatives required depends only on $b$ for a $(\gamma,a,b)-$dense metric.) These representatives are then used to specify a set of $\alpha_\T-$submetric learners (via Algorithm \ref{alg:thresholdcombiner}) which are passed to Algorithm \ref{alg:combiner} to construct a good final combined submetric with probability at least $1-\dL/2$. (That is, Algorithm \ref{alg:submetriclearner} splits its failure probability ``budget'' evenly between the choice of representatives and the learners for each representative.)


\begin{algorithm}[]
  \caption{}
  \label{alg:submetriclearner}
\begin{algorithmic}[1]
  \LineComment{Inputs: error and failure probability parameters $\varepsilon,\delta$, density parameter $b$, a set of threshold function learners $\{L_{t_i \in \T}^r\}$, and the threshold set $\T$.}
  \Procedure{$\submetriclearner$}{$\varepsilon, \delta, b,\{L_{t_i \in \T}^r\}, \T$}
  \State Sample $R \sim \U$ such that $|R| = \frac{1}{b}\ln(\frac{2}{b\dL})$.
  \State Initialize an empty list $L$.
  \For{$r \in R$}
    \State $L_r(\er,\dr) := \thresholdcombiner(\T, \{L_{t_i \in \T}^r\},\er,\dr)$
    \State Add $L_r$ to $L$.
  \EndFor
  \State $\eR \leftarrow \varepsilon$
  \State $\dR \leftarrow \delta/2$
  \State \textbf{return} $h_R(u,v) = \combiner(L, \eR,\dR)$
  \EndProcedure
\end{algorithmic}
\end{algorithm}

\begin{theorem}\label{theorem:maingeneralization}
Given a distance metric $\D$, and a distribution $\U$ over the universe 
if
\begin{enumerate}
    \item There exist a set of thresholds $\T$ and efficient learners $\{L_{t_i \in \T}^r\}$ as in Assumption \ref{assumption:pacthreshold}, and
    \item $\D$ is $(\gamma, a, b)-$dense and
$(p,\frac{6\gamma + \alpha_\T}{1-c})-$\diffuse~on $\U$,
\end{enumerate}
then there exists an efficient submetric learner which produces a hypothesis $h_R$ with probability greater than $1- \dL$ such that
\begin{enumerate}
    \item $\Pr_{x,y \sim \U \times \U}[h_R(x,y) \geq \D(x,y) + \alpha_\T] \leq \eL.$
    \item $h_R$ is $(p - (1-a)^2 - \eL, c)-$nontrivial for $\U$.
\end{enumerate}
which runs in time $O(poly(\frac{1}{b}\ln(\frac{1}{b\delta}),|\T|,\frac{1}{\eL},\frac{1}{\dL}))$ for all $\eL,\dL \in (0,1]$.
\end{theorem}

\begin{proof}
Claim: Algorithm \ref{alg:submetriclearner} parametrized with a set of thresholds and learners as specified in Assumption \ref{assumption:pacthreshold} and $b$ for a $(\gamma,a,b)-$dense metric is an efficient submetric learner as specified in the Theorem statement.
We prove the claim with respect to each aspect of the theorem separately for clarity.

\textbf{Running time.}
Algorithm \ref{alg:submetriclearner} makes a single call to Algorithm \ref{alg:combiner} which runs in time $O(poly(|R|,|\T|,\frac{1}{\eR},\frac{1}{\dR}))$ (per Theorem \ref{theorem:simplegeneralization}). The parameters are set such that $|R|=\frac{1}{b}\ln(\frac{2}{b\dL})$, $\eR=\eL$ and $\dR=\dL/2$. Thus Algorithm \ref{alg:submetriclearner} runs in time $O(poly(\frac{1}{b}\ln(\frac{1}{b\dL}), |\T|, \frac{1}{\eR}, \frac{1}{\dR} ))$.



\textbf{Failure probability.}
The failure probability $\dL$ is split evenly between the failure to produce a good set of representatives (per Lemma \ref{nips:lemma:randomrepgeneralization}) and the failure probability of Algorithm \ref{alg:combiner}.

\textbf{Overestimate error probability.}
Algorithm \ref{alg:combiner} is invoked directly with $\eL$, so the overestimate error probability is $\eL$.

\textbf{Nontriviality.} The probability that the set of randomly chosen representatives in Algorithm \ref{alg:submetriclearner} does not form a $3\gamma$-net for at least $a$ weight of $\U$ is less than or equal to $\dL/2$ per Lemma \ref{nips:lemma:randomrepgeneralization}. Given that the randomly chosen representatives do form a $3\gamma-$net, notice that
as in Corollary \ref{nips:lemma:fullnetc}.1,  
if a $p-$fraction of distances in $\U \times \U$ have distance greater than $\frac{6\gamma + \alpha_\T}{1-c}$ that $h_R$ is $(p - (1-a)^2 - \eL,c)-$nontrivial.
\end{proof}

In the spirit of Corollary \ref{cor:simplegeneralization}, we can also re-state Theorem \ref{theorem:maingeneralization} in terms of arbitrary learners for $h_r$, rather than constructing directly from threshold function learners.

\begin{corollary}\label{cor:maingeneralization}
Given a distance metric $\D$ and a distribution $\U$ over the universe 
if
\begin{enumerate}
    \item There exist a set of efficient learners $L=\{L_r\}$ such that given access to labeled samples, each $L_r$ produces a hypothesis $h_r$ such that $\Pr_{x,y \sim \U \times \U}[|h_r(x,y) - \D(u,v)| \geq \alpha] \leq \er$ with probability at least $1-\dr$, and
    \item $\D$ is $(\gamma, a, b)-$dense and
$(p,\frac{6\gamma + \alpha}{1-c})-$\diffuse~on $\U$,
\end{enumerate}
then there exists an efficient submetric learner which produces a hypothesis $h_R$ with probability greater than $1- \dL$ such that
\begin{enumerate}
    \item $\Pr_{x,y \sim \U \times \U}[h_R(x,y) \geq \D(x,y) + \alpha] \leq \eL.$
    \item $h_R$ is $(p - (1-a)^2 - \eL, c)-$nontrivial for $\U$.
\end{enumerate}
which runs in time $O(poly(\frac{1}{b}\ln(\frac{1}{b\delta}),\frac{1}{\eL},\frac{1}{\dL}))$ for all $\eL,\dL \in (0,1]$.
\end{corollary}

Theorem \ref{theorem:maingeneralization} can also be restated to take into account postprocessing to reduce the additive error. In particular, Corollary \ref{cor:maingeneralizationpost} trades an increase of $\alpha$ in \diffusion~for a reduction of $\alpha$ in the error guarantee.
\begin{corollary}\label{cor:maingeneralizationpost}
Given a distance metric $\D$ and a distribution $\U$ over the universe 
if
\begin{enumerate}
    \item There exist a set of efficient learners $L=\{L_r\}$ such that given access to labeled samples, each $L_r$ produces a hypothesis $h_r$ such that $\Pr_{x,y \sim \U \times \U}[|h_r(x,y) - \D(u,v)| \geq \alpha] \leq \er$ with probability at least $1-\dr$, and
    \item $\D$ is $(\gamma, a, b)-$dense and
$(p,\frac{6\gamma + 2\alpha}{1-c})-$\diffuse~on $\U$,
\end{enumerate}
then there exists an efficient submetric learner which produces a hypothesis $h_R$ with probability greater than $1- \dL$ such that
\begin{enumerate}
    \item $\Pr_{x,y \sim \U \times \U}[h_R(x,y) \geq \D(x,y) ] \leq \eL$
    \item $h_R$ is $(p - (1-a)^2 - \eL, c)-$nontrivial for $\U$.
\end{enumerate}
which runs in time $O(poly(\frac{1}{b}\ln(\frac{1}{b\delta}),\frac{1}{\eL},\frac{1}{\dL}))$ for all $\eL,\dL \in (0,1]$.
\end{corollary}

\subsubsection{\Human~Query Complexity}
We now formally reason about the query complexity to the \human~to generate training data for Algorithm \ref{alg:submetriclearner}.
\begin{theorem}\label{thm:querycomplexity}
Assuming the minimum gap between any pair of thresholds in $\T $ is at most a constant $\alpha$, sufficient labeled training data for Algorithm \ref{alg:submetriclearner} can be produced from  $O(\maxalphlogRb)$ queries to $\mreal$ and $\frac{1}{b} \ln(\frac{1}{b\delta})\hat{N}\log(\hat{N})$ queries to $\mtrip$, and $O(\frac{1}{b} \ln(\frac{1}{b\delta})\hat{N}\log(\frac{1}{b} \ln(\frac{1}{b\delta})))$ queries to $\mquad$, where $\hat{N} = O(poly(\frac{1}{b}\ln(\frac{1}{b\delta}), \frac{1}{\alpha}, \frac{1}{\varepsilon}, \frac{1}{\delta})$ is the number of samples required to train a single threshold learner. 
\end{theorem}
\begin{proof}
First, notice that no assumption is made in the proofs of error or failure probability which requires independence of error or failure probability of the threshold function hypotheses. Thus, labeling a single set of training data at granularity $\alpha$ is sufficient for all of the threshold function learners for a single representative, i.e. we do not need to label a new sample for each threshold function.

Call the number of samples needed to train a threshold function $\hat{N}$. Recall from Assumption \ref{assumption:pacthreshold} that the threshold function learners run in time $O(poly(\frac{1}{\et},\frac{1}{\dt}))$. The choice of parameters in Algorithms \ref{alg:submetriclearner}, \ref{alg:combiner} and \ref{alg:thresholdcombiner} result in \[\et = \frac{\er}{2|\T|} = \frac{\eR}{2|R||\T|}=\frac{\eL\alpha}{\frac{2}{b}\ln(\frac{2}{b\dL})}\]
and
\[\dt = \frac{\dr}{|\T|} = \frac{\dR}{|R||\T|} = \frac{\dL\alpha}{\frac{2}{b}\ln(\frac{2}{b\dL})}\]
Thus the threshold function learners run in time $O(poly(\frac{1}{\eL}, \frac{1}{\dL}, \frac{1}{\alpha}, \frac{1}{b}\ln(\frac{1}{b\dL})))$, and can use no more than that many samples.

Recall from Theorem 
\ref{theorem:multiplerepquad}
that to produce labels for a set of elements of size $N$ to granularity $\alpha$ for a set of representatives $R$, we required at most $O(\maxalphlogR)$ queries to $\mreal$ and $O(|R|N\log(N))$ queries to $\mtrip$ and $O(|R|N\log(|R|))$ queries to $\mquad$. Therefore to produce 
labels for a universe of size $\hat{N}=O(poly(\frac{1}{\eL}, \frac{1}{\dL}, \frac{1}{\alpha}, \frac{1}{b}\ln(\frac{1}{b\dL})))$, for a set of representatives of size $|R|=O(\frac{1}{b}\ln(\frac{1}{b\delta}))$ we will require  $O(\maxalphlogRb)$ queries to $\mreal$ and $\frac{1}{b} \ln(\frac{1}{b\delta})\hat{N}\log(\hat{N})$ queries to $\mtrip$ and $\frac{1}{b} \ln(\frac{1}{b\delta})\hat{N} \log(\frac{1}{b}\ln(\frac{1}{b\delta}))$ queries to $\mquad$.
\end{proof}

\section{Relaxed query model}\label{section:tctc}
In this section, we extend our results to a second, relaxed \arbiter~model, in which \arbiters~are not expected to make arbitrarily small distinctions between distances or individuals or to provide arbitrarily precise real-valued distances. 
The relaxed model assumes that there are two fixed constants, $\alphL$, the minimum precision with which the \arbiter~can distinguish elements or distances, and $\alphH$, a bound on the magnitude of the (potentially biased) noise in the \arbiter's real-valued responses. 
For any comparisons with difference smaller than $\alphL$, the \arbiter~declares the elements indistinguishable or the difference ``too close to call.'' The model allows for a ``gray area'' between $\alphL$ and $\alphH$ in which the \arbiter~may either respond with the true answer or ``too close to call.'' For any differences larger than $\alphH$, the arbiter responds with the true answer. 
We assume that $\alphL$ and $\alphH$ are fixed constants for each task and cannot be manipulated.\footnote{i.e., we cannot parametrize $\alphL$ or $\alphH$ to use the \arbiter~as an arbitrary threshold distinguisher to improve query complexity and accuracy tradeoffs.}


\begin{definition}[\Rvqtctc] $\mrealtctc(u,v):=\D(u,v)\blueversion{\pm \eta}$, \blueversion{where $|\eta| < \alphH$} i.e. the \human~provides a real-valued distance between $u$ and $v$\blueversion{~with error at most $\alphH$}. 
\end{definition}
\begin{definition}[\Tqtctc] Given $a,b,c \in U$, define $\mathsf{diff}:=|\D(a,b) - \D(a,c)|$.
  \[
 \mtriptctc(a,b,c):=
  \begin{cases}
   \text{if }\mathsf{diff} \leq \alphL  & -1\\
   \text{if } \mathsf{diff} \in (\alphL,\alphH) & -1 \text{ or } [1 \text{ if }\D(a,b) < \D(a,c)\text{, }0 \text{ if }\D(a,c)\leq \D(a,b)]\\
   \text{if } \mathsf{diff}\geq \alphH & 1 \text{ if }\D(a,b) < \D(a,c)\text{, }0 \text{ if }\D(a,c)\leq \D(a,b)
  \end{cases}
\]
\end{definition}

\begin{definition}[\Quadqtctc]
  Given $a,b,x,y \in U$, define $\mathsf{diff}:=|\D(a,b) - \D(x,y)|$.
    \[
   \mquadtctc(a,b,c):=
    \begin{cases}
     \text{if }\mathsf{diff} \leq \alphL  & -1\\
     \text{if } \mathsf{diff} \in (\alphL,\alphH) & -1 \text{ or } [1 \text{ if }\D(a,b) < \D(x,y)\text{, }0 \text{ if }\D(x,y)\leq \D(a,b)]\\
     \text{if } \mathsf{diff}\geq \alphH & 1 \text{ if }\D(a,b) < \D(x,y)\text{, }0 \text{ if }\D(x,y)\leq \D(a,b)
    \end{cases}
  \]
\end{definition}

In addition to the \arbiter~consistency assumptions enumerated in Section \ref{section:preliminaries} (all \arbiters~agree, query responses do not change over the learning period, real-valued and relative query responses are consistent), we also assume that if the \arbiter~answers that the distances between ($a$ and $b$) and ($x$ and $y$) are indistinguishable, then the real-valued distances will also be at most $\alphH$ apart, and analogously that if the distances are distinguishable, then the real-valued distances will be at least $\alphL$ apart.

In the remainder of this section, we extend our results from the original ``\exact'' model to this ``\tctc'' (TCTC) model. We show that the \tctc~model allows for a significant reduction in real-valued query complexity (from logarithmic to constant) but at the cost of always having perceivable additive error in the submetrics produced, i.e. no $\alpha-$submetric for $\alpha<2\alphH$ can be achieved without postprocessing and a corresponding trade-off in non-triviality parameters.

\subsection{Submetrics from human judgements in the \tctc~model}
We first extend the results of Section \ref{section:humansubmetrics} to the \tctc~model. Algorithm \ref{alg:orderingtometricBLUE} is the \tctc~analog of Algorithm \ref{alg:orderingtometric}.\footnote{As before, we use $\midpointof$ to specify the midpoint function, which chooses the midpoint for audit length lists and rounds down for even length lists.} Algorithm \ref{alg:orderingtometricBLUE} follows the same basic recipe of sorting and then labeling, but the sorting step produces sorted \textit{sets} whose distances  from the representative are indistinguishable.
The elements in each set are then labeled with the distance between a distinguished element in the set and the representative. Thus the error of a distance label is a combination of the error of the real valued query for the distinguished element and the difference with the distinguished element's distance to the representative.

\begin{algorithm}[]
  \caption{\showalgcaption{$\orderingtosubmetrictctc(\ord, r, \alpha, \mrealtctc)$}}
  \label{alg:orderingtometricBLUE}
\begin{algorithmic}[1]
  \LineComment{Inputs: the universe $U$, the representative element $r$, interfaces to the \human, $\mtriptctc$ and $\mrealtctc$.}
  \Procedure{$\orderingtosubmetrictctc$}{$U,r,\mrealtctc$}
  \State $(\ord, \leftovers) \leftarrow \mathsf{TripletOrdering}(\mtriptctc, U, r)$
  \For{$x \in \ord$}
    \State $f_r(x) \leftarrow \mrealtctc(r,x)$
  \EndFor
  \For{$(y,x) \in \leftovers$}
    \State $f_r(y) \leftarrow f_r(x)$
  \EndFor
  \State \textbf{return} $\D_r'(x,y):= |f_r(x) - f_r(y)|$
  \EndProcedure
  \State
  \Function{$\mathsf{TripletOrdering}$}{$\mtriptctc,U,r$}
  \State Initialize an empty list $\ord$
  \State Initialize an empty list $\leftovers$
  \State Append $\{r\}$ to $\ord$
\For{$x \in U$}
    \State $\mathsf{\binaryinserttctc}(r,\leftovers,\ord,0,\size(\ord),x,\mtriptctc)$
  \EndFor
  \State \textbf{return} $\ord, \leftovers$
  \EndFunction
  \State
  \LineComment{Inputs: $r$ the representative, $\leftovers$ a list of ``indistinguishable'' element pairs, $L$ an ordered list of elements by distance from $r$, the list range delimiters $b$ and $e$, the element to insert $x$, an interface to the \human, $\mtriptctc$. Inserts the element $x$ into the appropriate position in the sorted list $L$ or adds $(x,y)$ to the $\leftovers$ if $x$ is indistinguishable from an element $y$ already contained in $L$.}
  \Function{$\binaryinserttctc$}{$r,\leftovers,L,b,e, x,\mtriptctc$}
   \If{$b=e$}
     \State $c \leftarrow \mtriptctc(r,x,L[b])$
     \If{$c=0$}
         \State Insert $x$ at position $b+1$ in $L$
     \ElsIf{$c=1$}
         \State Insert $x$ at position $b-1$ in $L$
     \Else
        \State Append $(x, L[b])$ to $\leftovers$.
     \EndIf
     \State \textbf{return}
   \EndIf
  \State $\midpoint \leftarrow \midpointof(L)$
  \State $c\leftarrow \mtriptctc(r,x,L[\midpoint]))$
  \If{$c =0$}
    \State $\binaryinserttctc(r,\leftovers,L,\midpoint,e,(r,x),\mtriptctc)$
  \ElsIf{$c=1$}
    \State $\binaryinserttctc(r,\leftovers,L,b,\midpoint,(r,x),\mtriptctc)$
  \Else
    \State Append $(x, L[\midpoint])$ to $\leftovers$.
    \State \textbf{return}
  \EndIf
  \EndFunction
  \end{algorithmic}
\end{algorithm}

Theorem \ref{theorem:labelqueriesBLUE} is the \tctc~analog of Theorem \ref{theorem:labelqueries}, and states that Algorithm \ref{alg:orderingtometricBLUE} requires only $O(\frac{1}{\alphL})$ real-valued queries and $O(N\log(N))$ triplet queries to produce a $4\alphH$-submetric.
The proof of the theorem follows from observing that there are at most $O(\frac{1}{\alphL})$ sets of elements with indistinguishable differences in distance (i.e. differences of less than $\alphL$) in $[0,1]$, and that order $\alphH$ errors in the distance labels accrue in the mapping to distinguish elements and in the real valued queries for the distinguished element. The primary difference between Algorithms \ref{alg:orderingtometricBLUE} and \ref{alg:orderingtometric} is that in Algorithm \ref{alg:orderingtometricBLUE} the \arbiter~identifies when elements or distances indistinguishable (difference less than $\alphL$) from relative queries alone. Thus Algorithm \ref{alg:orderingtometricBLUE} groups indistinguishable elements together and acts only on the subset of distinguishable elements, which has size bounded by $O(\frac{1}{\alphL})$. The removal of the $\log(N)$ dependency in the number of real-valued queries compared with Lemma \ref{lemma:orderingtometric} follows from the query model explicitly preventing any recursive calls on ranges of size less than $\alphL$.
\begin{theorem}\label{theorem:labelqueriesBLUE}
Given access to $\mrealtctc$ and $\mtriptctc$, \blueversion{Algorithm \ref{alg:orderingtometricBLUE} constructs a $4\alphH-$submetric} from \blueversion{$O(\frac{1}{\alphL}$)} queries to $\mrealtctc$ and \blueversion{$O(N\log(\frac{1}{\alphL}))$} queries to $\mtriptctc$
which preserves distances (up to the additive error) from a representative $r$.
\end{theorem}
\begin{proof}

\textbf{Query complexity.}
Queries to $\mtriptctc$ are made only by $\binaryinserttctc$. Notice that if an element in $\binaryinserttctc$ is ever within distance $\alphL$ of an existing item in the list, it is added to $\leftovers$ rather than the working ordering, $L$. Thus, although $\binaryinserttctc$ is called for each element, it operates on a list of size at most $O(\frac{1}{\alphL})$, as there are at most $O(\frac{1}{\alphL})$ elements with distances from $r$ which are different by at least $\alphL$, and $L$ contains at most one element from each indistinguishable set. Thus, it makes $O(\log(\frac{1}{\alphL}))$ recursive calls for each element, yielding the desired bound of $O(N\log(\frac{1}{\alphL}))=O(N)$ queries to $\mtriptctc$.
The desired bound on real-valued queries follows from observing that the ordering labeled by $\mrealtctc$ (Lines 3-5) has at most $O(\frac{1}{\alphL})$ elements.

\textbf{Correctness.} Each element is considered by the algorithm, and is either sorted into the correct position in $L$ via binary search or, if its distance is indistinguishable from that of another element in $L$, it is added to an associated set of indistinguishable elements. Once the elements are sorted, each element $x$ is either labeled with its true distance with less than $\alphH$ error, i.e. $\mrealtctc(r,x)$, or the distance of a distinguished element whose distance from $r$ is within $\alphH$ of its distance from $r$. Thus $|f_r(y)-\D(r,y)| < 2\alphH$, accounting for the additional error in evaluations of $\mrealtctc$. Therefore $|\D_r'(x,y) - |\D(r,x) - \D(r,y)|| \leq 4 \alphH$.
\end{proof}

Theorem \ref{theorem:multiplerepquadBLUE} likewise extends the results of Theorem \ref{theorem:multiplerepquad} to the \tctc~model using the same observations as in the proof of Theorem \ref{theorem:labelqueriesBLUE}. As in Algorithm \ref{alg:listmerge}, Algorithm \ref{alg:listmergeBLUE} uses quad queries to sort (element, representative) pairs, and then labels the resulting list at the specified granularity. As in Algorithm \ref{alg:orderingtometricBLUE}, the sorting step produces a sorted list of indistinguishable (element, representative) pair sets of size bounded by $O(\frac{1}{\alphL})$.
\begin{theorem}\label{theorem:multiplerepquadBLUE}
Given a set of representatives $R$ and access to $\mrealtctc$ and $\mquadtctc$,  \blueversion{a $4\alphH-$submetric} can be constructed from \blueversion{$O(\frac{1}{\alphL})$} queries to $\mrealtctc$ and \blueversion{$O(|R|N\log(\frac{1}{\alphL}))$} queries to $\mquadtctc$ which preserves distances (up to the additive error) from the set of representatives $R$.
\end{theorem}


\begin{proof}
Consider Algorithm
\ref{alg:listmergeBLUE}.

\textbf{Query complexity.}
Queries to $\mquadtctc$ are made only by $\binaryinsertpairtctc$. As in the proof of Theorem \ref{theorem:labelqueriesBLUE}, if an element in $\binaryinsertpairtctc$ is ever within distance $\alphL$ of an existing item in the list, it is added to $\leftovers$ rather than the working ordering, $L$. Thus, although $\binaryinsertpairtctc$ is called for each element $|R|$ times, it operates of a list of size at most $O(\frac{1}{\alphL})$ so it makes $O(\log(\frac{1}{\alphL}))$ recursive calls for each element, yielding the desired bound of $O(|R|N\log(\frac{1}{\alphL}))$ queries to $\mquadtctc$.
The desired bound on real-valued queries follows from observing that the ordering labeled by $\mrealtctc$ (Lines 3-5) has at most $O(\frac{1}{\alphL})$ elements, as in the proof of Theorem \ref{theorem:labelqueriesBLUE}.

\textbf{Correctness.} By the same logic as in the proof of Theorem \ref{theorem:labelqueriesBLUE}, each element $x$ is either labeled with its true distance with at most $\alphH$ error, $\mrealtctc(r,x)$ or the distance of a distinguished element whose distance from $r$ is within $\alphH$ of its distance. Thus $|f_r(y)-\D(r,y)| < 2\alphH$, accounting for the additional error in evaluations of $\mrealtctc$. Therefore $|\D_r'(x,y) - |\D(r,x) - \D(r,y)|| \leq 4 \alphH$.
\end{proof}

\begin{algorithm}[]
  \caption{\showalgcaption{$\multipleorderingtosubmetrictctc(\ord, r, \alpha, \mrealtctc)$}}
  \label{alg:listmergeBLUE}
\begin{algorithmic}[1]
  \LineComment{Inputs: the universe $U$, the representative element set $R$, interfaces to the \human, $\mquadtctc$ and $\mrealtctc$.}
  \Procedure{$\multipleorderingtosubmetrictctc$}{$U,R,\mrealtctc$, $\mquadtctc$}
  \State Initialize an empty ordering $\ord$
  \State Initialize an empty list $\leftovers$
  \For{$r \in R$}
    \For{$x \in U$}
      \State $\binaryinsertpair(r,\leftovers, \ord, 0,\size(\ord),x,\mquadtctc)$
    \EndFor
  \EndFor

  \For{$(x,r) \in \ord$}
    \State $f_r(x) \leftarrow \mrealtctc(r,x)$
  \EndFor
  \For{$((y,r_1),(x,r_2)) \in \leftovers$}
    \State $f_{r_1}(y) \leftarrow f_{r_2}(x)$
  \EndFor
\For{$r \in R$}
  \State $\D_r'(x,y):= |f_r(x) - f_r(y)|$
  \EndFor
  \State \textbf{return} $\DconstR(x,y):= \maxmerge(\{\D_r' | r \in R\},x,y)$
  \EndProcedure
   \State
  \LineComment{Inputs: $r$ the representative, $\leftovers$ a list of pairs of (representative, element) pairs, $L$ an ordered list of (representative, element) pairs, the list range delimiters $b$ and $e$, the element, representative pair to insert $(x,r)$, an interface to the \human, $\mquadtctc$.  Inserts the (element, representative) pair $(x,r)$ into the appropriate position in the sorted list $L$ or adds $((x,r),(y,r_2))$ to the $\leftovers$ if the distance between $(x,r)$ is indistinguishable from the distance between $(y, r_2)$ already contained in $L$.}
  \Function{$\binaryinsertpairtctc$}{$r,\leftovers,L,b,e, (x,r), \mquadtctc$}
   \If{$b=e$}
     \State $(y,r_2) \leftarrow L[b]$
     \State $c \leftarrow \mquadtctc((r,x),(r_2,y))$
     \If{$c=0$}
         \State Insert $(x,r)$ at position $b+1$ in $L$
     \ElsIf{$c=1$}
         \State Insert $(x,r)$ at position $b-1$ in $L$
     \Else
        \State Append $((x, r), (y,r_2))$ to $\leftovers$.
     \EndIf
     \State \textbf{return}
   \EndIf
  \State $\midpoint \leftarrow \midpointof(L)$
  \State $c\leftarrow \mquadtctc((x,r),L[\midpoint]))$
  \If{$c =0$}
    \State $\mathsf{BinaryInsert}(r,\leftovers,L,\midpoint,e,(x,r),\mquadtctc)$
  \Else
    \State $\mathsf{BinaryInsert}(r,\leftovers,L,b,\midpoint,(x,r),\mquadtctc)$
  \EndIf
  \EndFunction

  \end{algorithmic}
\end{algorithm}

\paragraph{Bounds on perceivable error.}
Unlike in the \exactarbiter~model in which the additive error of a submetric can be made an  arbitrarily small constant, Algorithms \ref{alg:orderingtometricBLUE} and \ref{alg:listmergeBLUE} result in additive error at least $4\alphH$. A reasonable question is whether any query procedure in the \tctc~model can produce a submetric with no perceivable additive error without some additional post-processing. Indeed, even with the naive construction of asking $O(N)$ real-valued queries the submetric produced can have additive error strictly greater than $\alphH$, without further post-processing.

\begin{proposition}\label{lemma:tctcerrorlowerbound}
The representative submetric $\D_r(x,y):= |\mrealtctc(r,x) - \mrealtctc(r,y)|$ can have additive error greater than $\alphH$.
\end{proposition}
The proof of the proposition follows from the observation that each query to $\mrealtctc$ has error of $\eta $ such that $|\eta| < \alphH$. Thus, in the worst case, when $\eta >\alphH/2$, $(\D(r,x) + \eta) - (\D(r,y) - \eta) = \D(r,x) - \D(r,y) + 2\eta > \D(r,x) - \D(r,y) + \alphH$.
So we cannot expect to produce submetrics without perceptible error without some additional post-processing on the values queried from the \arbiter.

\subsection{Generalization}
We now turn our attention to extending the generalization results of Sections \ref{section:generalizinghumans} and \ref{section:choosingreps}.
Notice that unlike the \exact~model we won't necessarily be able to label a sample with 100\% accuracy for every threshold function for every representative.
The key problem is that in the \tctc~model each element's distance from a representative has \textit{bi-directional} error, i.e. it can have either over or under-estimated distance from the representative. This bi-directional error prevents us from using the nice properties of a consistent underestimator, so  translating to the desired binary labels for a given threshold is not straightforward. To get around this labeling problem, we modify the distribution of samples presented to each learner, in particular eliminating samples whose labels are ambiguous. We then reason about the error of the combination of hypotheses for distributions with disjoint sets of ambiguous points removed. 

Recall that Algorithm \ref{alg:listmergeBLUE} assigns distances $f_r(x)$ such that $|f_r(x) - \D(r,x)| \leq 2\alphH$. Thus, any element $x$ such that $f_r(x) > t_i + 2\alphH$ is truly greater than distance $t_i$ from $r$, (and analogously less than $t_i$ if $f_r(x)<t_i - 2\alphH$). Intuitively, this means that we can generate accurate threshold function labels for points sufficiently $(2\alphH)$ far from the threshold.
We formally define the \thresholdout~below to capture only the elements whose relative distances are unambiguous.



\begin{definition}[\Thresholdout]
Given a distribution $\U$ over a universe of individuals $U$, a representative $r$, 
a labeling procedure which produces (noisy) distance labels with bi-directional additive error of at most $2\alphH$ from the representative $r$, 
and a threshold $t$, the \thresholdout~$\Utr$ is the re-normalized distribution $\U$ with all weights on elements labeled with distances within $2\alphH$ of $t$ set to 0.
\end{definition}

Notice that the \thresholdout~is well-defined without knowledge of the exact distances from the representative, as it is specified based on the labeling procedure, rather than exact distances from the representative itself. Thus, we can reason about learning on the distribution $\Utr$ without worrying about whether any elements are ambiguously labeled.
Algorithm \ref{alg:labelgenBLUE} specifies a labeling procedure for training data for each of the threshold functions for the distributions $\Utr$ for $t \in \T$, i.e., only samples with unambiguous labels.

\begin{algorithm}[]
  \caption{\showalgcaption{$\mathsf{GenerateLabels}$}}
  \label{alg:labelgenBLUE}
\begin{algorithmic}[1]
  \LineComment{Inputs: a sample $ S \sim \U$, the representative element set $R$, a set of thresholds $\T$, interfaces to the \human, $\mquadtctc$ and $\mrealtctc$. Returns a set of samples from $\Utir$ for each representative in $R$ and each threshold in $\T$.}
  \Procedure{$\mathsf{GenerateLabels}$}{$S,R,\T,\mrealtctc$, $\mquadtctc$}
  \State Initialize an empty ordering $\ord$
  \State Initialize an empty list $\leftovers$
  \For{$r \in R$}
    \For{$x \in S$}
      \State $\binaryinsertpair(r,\leftovers, \ord, 0,\size(\ord),x,\mquadtctc)$
    \EndFor
  \EndFor
  \For{$t \in \T$}
    \For{$r \in R$}
        \State Initialize an empty list of (element, label) pairs $M_t^r$
    \EndFor
\EndFor
\For{$(x,r) \in \ord$}
    \State $f_r(x) \leftarrow \mrealtctc(r,x)$
  \EndFor
  \For{$((y,r_1),(x,r_2)) \in \leftovers$}
    \State $f_{r_1}(y) \leftarrow f_{r_2}(x)$
  \EndFor
  \For{$t \in \T$}
  \For{$r \in R$}
  \For{$x \in S$}
        \If{$f_r(x) > t + 2\alphH$}
            \State Append $(x, 1)$ to $M_t^r$ \Comment{Unambiguous label.}
        \ElsIf{$f_r(x) < t - 2\alphH$}
            \State Append $(x, 0)$ to $M_t^r$ \Comment{Unambiguous label.}
        \EndIf
    \EndFor
    \EndFor
    \EndFor
  \State \textbf{return} $\{M_t^r | t \in \T, r \in R\}$ 
  \EndProcedure
  \end{algorithmic}
\end{algorithm}

In order for the threshold learners to succeed, we need sufficient labeled training data for each threshold function for each representative. However, Algorithm \ref{alg:labelgenBLUE} tosses out any examples for a given learner which are too close to the threshold value. Thus, there is some risk that there are too few samples produced for a given threshold. Lemma \ref{lemma:labelgenBLUE} below states that Algorithm \ref{alg:labelgenBLUE} generates a sufficient number of  labeled samples of $\Utir$ all but one $t_i \in \T$ for each representative given an initial sample of size $|S|=3\hat{m}$, where $\hat{m}$ is the number of labeled samples required for each threshold function learner.

\begin{lemma}\label{lemma:labelgenBLUE}
Given a set of samples $S \sim \U$ of size $3\hat{m}$ and a set of thresholds $\T$ such that $|t_i-t_j|>\labelbound$ for all $i,j$, Algorithm $\ref{alg:labelgenBLUE}$ produces at least $\hat{m}$ labeled examples from $\Utir$ for at least $|\T|-1$ of the thresholds in $\T$.
\end{lemma}
\begin{proof}
\textbf{Correctness of labels.} Recall from the proof of Theorem \ref{theorem:multiplerepquadBLUE}, that each element $u \in U$ is labeled with a distance $f_r(x)$ such that $|f_r(x)-\D(r,x)| \leq \labelbound$ for each representative $r \in R$. Thus, each element labeled above or below $t_i$ which is at least $\labelbound$ distant from $t_i$ is correctly labeled for the representative $r$.

\textbf{Quantity of labeled examples.}
Since each threshold is at least $\labelbound$ away from its neighboring thresholds, therefore any element which is discarded for $t_i$ is included as a labeled example for every other threshold $t_{j\neq i}$ for a representative. Suppose that at least one threshold $t_k$ has fewer than $\hat{m}$ labeled examples for a representative. Then at least $2\hat{m}$ examples were discarded for $t_k$ for the representative. However, this leaves at most $\hat{m}$ samples which could be discarded for any other threshold, so all of the other thresholds must have at least $2\hat{m}$ labeled samples for this representative. Thus, at most one threshold will have fewer than $\hat{m}$ labeled samples for each representative.

\textbf{Correctness of distribution.}  $\Utir$ is defined with respect to the labeling procedure, and thus it is possible to simulate $\Utir$ by labeling elements and discarding an elements whose labels are ambiguous. Thus Algorithm \ref{alg:labelgenBLUE} simulates $\Utir$, and the sets of labeled data produced are indistinguishable from a set drawn from $\Utir$ directly.

\end{proof}

With the labeling procedure in place, we now introduce Algorithm \ref{alg:thresholdcombinerBLUE}, the \tctc~analog of Algorithm \ref{alg:submetriclearner}. As in Algorithm \ref{alg:submetriclearner}, Algorithm \ref{alg:thresholdcombinerBLUE} first samples a set of representatives, according to the size requirements of Lemma \ref{nips:lemma:randomrepgeneralization}, generates a set of labeled samples for each threshold function via Algorithm \ref{alg:labelgenBLUE} and then calls the threshold learners on the appropriate modified distributions with appropriately scaled parameters and combines their resulting hypotheses into a single hypothesis for the combined submetric.

To make the sample generation book-keeping clearer, we slightly modify the specification of the threshold learners (but not the core assumption) so that we can more clearly specify the sample distributions passed to each learner. We also introduce a minimum granularity for the thresholds determined by $\alphH$.

\begin{assumption}\label{assumption:pacthresholdBLUE}
Given a metric $\D$ and a representative $r$, there exists a set of thresholds $\T$ such that
\begin{enumerate}
    \item $t \in [0,1]$ for all $t \in \T$,
    \item $0 \in \T$,
    \item $\threshspace < \alpha_\T = \max_{i \in [|\T|-1]}t_{i+1} - t_{i}$,
    \item $|\T|= O(1)$,
\end{enumerate}
and for every $t \in \T$ there exists an efficient learner $L_t^r(\et,\dt,\tctclabeledsamples)$ which for all $\et,\dt \in (0,1]$, with probability at least $1-\dt$ produces a hypothesis $h_t^r$ such that
\[\Pr_{x \sim \tctclabeledsamples}[h_t^r(x) \neq T_t^r(x)] \leq \et\] in time $O(poly(\frac{1}{\et},\frac{1}{\dt}))$ with access to labeled samples of $T_t^r(u \sim \tctclabeledsamples)$ for any distribution $\tctclabeledsamples$ over the universe. 
That is, the concept class $T_t^r$ is efficiently learnable for all $t \in \T$.
\end{assumption}

\begin{algorithm}[]
  \caption{\showalgcaption{a caption could go here.}}
  \label{alg:thresholdcombinerBLUE}
\begin{algorithmic}[1]
\LineComment{Inputs: error and failure probability parameters $\varepsilon,\delta$, density parameter $b$, a set of threshold function learners $\{L_{t_i \in \T}^r\}$, and the threshold set $\T$.}
  \Procedure{$\submetriclearnertctc$}{$\varepsilon, \delta, b,\{L_{t_i \in \T}^r\}, \T$}
  \State Sample $R \sim \U$ such that $|R| = \frac{1}{b}\ln(\frac{2}{b\dL})$.
  \State Initialize an empty list $L$.
  \State Choose a sample $S\sim \U$ such that $|S|=3\hat{m}$, the number of samples required for $L_{t_i \in \T}^r$.
  \State $M=\{M_t^r\} \leftarrow \mathsf{GenerateLabels}(U, R, \T, \mrealtctc, \mquadtctc)$
  \For{$r \in R$}
  \State Initialize $\T_r\leftarrow \T$
  \For{$t \in \T$}
  \If{$|M_t^r|< \hat{m}$}
    \State $\T_r \leftarrow \T \backslash \{t\}$
  \EndIf
  \EndFor
  \EndFor
  \For{$r \in R$}
    \State $L_r(\er,\dr) := \thresholdcombinertctc(\T_r, \{L_{t_i \in \T}^r\},\{M_{t}^r | t \in \T\},\er,\dr)$
    \State Add $L_r$ to $L$.
  \EndFor
  \State \textbf{return} $h_R(u,v) = \combiner(L, \varepsilon,\delta/2)$
  \EndProcedure
\State
\LineComment{Inputs: a set of learners $\{L_r\}$ for each representative $r \in R \subseteq U$, and error and failure probability parameters $\eR,\dR$.}
  \Function{$\combiner$}{$L=\{L_r\}, \eR, \dR$}
  \State Initialize an empty list $H_R$.
  \For{$L_r \in L$}
      \State $h_r \leftarrow L_i(\frac{\eR}{|R|},\frac{\dR}{|R|})$
      \State Add $h_r$ to $H_R$
  \EndFor
  \State \textbf{return} $h_R(u,v) := \maxmerge(H_R, u, v)$
  \EndFunction
  \State
    \LineComment{Inputs: $\T = \{t_i\}$, the set of thresholds. $L=\{L_{t_i}^r\}$ the set of learners for each threshold in $\T$ for a particular representative, $r$. A set of sets of labeled examples $\{M_t^r|r\in R\}$. Error and failure probability parameters $\er, \dr$.}
    \Function{$\thresholdcombinertctc$}{$\T,L=\{L_{t_i\in \T}^r\}, \{M_{t_i\in \T}^r\},\er, \dr$}
    \State Initialize an empty list of hypotheses $H_\T^r$.
    \For{$L_{t_i}^r \in L$}
      \State $h_{t_i}^r \leftarrow L_{t_i}^r(\frac{\er}{2|\T|},\frac{\dr}{|\T|},M_{t_i}^r)$
      \State Add $h_{t_i}^r$ to $H_\T^r$.
  \EndFor
  \State \textbf{return} $h_r(x,y) := |\linearvote(\T, H_\T^r, x) - \linearvote(\T,H_\T^r,y)|$

 \EndFunction
 \State
\end{algorithmic}
\end{algorithm}

In order to prove the analog of the combined \exactarbiter~generalization (Theorem \ref{theorem:maingeneralization})  we split the analysis into two steps. First, we analyze the error of $\thresholdcombinertctc$ running on the modified distributions and threshold sets and adjust the density and \diffusion~requirements accordingly.
We then complete the argument by analyzing the full $\submetriclearnertctc$ procedure parameter choices to derive the desired error and query complexity bounds.

Lemma \ref{lemma:thresholdcombinerBLUE} states that $\thresholdcombinertctc$, when parametrized with learners and samples from $\U_{t \in \T}^r$ where $\alpha_\T>\threshspace$ results in hypotheses which overestimate or underestimate distances by more no more than $\blueerr$ with probability $\leq \varepsilon$ with high probability.
\begin{lemma} \label{lemma:thresholdcombinerBLUE}
Given a set of thresholds such that for all $t_{i},t_{j} \in \T$, $|t_{i}-t_{j}|=\alpha_\T> \threshspace$,
a set of learners as in Assumption \ref{assumption:pacthresholdBLUE},
and access to sufficient labeled examples of $L_{t_{i}\in T}^{r}$ from $\Utir$ for all $t_i \in \T$,
with probability at least $1-\dr$  $\thresholdcombinertctc$ produces a hypothesis $h_r$ such that
\[\Pr[|h_r(x,y) - \D(x,y)|> \blueerr] \leq \er\]
i.e., the representative submetric is efficiently learnable.
\end{lemma}
\begin{proof}
The essence of the proof is to show that in order for $h_r(x,y)$ to differ from $|\threshunder(x)-\threshunder(y)|$ by more than $\blueerr$, then at least one threshold function other than the true thresholds for $x$ and $y$ must be in error.

\textbf{Labeling samples.} Lemma \ref{lemma:labelgenBLUE} states that sufficient labeled samples can be produced for all but one of the thresholds for each representative. Given that we must assume that we fail to produce labeled training data for one of the thresholds, the maximum gap between any pair of thresholds with sufficient training data is $2\alpha_\T$.

\textbf{Producing a sufficiently large error in $\linearvote$.} Consider an element $u$ such that its true distance from the representative $r$ is  between $t_{i}$ and $t_{i+1}$. That is, $T_{t_{> i}}^r(u)=1$ and $T_{t_{\leq i}^r}=0$. First, notice that for $\linearvote(\T,H_\T^r,u)$ to diverge from the true threshold, $t_i$, by more than two threshold values, at least one threshold function hypothesis other than $h_{t_i}$ must be in error.
Thus, even if $h_{t_i}$ is in error, at least one other $h_{t_j}$ must also be in error to produce $\linearvote(\T,H_\T^r,u) \in \{t_j | j > i + 2, j < i - 2\}$.
Thus, it is sufficient to reason about the probability that at least one threshold hypothesis, other than the correct threshold is in error.

\textbf{Error probability.} Our analysis of error probability takes the worst-case assumption that for every element in $\U \backslash \Utir$\footnote{For simplicity of notation, we use $\U \backslash \Utir$ to denote the set of items with positive weight in $\U$ but weight $0$ in $\Utir$.}, the hypothesis $h_{t_i}^r$ is in error. For each threshold $t_i \in \T$, we use $w_i$ to represent the weight of $\U \backslash \Utir$ under $\U$. Recall from the proof of Lemma \ref{lemma:labelgenBLUE} that all of the $\U\backslash \Utir$ are disjoint, so $\sum_i w_i \leq 1$. Notice that if each hypothesis is learned successfully, then it will have error probability at most $(1-w_i)\et$ to distribute on $\Utir$, and we assume that it always behaves badly on the weight $w_i$ region $\U \backslash \Utir$.
In the worst case, a threshold function can be in error outside of its ``bad'' region resulting in a mistake of more than $\blueerr$  with probability at most $|T|\et$. (Recall that the maximum gap between thresholds, accounting for the label generation is $2\alpha_\T$.)
Thus, the probability that $\linearvote$ produces a value at least $\blueerr$ distant from the true threshold value for either element in a pair drawn from $\U$ is at most $2\sum_{t_i \in \T}\frac{\er}{2|T|} =\er$.

\end{proof}

The final piece is to state the full generalization result including non-triviality guarantees. This theorem statement and proof are nearly identical to the \exactarbiter~versions, with modifications only to account for the difference in the additive error parameter.

\begin{theorem}\label{theorem:maingeneralizationBLUE}
Given a distance metric $\D$, and a distribution $\U$ over the universe
if
\begin{enumerate}
    \item There exist a set of thresholds $\T$ and efficient learners $\{L_{t_i \in \T}^r\}$ as in Assumption \ref{assumption:pacthresholdBLUE}, and
    \item $\D$ is $(\gamma, a, b)-$dense and
$(p,\frac{6\gamma + \blueerr}{1-c})-$\diffuse~on $\U$,
\end{enumerate}
then there exists an efficient submetric learner which produces a hypothesis $h_R$ with probability greater than $1- \dL$ such that
\begin{enumerate}
    \item $\Pr_{x,y \sim \U \times \U}[h_R(x,y) \geq \D(x,y) + \blueerr] \leq \eL.$
    \item $h_R$ is $(p - (1-a)^2 - \eL, c)-$nontrivial for $\U$.
\end{enumerate}
which runs in time $O(poly(\frac{1}{b}\ln(\frac{1}{b\delta}),|\T|,\frac{1}{\eL},\frac{1}{\dL}))$ for all $\eL,\dL \in (0,1]$.
\end{theorem}

\begin{proof}
Claim: Algorithm \ref{alg:thresholdcombinerBLUE} parametrized with a set of thresholds and learners as specified in Assumption \ref{assumption:pacthresholdBLUE} and $b$ for a $(\gamma,a,b)-$dense metric is an efficient submetric learner as specified in the Theorem statement.
We prove the claim with respect to each aspect of the theorem separately for clarity.

\textbf{Running time.} The running time argument is equivalent to the argument for the \exactarbiter~version, with the additional observation that labeling a sufficient number of samples requires an additional factor of $3$ samples.

\textbf{Failure probability.}
As in the \exactarbiter~version, the failure probability $\dL$ is split evenly between the failure to produce a good set of representatives (per Lemma \ref{nips:lemma:randomrepgeneralization}) and the failure probability of the representative submetric learners.

\textbf{Overestimate error probability.}
The argument proceeds as in the \exactarbiter~version, relying on the error analysis of Lemma \ref{lemma:thresholdcombinerBLUE}.

\textbf{Nontriviality.} The argument proceeds as in the \exactarbiter~version, with $\alpha_\T$ scaled to account for the additional error in the intermediate hypotheses $\{h_r | r \in R\}$.
\end{proof}
Finally, we re-state Theorem \ref{thm:querycomplexity} in the \tctc~model to account for the improved query complexity in label generation.
\begin{theorem}\label{thm:querycomplexitytctc}
Sufficient labeled training data for Algorithm \ref{alg:thresholdcombinerBLUE} can be produced from \blueversion{$O(\frac{1}{\alphL})$} queries to $\mreal$
and $O(\hat{N} \frac{1}{b}\ln(\frac{1}{b\delta}) \log(\frac{1}{\alphL}))$ queries to $\mquad$  where $\hat{N} = O(poly(\frac{1}{b}\ln(\frac{1}{b\delta}), \frac{1}{\alphH}, \frac{1}{\varepsilon}, \frac{1}{\delta})$  is the number of samples required to train a single threshold learner given a set of evenly spaced thresholds $\T$ such that $t_i - t_{i-1} > \threshspace =o(1)$.
\end{theorem}
\begin{proof}

Call the number of samples needed to train a threshold function $\hat{N}$. Recall from Assumption \ref{assumption:pacthreshold} that the threshold function learners run in time $O(poly(\frac{1}{\et},\frac{1}{\dt}))$. The choice of parameters in Algorithms \ref{alg:submetriclearner}, \ref{alg:combiner} and \ref{alg:thresholdcombiner} result in \[\et = \frac{\er}{2|\T|} = \frac{\eR}{2|R||\T|}=\frac{2\eL\alphH}{\frac{1}{b}\ln(\frac{2}{b\dL})}\]
and
\[\dt = \frac{\dr}{|\T|} = \frac{\dR}{|R||\T|} = \frac{4\dL\alphH}{\frac{1}{b}\ln(\frac{2}{b\dL})}\]
Thus the threshold function learners run in time $O(poly(\frac{1}{\eL}, \frac{1}{\dL}, \frac{1}{\alphH}, \frac{1}{b}\ln(\frac{1}{b\dL})))$, and can use no more than that many samples.

Recall from Lemma \ref{lemma:labelgenBLUE} that to label a set of $\hat{N}$ samples will require labeling $3\hat{N}$ samples via Algorithm \ref{alg:listmergeBLUE}. Recall from the proof of Theorem \ref{theorem:multiplerepquadBLUE} that such a set of labels can be produced from $O(\frac{1}{\alphL})$ queries to $\mrealtctc$ and $O(|R|\hat{N}\log(\frac{1}{\alphL}))$ queries to $\mquadtctc$.

Therefore to produce 
labels for a universe of size $\hat{N}=O(poly(\frac{1}{\eL}, \frac{1}{\dL}, \frac{1}{\alphH}, \frac{1}{b}\ln(\frac{1}{b\dL})))$, for a set of representatives of size $O(\frac{1}{b}\ln(\frac{1}{b\delta}))$ we will require  \blueversion{$O(\frac{1}{\alphL})$} queries to $\mrealtctc$ and $O(\hat{N} \frac{1}{b}\ln(\frac{1}{b\delta}) \log(\frac{1}{\alphL}))$ queries to $\mquadtctc$.
\end{proof}


As previously noted, although the query complexity guarantees are much improved in the \tctc~model, the additive error of Algorithm \ref{alg:thresholdcombinerBLUE} is greater than the perceivable error threshold of $\alphH$. Of course, the metric designer may choose to post-process the resulting hypothesis be reducing every reported difference by $\blueerr$, resulting in a $0$-submetric with $\epsilon$ error probability. However, this will require a corresponding increase in the \diffusion~parameter to maintain the same non-triviality guarantee, as in Corollary \ref{cor:maingeneralizationpost}.


\subsection*{Summary}
We have shown that in cases where the \arbiter~is not required to answer queries with specificity below a certain level of granularity, $\alphH$, that we can still achieve a small constant additive error of $O(\alphH)$ with a constant number of real-valued distance queries $O(\frac{1}{\alphL})$ and $O(|R|N\log(\frac{1}{\alphL}))$ relative distance queries. This additional model provides a good initial step towards handling imprecise arbiter decisions. Extending to other error modalities is an important direction for future work.

\section{Discussion}\label{section:discussion}
\subsection{Summary of main results}
In summary, we have established a useful framework of nontrivial submetrics as approximations to the true metric for Individual Fairness. We have also shown that constructing submetrics based on threshold rounding on distances from representative elements has both good over and under-estimate error properties.

We have examined a limited, realistic query model of relative distance queries and real valued queries, and have shown how to construct submetrics on a fixed universe of individuals with a sublinear number of real-valued queries and $O(|R|N\log(N))$ relative distance queries. These procedures are useful both as a complete solution for offline settings, where the whole universe to be classified is known in advance, and as a way to generate training data for other fair classification schemes.

We have also shown how to learn hypotheses for a submetric which generalize well to unseen samples based on limited assumptions of efficient learnability of threshold functions. We demonstrated a technique to obtain good nontriviality guarantees in a specific setting for two dimensional Euclidean distances, and a more general framework for reasoning about the performance of a small set of random representatives with a reasonable number of queries to the \human~to generate labeled examples for training.

With the statement of our results completed, we now pose several points of discussion and critique of the work as areas for future work and improvement.

\subsection{Metric structure}
Many settings where fairness is critical involve high-dimensional or unstructured data, e.g. college applications which include years of grades, test scores, free text essays and recommendations and many other features. In Section \ref{section:choosingreps}, we showed one special case in which the metric structure could be exploited to create more accurate submetrics with fewer representatives. How likely is it that the true metric is low-dimensional with such ``nice'' structure?  We contend this case is more likely than it may initially appear.
Consider a \human~tasked with determining similarity for college applicants. She cannot possibly hold the entire applicant's feature description in her working memory at once and compare it line by line with the next applicant's. Instead, the \human~likely has an intuitive model of what it means to be a good student, perhaps someone who is \textit{talented} and has good \textit{work ethic}. As she compares students, her true comparison is based on these unobservable, complex mappings of the high dimensional application to talent and work ethic, which represent her judgment criteria for similarity. Even if the \human~cannot articulate her mapping from the high dimensional applicant information to her low dimensional representation, her judgments which reflect the low dimensional representation can still be used for triangulation. 
There is also an opportunity to build on prior work concerning human decision-making and categorization in other disciplines. Further cross-disciplinary inquiry is likely to be highly beneficial to producing more realistic models of how humans encode similarity judgments.

In this work, we have relied on learning methods with particular theoretical guarantees for generalization and nontriviality with the goal of stating results independent from assumptions on the form of the metric. This focus has resulted in conservative nontriviality guarantees and numbers of representatives. In practice, exploring alternative methods based on metric structure assumptions, whether or not they have theoretical guarantees on outcomes, and instead budgeting some labeled data to measure empirical error may be more practical.

\subsection{Resolving disagreements between \humans}
Thus far, we have assumed that our procedures either use a single, internally consistent \human~or that multiple \humans~agree on all queries. Multiple \humans~is likely preferable from the perspective of better capturing society's view, and can answer relative distance queries in parallel for Algorithm \ref{alg:tripletordering}. However, the assumption that all \humans~agree on every query is not likely to hold up in practice.

In the case of small disagreements between \humans, $\minmerge$ (defined analogously to $\maxmerge$) is a viable option. For example, if the ordering produced by two \humans~from a particular representative is the same, but there are inconsistencies in the real-valued queries (after any necessary scaling), $\minmerge$ will smooth out any small disagreements. Setting $\alphH$ and $\alphL$ to capture the varying levels of agreement can also have a similar effect in the \tctc~model.

When \humans~strongly disagree, we consider this to be a situation where discussion between the \humans, and perhaps additional external parties, is needed.
If we assume that our \humans~are all fair-minded individuals (i.e., without explicitly unpalatable biases), then our interpretation of significant disagreements should be careful to acknowledge that 
disagreements may stem from either
(1) differences in domain expertise,
(2) genuine lack of consensus in society's view of similarity for the task,
(3) human or system error or bias in display of or acquisition of data,
(4) other potentially serious failure modalities.

We view the potential for such disagreements as a feature, not a bug, and would be concerned if any system gathering judgments from \humans~\textit{never} encountered disagreement.
In the case of (1) we anticipate that there may be cases where a particular \human~is selected precisely because she represents a unique viewpoint or has domain experience with different groups of individuals. Ensembles of \humans~with expertise in different groups of individuals may find augmenting the procedures outlined in this work with more nuanced merge and discussion steps for reconciliation between \humans~to be beneficial.
In the case of lack of consensus (2), procedural fairness or other interventions may be more desirable than fairness derived from outcomes.
We discuss human or system bias in Section \ref{subsection:humandesign}.

A significant benefit of disagreement with a proposed submetric or individual query is that these disagreements represent specific, well-articulated cases rather than hypothetical or meta-disagreements. Our hope is that the discussion of specific cases will be more likely to result in agreement, either in the outcome or in the choice of an alternate procedure, than hypothetical cases or group-level statistics.

\subsection{Selection of \humans}
Conspicuously missing from our discussion of \humans~is guidance on \textit{whom} to select to be a \human. Our most basic requirement is that a \human~be a ``fair minded individual,'' but practically speaking, this gives little indication of selection criteria.
That being said, the selection of \humans~is likely to be critically important to the acceptance of any submetric produced. Our position is that selection of \humans~is a question which must be resolved at a philosophical, policy and social level.
We can foresee many questions related to arbiter selection. For example, should historically disadvantaged groups be given the choice of some number of the \humans? Is there some minimum qualification or ``bias test'' one must pass to be considered? 
Resolving, or even attempting to fully articulate such questions is well outside the scope of this work, and we anticipate that it is a significant area for future cross-disciplinary inquiry.

However, we are optimistic that selecting a group of \humans~is possible, because the learning process permits changing the set of \humans~or the merge strategy over time without ``throwing away'' past effort. Consider learning a separate submetric for each \human~and merging these submetrics (either through $\maxmerge,$ $\minmerge$, or any other more nuanced merging strategy). Adding or removing a \human~is not wholly destructive to the existing submetric, although this strategy may preclude parallelizing relative distance queries. Loosely speaking, we may not be able to give good guidance up front about who should be a \human, but we can produce submetrics in a way that adding or removing an \arbiter~from the set is straightforward, allowing the metric to evolve as our understanding or opinion of who should be in the set of \humans~and how their judgments should be combined evolves. This replacement strategy may also help in cases where opinions shift gradually over time, and older \arbiter~submetrics may be swapped out for newer judgments to reflect shifting attitudes.

\subsection{Query process and interface design}\label{subsection:humandesign}
The design and process implementation for the interactions with \humans~is a significant area for future work. 
Problems of anchoring, particularly if many individuals are compared to the same representative, in addition to other issues with human judgment will be a significant consideration in system design \cite{tversky1974judgment}.
Alternative query types could be explored, or alternative presentation of queries could be made to improve the consistency of answers or try to counteract implicit biases.

Of particular concern with the design of the interface is how information is presented and whether the presentation will allow or encourage implicit biases to creep into judgments. It is likely impossible to remove all signal for sensitive attributes like race or gender from the presentation of information to the judge. Indeed, there are many cases where the inclusion of sensitive information is critical to evaluating fairness.
One possible way to detect and correct implicit bias would be to explicitly ask the \human~if they believe a sensitive feature \textit{should} impact a particular judgment. If they respond that it should \textit{not}, then the system could spot check by asking other \arbiters~to evaluate the same query with as much sensitive information stripped out or changed to an alternative as possible. If the evaluations of the other \humans~indicate that removing or changing the sensitive information resulted in different judgments, then additional care could be taken to reconcile the sensitive-attribute-blind responses. This is by no means a complete solution to removing implicit bias, but we think that exploring how information is presented and in particular comparing judgments based on differing information will be critical to gathering consistent and consistently fair judgments from the \humans.

Any judgments \humans~make based on the information presented to them will be just that: \textit{based on the information presented to them}. In many cases, we might want to allow the \human~to gather or request additional information if it is important to their judgment. For example, a \human~evaluating a college application might see that a student took a year away from school. She may determine that additional information is needed to make any meaningful comparisons, because a year away from school for medical or family reasons is very different than a year's suspension. Building in a way for \humans~to gather more information, and document the information they find for any later evaluations to consider, is likely to be expensive but may also be necessary to produce valid judgments.

\subsection{\Arbiter~agreement with \submetrics}\label{section:arbiteragreement}
Our initial assumption might be that a set of \humans~will agree with a submetric learned based on their judgments, modulo error parameters.
In the case of real valued distance queries, the procedures outlined in this work will result in submetrics which underestimate real-valued distances with small error with high probability.
However, with respect to \textit{relative} distances, agreement is not guaranteed.
For example, all of the \humans~may agree that $a$ is more similar to $b$ than $c$, but the submetric may consider $a$ more similar to $c$ than $b$ (while still maintaining smaller real value distances) depending on the choice of representative elements.\footnote{ This is illustrated in Figure \ref{fig:repcomparison}, in which choosing the representative $r1$ preserves the relative distance comparison between points $2$, $4$ and $5$ ($2$ is closer to $4$ than $5$) but choosing $r2$ does not.} If all original distances are maintained to a sufficient degree, then relative distances will also be preserved. However, when trade-offs between distance preservation and \human~cost must be made, there is the potential to violate relative distance judgments made by the \human.

Although this does not technically violate the Individual Fairness definition of \cite{Dwork-FTA}, there may be many scenarios where treating \textit{dissimilar} individuals dis-similarly is just as important as treating similar individuals similarly. For example, in the case of setting taxation rates for individuals, one would likely consider treating the wealthiest and poorest individuals the same but treating the middle class differently to be unfair.
Augmenting the existing Individual Fairness definition with the requirement that dissimilar individuals be treated dis-similarly is not entirely straightforward. In particular, there is no binary classifier which will maintain relative distances between three equally distant individuals. However, given the uncomfortable idea that the \humans~may not agree with the relative distances produced, it seems worthwhile to consider whether, or in which cases, it is desirable or possible to strengthen the Individual Fairness definition of \cite{Dwork-FTA} to capture relative distance constraints.

\subsection{When \arbiters~agree but learning is hard}
An important scenario to consider is the case in which the \humans~agree on all or most queries, but our usual learning procedures fail to produce a submetric which generalizes to unseen samples. Again, we view this failure as a feature rather than a bug as it may indicate that either (1) there are alternative learning strategies we should try or (2) that the metric is complex enough that human oversight is always needed to make fair decisions.
Our model of the \arbiter~evaluating distances over an unobservable set of relevant attributes is very similar to the ``construct space'' of Friedler et al., \cite{friedler2016possibility}. Friedler et al. put forth a formalization of fairness in which the goal is to achieve fairness over an unobservable construct space which captures the relevant attributes (e.g., grit, talent, work ethic, etc) but our information constrained to the ``observed'' space. In some sense, we take the view that the \arbiter~is acting as a translator between these unobserved, difficult to articulate attributes and the observed features. As such, there isn't always a guarantee that the observed features available for classification will be sufficient to capture the nuance in the \arbiters' judgments. 
In some sense, replacing direct human judgments with automated decisions in sensitive settings should be viewed as a privilege and not a right. Sensitive settings in which \humans~agree, but our system cannot generalize in a way that they would agree with, should be subject to significant scrutiny and the replacement of human judgment with automated decision-making should not be taken as given.

\subsection{Comparison of submetrics}
In this work, we have been somewhat unsophisticated in our comparisons of alternative submetrics beyond the basic worst-case additive error measure and nontriviality. 

In this work we primarily consider absolute additive error. However,  practical evaluation of error may be based entirely on how much an adversary could ``get away with'' using a submetric to derive a classifier.
Suppose we are concerned an adversary will discriminate against a large subset of individuals $V \subseteq U$ and derives utility proportional to the difference in distances between pairs of elements $(u,v)$ where $v \in V$, and $u \in U \backslash V$. A large number of small errors would allow the adversary to pull all or most members of $V$ further away from their $U \backslash V$ counterparts. Alternatively, a smaller number of very large errors, so long as they are not concentrated on pairs containing a small group of individuals in $V$, will be harder for the adversary to take advantage of, because there are many accurate distances making it difficult to ``move'' elements of $V$ relative to their close counterparts in $U \backslash V$. We expect that many of the error type questions we would pose for metric learning have a close analogy to the problem of selecting comparison sets in \cite{kim2018fairness}.


From a more constructive perspective, we might also find it difficult to compare nontriviality parameters absent understanding of how the submetric will be used. For example, a submetric which preserves distances very well between unqualified individuals but does little to distinguish qualified individuals may not be terribly helpful in deciding between qualified individuals.
Developing a more nuanced model for evaluation of submetrics, both from the perspective of abuse and constructing distinguishing classifiers, will be critical to providing good guarantees on submetric use in practice.

\printbibliography

@article{jung2019eliciting,
  title={Eliciting and Enforcing Subjective Individual Fairness},
  author={Jung, Christopher and Kearns, Michael and Neel, Seth and Roth, Aaron and Stapleton, Logan and Wu, Zhiwei Steven},
  journal={arXiv preprint arXiv:1905.10660},
  year={2019}
}

@article{dasgupta2017learning,
  title={Learning from partial correction},
  author={Dasgupta, Sanjoy and Luby, Michael},
  journal={arXiv preprint arXiv:1705.08076},
  year={2017}
}

@article{DBLP:journals/corr/BelletHS13,
  author    = {Aur{\'{e}}lien Bellet and
               Amaury Habrard and
               Marc Sebban},
  title     = {A Survey on Metric Learning for Feature Vectors and Structured Data},
  journal   = {CoRR},
  volume    = {abs/1306.6709},
  year      = {2013},
  url       = {http://arxiv.org/abs/1306.6709},
  archivePrefix = {arXiv},
  eprint    = {1306.6709},
  bibsource = {dblp computer science bibliography, https://dblp.org}
}

@article{kulis2013metric,
  title={Metric learning: A survey},
  author={Kulis, Brian and others},
  journal={Foundations and Trends{\textregistered} in Machine Learning},
  volume={5},
  number={4},
  pages={287--364},
  year={2013},
  publisher={Now Publishers, Inc.}
}

@inproceedings{van2012stochastic,
  title={Stochastic triplet embedding},
  author={Van Der Maaten, Laurens and Weinberger, Kilian},
  booktitle={Machine Learning for Signal Processing (MLSP), 2012 IEEE International Workshop on},
  pages={1--6},
  year={2012},
  organization={IEEE}
}

@article{zou2015crowdsourcing,
  title={Crowdsourcing feature discovery via adaptively chosen comparisons},
  author={Zou, James Y and Chaudhuri, Kamalika and Kalai, Adam Tauman},
  journal={arXiv preprint arXiv:1504.00064},
  year={2015}
}

@inproceedings{valiant1984theory,
  title={A theory of the learnable},
  author={Valiant, Leslie G},
  booktitle={Proceedings of the sixteenth annual ACM symposium on Theory of computing},
  pages={436--445},
  year={1984},
  organization={ACM}
}

@article{gillen2018online,
  title={Online Learning with an Unknown Fairness Metric},
  author={Gillen, Stephen and Jung, Christopher and Kearns, Michael and Roth, Aaron},
  journal={arXiv preprint arXiv:1802.06936},
  year={2018}
}

@article{Dwork-FTA,
  author    = {Cynthia Dwork and
               Moritz Hardt and
               Toniann Pitassi and
               Omer Reingold and
               Richard S. Zemel},
  title     = {Fairness Through Awareness},
  journal   = {CoRR},
  volume    = {abs/1104.3913},
  year      = {2011},
  url       = {http://arxiv.org/abs/1104.3913},
  bibsource = {dblp computer science bibliography, http://dblp.org}
}

@inproceedings{wilber2014cost,
  title={Cost-effective hits for relative similarity comparisons},
  author={Wilber, Michael J and Kwak, Iljung S and Belongie, Serge J},
  booktitle={Second AAAI Conference on Human Computation and Crowdsourcing},
  year={2014}
}

@inproceedings{jamieson2011low,
  title={Low-dimensional embedding using adaptively selected ordinal data},
  author={Jamieson, Kevin G and Nowak, Robert D},
  booktitle={Communication, Control, and Computing (Allerton), 2011 49th Annual Allerton Conference on},
  pages={1077--1084},
  year={2011},
  organization={IEEE}
}

@inproceedings{frome2007learning,
  title={Learning globally-consistent local distance functions for shape-based image retrieval and classification},
  author={Frome, Andrea and Singer, Yoram and Sha, Fei and Malik, Jitendra},
  booktitle={Computer Vision, 2007. ICCV 2007. IEEE 11th International Conference on},
  pages={1--8},
  year={2007},
  organization={IEEE}
}

@inproceedings{adaptively-learning-crowd-kernel,
author = {Tamuz, Omer and Liu, Ce and Belongie, Serge and Shamir, Ohad and Kalai, Adam},
title = {Adaptively Learning the Crowd Kernel},
booktitle = {},
year = {2011},
month = {June},
publisher = {},
url = {https://www.microsoft.com/en-us/research/publication/adaptively-learning-crowd-kernel/},
address = {},
pages = {},
journal = {},
volume = {},
chapter = {},
isbn = {},
}

@article{kim2018fairness,
  title={Fairness Through Computationally-Bounded Awareness},
  author={Kim, Michael P and Reingold, Omer and Rothblum, Guy N},
  journal={arXiv preprint arXiv:1803.03239},
  year={2018}
}

@article{rothblum2018probably,
  title={Probably Approximately Metric-Fair Learning},
  author={Rothblum, Guy N and Yona, Gal},
  journal={arXiv preprint arXiv:1803.03242},
  year={2018}
}

@article{DBLP:journals/corr/KleinbergMR16,
  author    = {Jon M. Kleinberg and
               Sendhil Mullainathan and
               Manish Raghavan},
  title     = {Inherent Trade-Offs in the Fair Determination of Risk Scores},
  journal   = {CoRR},
  volume    = {abs/1609.05807},
  year      = {2016},
  url       = {http://arxiv.org/abs/1609.05807},
  bibsource = {dblp computer science bibliography, http://dblp.org}
}

@article{chouldechova2017fair,
  title={Fair prediction with disparate impact: A study of bias in recidivism prediction instruments},
  author={Chouldechova, Alexandra},
  journal={arXiv preprint arXiv:1703.00056},
  year={2017}
}

@article{friedler2016possibility,
  title={On the (im) possibility of fairness},
  author={Friedler, Sorelle A and Scheidegger, Carlos and Venkatasubramanian, Suresh},
  journal={arXiv preprint arXiv:1609.07236},
  year={2016}
}

@inproceedings{dwork2001rank,
  title={Rank aggregation methods for the web},
  author={Dwork, Cynthia and Kumar, Ravi and Naor, Moni and Sivakumar, Dandapani},
  booktitle={Proceedings of the 10th international conference on World Wide Web},
  pages={613--622},
  year={2001},
  organization={ACM}
}

@inproceedings{volkovs2012learning,
  title={Learning to rank by aggregating expert preferences},
  author={Volkovs, Maksims N and Larochelle, Hugo and Zemel, Richard S},
  booktitle={Proceedings of the 21st ACM international conference on Information and knowledge management},
  pages={843--851},
  year={2012},
  organization={ACM}
}

@unpublished{DKRRY2019, 
   author = {Dwork, Cynthia and Kim, Michael and Reingold, Omer and Rothblum, Guy and Yona, Gal}, 
   title = {Learning from Outcomes: Evidence-Consistent Rankings}, 
   note = {Manuscript submitted for publication}, 
   year = {2019} }

@article{tversky1974judgment,
  title={Judgment under uncertainty: Heuristics and biases},
  author={Tversky, Amos and Kahneman, Daniel},
  journal={science},
  volume={185},
  number={4157},
  pages={1124--1131},
  year={1974},
  publisher={American association for the advancement of science}
}

\end{document}